\documentclass[11pt]{article}
\usepackage[backend=biber, maxbibnames=99]{biblatex}
\usepackage{amsmath,amsfonts,amssymb,amsthm,  mathrsfs,mathtools}
\usepackage[margin=1in]{geometry}
\usepackage{hyperref}
\hypersetup{colorlinks}
\usepackage[ruled]{algorithm2e}
\usepackage{subcaption}
\usepackage{xfrac}
\usepackage{enumitem}

\usepackage{booktabs}
\usepackage{multirow}

\theoremstyle{plain}
\newtheorem{theorem}{Theorem}[section]
\newtheorem{lemma}[theorem]{Lemma}
\newtheorem{corollary}[theorem]{Corollary}

\theoremstyle{definition}

\newtheorem{assumption}[theorem]{Assumption}
\newtheorem{observation}[theorem]{Observation}
\DeclareMathOperator*{\argmax}{arg\,max}
\DeclareMathOperator*{\argmin}{arg\,min}

\newcommand{\prob}{\textsc{EqCenter}}

\begin{document}

\title{A New Notion of Individually Fair Clustering: \\ $\alpha$-Equitable $k$-Center}

\author{Darshan Chakrabarti\thanks{{Harvard University.
Email: \href{mailto:dchakrabarti@seas.harvard.edu}{dchakrabarti@seas.harvard.edu}}} \and John P. Dickerson\thanks{{University of Maryland, College Park. 
Email: \href{mailto:john@cs.umd.edu}{john@cs.umd.eduu}}} \and Seyed A. Esmaeili\thanks{{University of Maryland, College Park. Email: \href{mailto:esmaeili@cs.umd.edu}{esmaeili@cs.umd.edu}}} \and Aravind Srinivasan\thanks{{University of Maryland, College Park. 
Email: \href{mailto:srin@cs.umd.edu}{srin@cs.umd.edu}}} \and Leonidas Tsepenekas \thanks{{University of Maryland, College Park.  
Email: \href{mailto:ltsepene@umd.edu}{ltsepene@umd.edu}}} }


\date{}
\maketitle

\begin{abstract}
   Clustering is a fundamental problem in unsupervised machine learning, and due to its numerous societal implications fair variants of it have recently received significant attention. In this work we introduce a novel definition of individual fairness for clustering problems. Specifically, in our model, each point $j$ has a set of other points $\mathcal{S}_j$ that it perceives as similar to itself, and it feels that it is being fairly treated if the quality of service it receives in the solution is $\alpha$-close (in a multiplicative sense, for some given $\alpha \geq 1$) to that of the points in $\mathcal{S}_j$. We begin our study by answering questions regarding the combinatorial structure of the problem, namely for what values of $\alpha$ the problem is well-defined, and what the behavior of the \emph{Price of Fairness (PoF)} for it is. For the well-defined region of $\alpha$, we provide efficient and easily-implementable approximation algorithms for the $k$-center objective, which in certain cases also enjoy bounded-PoF guarantees. We finally complement our analysis by an extensive suite of experiments that validates the effectiveness of our theoretical results.
\end{abstract}

\section{Introduction}\label{sec:intro}

In a typical clustering problem, there is a set of points $\mathcal{C}$ in a metric space characterized by a distance function $d: \mathcal{C}^2 \mapsto \mathbb{R}_{\geq 0}$, where $d$ is some non-increasing function of similarity or proximity. The goal is to choose a set $S \subseteq \mathcal{C}$ of at most $k$ representative centers, and subsequently construct an assignment $\phi: \mathcal{C} \mapsto S$ that maps each point to one of the chosen centers, thus creating a collection of at most $k$ clusters. In addition, the quantity that really matters for each $j \in \mathcal{C}$, is the distance $d(j,\phi(j))$ to its corresponding cluster center $\phi(j)$. This distance \emph{represents the quality of service $j$ receives}. In classical clustering applications $d(j,\phi(j))$ would correspond to how similar $\phi(j)$ is to $j$, and in facility-location applications to the distance $j$ needs to travel in order to reach its service-provider. Hence, from an individual perspective, each $j$ requires $d(j,\phi(j))$ to be as small as possible. The most popular objectives in the literature ($k$-center, $k$-median, $k$-means) ``boil down'' this large collection of values $d(j,\phi(j))$, into an increasing function they try to minimize.  

In scenarios where the points correspond to selfish agents, it is natural to assume that they will be mindful of the quality of service other points receive. Specifically, a point $j$ may feel that it is being handled unfairly by a solution $(S,\phi)$, if $d(j,\phi(j))$ is much larger than the assignment distances a group $\mathcal{S}_j$ of other points obtains. In this context, the points of $\mathcal{S}_j$ are exactly those which $j$ perceives as similar to itself, hence it arguably believes that it should obtain similar treatment as them. 
As a practical example, consider the following application in an e-commerce site, where the points of $\mathcal{C}$ correspond to its users and $d(j,j')$ measures how similar the profiles of $j$ and $j'$ are. In order to provide relevant recommendations, the website needs to choose a set $S$ of $k$ representative users, and then assign each point to one of those based on a mapping $\phi: \mathcal{C} \mapsto S$. The recommendations $j$ gets will be based on $\phi(j)$'s profile, and in this case the quantity $d(j,\phi(j))$ corresponds to how representative $\phi(j)$ is for $j$, and hence how suitable $j$'s recommendations are. In this scenario, a point $j$ may feel unfairly treated, if points that are similar to it (points $j'$ with small $d(j,j')$) get better recommendations and consequently better service (see, e.g., the work of \cite{ad-privacy} for studies on similar users receiving  different types of job recommendations).

In addition, this sort of fairness considerations are  applicable when seeking equity in healthcare provision, such as in vaccine allocation: the clusters could represent groups of people who would be given health-related resources such as treatment from a facility, and we aim for similar people to get similar commute-times to their resource provider.

Here we formalize this abstract notion of fairness via two rigorous and related constraints, which we incorporate into the $k$-center problem. We focus on $k$-center due to its numerous practical applications, but mostly because of its theoretical simplicity, which allows us to explore in depth the intricacies and the combinatorial structure of this novel notion of individually-fair clustering.

\subsection{Formal problem definitions}\label{sec:def}
We are given a set of points $\mathcal{C}$ in a metric space characterized by the distance function $d: \mathcal{C}^2 \mapsto \mathbb{R}_{\geq 0}$. Moreover, the input includes a positive integer $k$ and a value $\alpha \geq 1$. Finally, for every $j \in \mathcal{C}$ we have a \emph{similarity set} $\mathcal{S}_j \subseteq \mathcal{C}$, denoting the group of points that are deemed similar to $j$. 

The goal in our problems of interest is to choose a set $S \subseteq \mathcal{C}$ of at most $k$ centers, and then find an assignment $\phi: \mathcal{C} \mapsto S$, such that the $k$-center objective, i.e., $\max_{j \in \mathcal{C}}d(j,\phi(j))$, is minimized. Further, we use two different constraints to capture the notion of fairness we aim to study.
\begin{itemize}
    \item \textbf{Per-Point Fairness ($PP$):} When we study the problem under this constraint, we want to make sure that for all $j \in \mathcal{C}$ with $\mathcal{S}_j \neq \emptyset$, we have: 
    \begin{align}
        d(j,\phi(j)) \leq \alpha \cdot \min_{j' \in \mathcal{S}_j}d(j',\phi(j')) \label{cons:PP}
    \end{align}
    Here $j$ is satisfied if its quality of service is at most $\alpha$ times the ``best" quality found in $\mathcal{S}_j$. Equivalently, we should guarantee that $d(j,\phi(j)) \leq \alpha \cdot d(j',\phi(j'))$ for all $j \in \mathcal{C}$ and $j' \in \mathcal{S}_j$.
    \item \textbf{Aggregate Fairness ($AG$):} Here for each $j \in \mathcal{C}$ with $\mathcal{S}_j \neq \emptyset$, we want to guarantee that: 
    \begin{align}
        d(j,\phi(j)) \leq \alpha \frac{\sum_{j' \in \mathcal{S}_j}d(j',\phi(j'))}{|\mathcal{S}_j|} \label{cons:AG}
    \end{align} Hence, here $j$ feels fairly treated if $d(j,\phi(j))$ is at most $\alpha$ times the average quality of $\mathcal{S}_j$.
\end{itemize}
We call our problem \textbf{$\alpha$-Equitable $k$-Center}, and denote it by \prob. Moreover, we consider it either under constraint (\ref{cons:PP}) or under constraint (\ref{cons:AG}). When we study it under (\ref{cons:PP}) we refer to it as \prob{-PP}, and similarly when we use constraint (\ref{cons:AG}) we denote it by \prob{-AG}. Further, both variants are NP-hard, since they trivially generalize $k$-center, which is known to be NP-hard.

Constraint (\ref{cons:PP}) provides a stronger notion of fairness, in that each point $j$ cares explicitly about every point $j' \in \mathcal{S}_j$. Constraint (\ref{cons:AG}) is weaker, in the sense that the points now compromise to  comparing their quality of service to the \emph{average} quality obtained by their similarity set. Due to this, a solution for (\ref{cons:PP}) also constitutes a solution for (\ref{cons:AG}), and hence for the same instance the optimal value of \prob{-AG} must be no larger than that of \prob{-PP}. \emph{This observation reveals an intriguing trade-off between how strict we want to be in our fairness constraints, and how much we care about the overall objective cost.} We further explore this issue in Section \ref{sec:exps}.

\subsubsection{The structure of the similarity sets $\mathcal{S}_j$}\label{sec:asm}
In our work we do not consider an arbitrary model of similarity, but we rather focus on distance based similarity. On a high-level, this means that points which are far apart in the metric space, cannot really be similar. Such an approach for instantiating similarity is extensively utilized for fair clustering \cite{brubach2020, brubach2021, anderson2020}, with elements of it appearing in \cite{Dwork2012, jung2019} as well. Moreover, this concept is highly realistic, since in many conceivable applications the function $d$ already captures a notion of resemblance. For instance, in the previously mentioned use-case of a recommendation system, two users that are close under $d$ have comparable profiles, and thus can be seen as similar.

The way we capture distance-based similarity in this paper, is by considering sets $\mathcal{S}_j$ that satisfy a well-established assumption from \cite{brubach2020}, which was used to define similarity between points in a different individually-fair clustering problem. Specifically, suppose that we have an instance $(\mathcal{C}, k)$ of vanilla/``unfair'' $k$-center, whose optimal value is $R^{*}_{unf}$. In other words, this is merely an instance of the standard $k$-center problem, where we want to choose $(S, \phi)$ with $|S| \leq k$ such that $\max_{j \in \mathcal{C}}d(j,\phi(j))$ is minimized, and no fairness constraints are imposed. Further, assume that this instance is extended to an instance of either \prob{-PP} or \prob{-AG}, by choosing an arbitrary $\alpha$ value and sets $\mathcal{S}_j$. Then the following will hold. 

\begin{assumption}\label{icml-asm}
For every $j \in \mathcal{C}$ we have $\mathcal{S}_j \subseteq \{j' \in \mathcal{C}: d(j,j') \leq \psi R^{*}_{unf}\}$, for some $\psi = O(1)$.
\end{assumption}

Therefore, for instances of \prob{-PP} and \prob{-AG}, two points can be similar if their distance is at most $\psi R^{*}_{unf}$, where $\psi$ is some small constant and $R^{*}_{unf}$ is the optimal value of the underlying unfair $k$-center instance. 

Although Assumption \ref{icml-asm} is adequately justified in \cite{brubach2020}, we also give some intuition for it. Consider the optimal solution for the unfair problem on $(\mathcal{C}, k)$. Then, the triangle inequality implies that a point $j$ will never be placed in the same cluster as some other $j'$ with $d(j,j') > 2 R^{*}_{unf}$. Hence, the optimal unconstrained/unfair solution that can be thought of as an expert when it comes to determining similarity (it constructs the most intra-similar clusters), does not deem the two points comparable enough to place them in the same cluster. Therefore, following the ``advice'' of the optimal unconstrained solution yields $\psi = 2$ in Assumption \ref{icml-asm}, and due to the previous explanation, this value can be actually interpreted as the \emph{canonical case} for $\psi$. 

For scenarios where we are not certain of whether Assumption \ref{icml-asm} holds, or the points have a fuzzy understanding of similarity that does not allow them to meaningfully define their sets $\mathcal{S}_j$, see Appendix \ref{sec:asm-enf} for an explainable way of enforcing $\mathcal{S}_j \subseteq \{j' \in \mathcal{C}: d(j,j') \leq \psi R^{*}_{unf}\}$ for all $j$.

To conclude, we need to define some more notation. Given similarity sets $\mathcal{S}_j$ for every $j \in \mathcal{C}$, we define $R_j = \max_{j' \in \mathcal{S}_j}d(j,j')$ and $R_m = \max_{j \in \mathcal{C}}R_j$.

\subsection{Our contributions and discussion of our results}\label{sec:contr}
In Section \ref{sec:str} we investigate the combinatorial structure of our newly introduced fairness constraints. At first, a question that naturally arises is for what values of $\alpha$ are our problems well-defined. We call a problem well-defined if it always admits a feasible solution $(S,\phi)$, i.e., $|S| \leq k$ and $\phi$ satisfies the corresponding fairness constraint for all $j$. Ideally, we would like our problems to admit feasible solutions for any possible value of $\alpha$. However we give the next result which indicates that absolute equity is not achievable.
\begin{theorem}
For both \prob{-PP} and \prob{-AG}, there exist instances with $\alpha < 2$ that do not admit any feasible solution.
\end{theorem}
We then proceed by showing that for $\alpha \geq 2$ there is always a feasible solution, thus settling the crucial question about the regime of $\alpha$ for which our problems are well-defined. 
\begin{theorem}
For both \prob{-PP} and \prob{-AG}, every instance with $\alpha \geq 2$ always admits a feasible solution.
\end{theorem}

Given that $\alpha \geq 2$ is the range we should focus on, we proceed by studying another vital concept, and that is the \emph{Price of Fairness (PoF)} \cite{Bertsimas11, Caragiannis09}. This notion is just a measure of relative loss in system efficiency, when fairness constraints are introduced.
Specifically, for a given instance of either \prob{-PP} or \prob{-AG}, PoF is defined as the value of the optimal solution to our fair problem, over the value of the optimal solution to the underlying $k$-center instance, where we drop the fairness constraints from the problem's requirements. In other words, PoF $=$ (optimal fair value)/(optimal unfair value). In the vast majority of fair clustering problems it is known that there exist instances with unbounded PoF. In line with those results, we show the following.
\begin{theorem}
There exist instances of \prob{-PP} and \prob{-AG} with unbounded PoF. 
\end{theorem}
All results of Section \ref{sec:str} are proven for $k \geq 2$. See that the $k=1$ case is trivial, since one can efficiently try each point as a center, see if any yields a feasible solution, and also find the optimal solution among the computed feasible ones. On the other hand, even when $k = 2$ and we only have ${|\mathcal{C}| \choose 2} +|\mathcal{C}|$ center sets to check, the number of possible assignments for each set of size $2$ is $2^{|\mathcal{C}|}$.

In Section \ref{sec:alg} we provide an approximation algorithm that covers instances with $\alpha \geq 2$ for both \prob{-PP} and \prob{-AG}. The main body of the algorithm remains the same for the two problems, with minor differences to capture each unique case. Our process of choosing centers constitutes an extension of a result by \cite{Khuller2000}. Our procedure gives useful guarantees regarding the distances between chosen centers, a feature that is crucially exploited in the assignment phase of the algorithm, where we carefully construct the mapping $\phi$. Our result is:

\begin{theorem}\label{thm:alg-res}
Suppose we are given an instance with $\alpha \geq 2$ for either \prob{-PP} or \prob{-AG}, whose optimal value is $R^*$. Our algorithm provides a feasible solution $(S,\phi)$ to either problem, for which $\max_{j \in \mathcal{C}}d(j,\phi(j)) \leq 5 \max\{R^*, R_m\}$.
\end{theorem}

Due to Assumption \ref{icml-asm}, we immediately have $R_m \leq \psi R^{*}_{unf}$ with $\psi = O(1)$. Moreover, because $R^{*}_{unf}$ is an obvious lower bound for $R^*$, our algorithm produces constant-factor approximate solutions. For example, in the canonical case of $\psi = 2$ it gives a $10$-approximate solution.

Even though $R_m = O(R^{*}_{unf})$, notice that because we might have $R^* \geq R_m$, the algorithm of Theorem \ref{thm:alg-res} does not provide bounded PoF guarantees. Nonetheless, in Section \ref{sec:alg} we also study the PoF behavior of our algorithms, and specifically we prove the following. 
\begin{theorem}\label{thm:pof-1}
A small modification to our main algorithm yields $(S,\phi)$, with: \textbf{(i)} $|S| \leq 2k$, \textbf{(ii)} both constraints (\ref{cons:PP}) and (\ref{cons:AG}) satisfied by $\phi$, and \textbf{(iii)} $\max_{j \in \mathcal{C}}d(j,\phi(j)) \leq 5 \max\{\psi R^{*}_{unf}, R^{*}_{unf}\} $.
\end{theorem}

\begin{theorem}\label{thm:pof-2}
When for all $j \in \mathcal{C}$ we have $R_j = R_d$ for some $R_d$, our algorithm for \prob{-AG} provides a feasible solution with cost at most $5 \max\{\psi R^{*}_{unf}, R^{*}_{unf}\} $.
\end{theorem}

The result of Theorem \ref{thm:pof-1} says that there is an easy way to get an algorithm with bounded PoF guarantees, if we are willing to sacrifice the cardinality constraint on the set of chosen centers. On the other hand, Theorem \ref{thm:pof-2} says that when the value $R_j$ is the same for all points, then our main result yields a true approximation with bounded PoF for \prob{-AG}.

Furthermore, we mention that \emph{all algorithms of Section \ref{sec:alg} are purely combinatorial (e.g., do not require convex programming), and hence very efficient and easily implementable}. 

In Section \ref{sec:assgn} we study the assignment problem for \prob{-PP} and \prob{-AG}. To be more precise, if we are given the optimal set of centers $S^*$, can we find the corresponding optimal assignment $\phi^*$? In a vanilla clustering setting this is trivial, since assigning points to their closest center is easily seen to yield the necessary results. \emph{However, as is the case in almost all literature on fair clustering, in the presence of fairness constraints like (\ref{cons:PP}) or (\ref{cons:AG}), such an assignment is not necessarily correct}. This was actually among the first observations made in the seminal work of \cite{Chierichetti2017}, which initiated the research area of fair clustering. As a side note, the aforementioned observation implies that for a $c \in S^*$, we might end up having $\phi^*(c) \neq c$. Nonetheless, for our problems this does not constitute a modeling issue. Recalling the motivational example of a recommendation system for a website, we see that for a client $j$ chosen as a representative, assigning $j$ to a different representative $j'$ is an acceptable outcome, as long as all individuals feel fairly treated. 

Therefore, since from a theoretical perspective the assignment problem is fundamental in a clustering setting and because in our case it appears highly non-trivial, we choose to address it explicitly. In the end, we manage to show that with a slightly intricate iterative algorithm, we can indeed compute the optimal assignment $\phi^*$ in polynomial time.

Finally, Section \ref{sec:exps} contains an extensive experimental evaluation, that validates the effectiveness and the efficiency of our proposed algorithms.

\subsection{Related work}\label{sec:rel-wrk}

The most well-studied notion of fairness in clustering is the demographic one. Herein, the points are partitioned into demographic groups, and what is required is a fair treatment or a proportional representation of these groups in the solution. This area was initiated by the groundbreaking work of \cite{Chierichetti2017}. Further work on demographic fairness includes \cite{bercea2019,bera2020,esmaeili2020,huang2019,backurs2019,ahmadian2019,kleindessner19a, chen19, abbasi2021}.

The concept of fairness we consider here falls under the broader umbrella of individual fairness. The fundamentals of individual fairness were introduced in the seminal work of \cite{Dwork2012} in the context of classification. In addition, \cite{Dwork2012} demonstrated a series of shortcomings for demographic fairness, making the case for individual fairness stronger. The high-level idea proposed in that work was that \emph{similar individuals should be treated similarly}. Our model follows this paradigm by modeling similarity through the sets $\mathcal{S}_j$, and requiring similar treatment through constraints (\ref{cons:AG}) and (\ref{cons:PP}).

Previous work on individually-fair clustering that adheres to the notion of \cite{Dwork2012} includes \cite{anderson2020, brubach2020, brubach2021, kar2021}. However, these papers interpret similar treatment in a different way. Specifically, two points $j,j'$ that are similar should be placed in the same cluster (under some stochastic or lower-bounding sense). Hence, similar treatment is defined as guaranteeing $\phi(j) = \phi(j')$. Unlike our model, these papers provide no guarantee on the gap between $d(j,\phi(j))$ and $d(j',\phi(j'))$. 

There are also individually-fair clustering problems that do not follow the concept of ``similar points should be treated similarly''. \cite{mahabadi2020, jung2019} define individual fairness as ensuring that for each $j$ there will be a chosen center within distance $r_j$ from it, where $r_j$ is the minimum radius such that $|\{j' \in \mathcal{C}~|~d(j,j') \leq r_j\}| \geq |\mathcal{C}|/k$. Finally, \cite{kleindessner2020} views individual fairness as ensuring that each point is on average closer to the points in its own cluster than to the points in any other cluster.

Another work that is closely related to our model is that of \cite{Balcan2019}. In that paper the authors study a classification problem where there is a set of already-known labels, and the points need to be assigned to those via some stochastic classifier. The points have preferences over the labels, given by some utility function, and the final classification should be envy-free in the standard sense. Our model differs from that of \cite{Balcan2019} for two reasons. First, our focus is on a clustering problem, where the labels are not known, a metric related objective needs to be minimized, and also the assignment has to be deterministic. Secondly, although the concept of envy-freeness is related to constraint (\ref{cons:PP}), there is the crucial difference of points in our case not envying the resources allocated to other individuals, but rather their final utility. In other words, in the language of \emph{Fair Division of Goods}, our model is closer to the notion of an \emph{equitable allocation} \cite{Varian1974} rather than an \emph{envy-free} one.
 
Regarding the ``vanila''/``unfair'' (fairness-constraints-free) $k$-center, the best known approximation ratio for it is $2$ \cite{Hochbaum1985, Gonzalez1985}; this is  best-possible unless P=NP \cite{Hochbaum1986}. This hardness result also trivially extends to both variants of \prob{} as well. 
\section{Structural properties of the problem}\label{sec:str}

As mentioned in the introduction, all our results here are for $k \geq 2$, since $k=1$ is a trivial case. At first, we want to investigate the range of $\alpha$ for which our problems always admit a feasible solution. Ideally, an $\alpha$ value close to $1$ would be the most fair, but as the following theorem suggests, such a guarantee is impossible.

\begin{theorem}\label{thm:lb}
For both \prob{-PP} and \prob{-AG}, there exist instances with $\alpha < 2$ that do not admit any feasible solution.
\end{theorem}

\begin{proof}
\begin{figure}[t]
\centering
\includegraphics[scale=0.25]{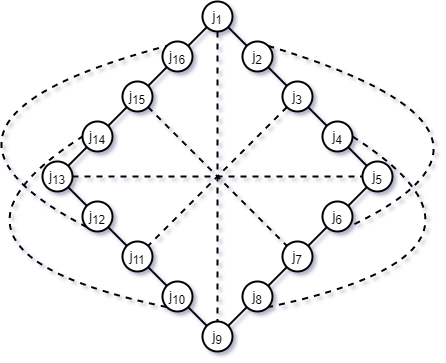}
\caption{Here $m=8$. Solid lines represent a distance of $1$ between points. Dashed lines correspond to similarity sets, eg., the dashed line between $j_1$ and $j_9$ shows that $\pi(j_1) = j_9$ and $\pi(j_9) = j_1$.}
\label{fig-1}
\end{figure}
Let $m$ be a very large even integer, with $\frac{m}{2}$ also being an even integer. We consider $2m$ points $\mathcal{C} = \{j_1, j_2, \hdots, j_{2m-1}, j_{2m}\}$ in a cycle, where $d(j_i, j_{i+1}) = 1$ for all $i \in [2m-1]$, and also $d(j_{2m}, j_1) = 1$. The rest of the distances are set to be the shortest path ones, based on those already defined. This is a valid metric space, since it constitutes the shortest path metric resulting from a simple cycle graph of $2m$ vertices.

To construct the similarity sets, we map each point $j$ to another point $\pi(j) \neq j$, such that the function $\pi: \mathcal{C} \mapsto \mathcal{C}$ is one-to-one and $\pi(\pi(j)) = j$. Given that, the similarity set of point $j$ will be set to be $\mathcal{S}_j = \{\pi(j)\}$. Now let $\mathcal{C}_1 = \{j_i ~|~ i \text{ is odd}\}$ and $\mathcal{C}_2 = \{j_i ~|~ i \text{ is even}\}$. For every odd $i \in [m]$, set $\pi(j_i) = j_{i+m}$ and $\pi(j_{i+m}) = j_i$. In this way, because $m$ is even, we map every point of $\mathcal{C}_1$ to some other point of $\mathcal{C}_1$. Also, note that for every $j \in \mathcal{C}_1$ we will have $d(j,\pi(j)) = m$. For the points $j_i \in \mathcal{C}_2$, consider them in increasing order of $i$. If $j_i$ is not already mapped to some other point, set $\pi(j_i) = j_{i + \frac{m}{2}}$ and $\pi(j_{i + \frac{m}{2}}) = j_i$. This is a valid assignment because $\frac{m}{2}$ is assumed to be an even integer. At the end of the above process, we have created a one-to-one mapping between the points of $\mathcal{C}_2$, such that for every $j \in \mathcal{C}_2$ we have $d(j,\pi(j)) = \frac{m}{2}$. This concludes the description of the similarity sets. Finally, this pairing process for $\mathcal{C}_1$ and $\mathcal{C}_2$ is possible, because both sets include an even number of points. See Figure \ref{fig-1} for an example.

To conclude the description of the input we also assume that $k=2$. At this point observe that the constructed instance also satisfies Assumption \ref{icml-asm} for $\psi \geq 2$, therefore covering the canonical case for $\psi$. This is because the optimal unfair value for the instance is easily seen to be $\frac{m}{2}$, while the maximum distance between similar points is $m$.

In addition, note that because for all $j$ we have $|\mathcal{S}_j| = 1$, constraints (\ref{cons:PP}) and (\ref{cons:AG}) are equivalent and hence showing infeasibility for this instance covers both \prob{-PP} and \prob{-AG}. Finally, to prove the statement of the theorem, it suffices to show that for all possible choices of centers and all possible corresponding assignments $\phi$, there will always be a point $j_p$ for which $d(j_p, \phi(j_p)) \geq 2 d(\pi(j_p), \phi(\pi(j_p)))$.

At first, notice that there exists no feasible solution that uses just one center. Supposing otherwise, let $c$ be the only chosen center. Then there exists only one possible assignment for $c$, and that is $\phi(c) = c$. Hence $d(c,\phi(c)) = 0$, and the fairness constraint for $\pi(c)$ will never be satisfied.

Now we will show that even solutions that pick two centers $c_1, c_2$ cannot admit any feasible assignment. We proceed via a case analysis on $d(c_1, c_2)$.
\begin{itemize}
    \item $d(c_1, c_2) \leq \frac{m}{3}$: Because the points of $\mathcal{C}_1$ and $\mathcal{C}_2$ alternate in the metric cycle, we know that there exists a $j \in \mathcal{C}_1$ such that $d(j,c_1) \leq 1$ (in the example of Figure~\ref{fig-1} we might have $c_1 = j_2, c_2=j_3$ and $j=j_1$, $\pi(j) = j_9$). By the triangle inequality we also get $d(j,c_2) \leq \frac{m}{3} + 1$. As for the point $\pi(j)$, we have:
    \begin{align}
        &d(\pi(j), c_1) \geq d(\pi(j), j) - d(j,c_1) \geq m - 1 \notag\\
        &d(\pi(j), c_2) \geq d(\pi(j), j) - d(j,c_2) \geq m - \frac{m}{3} - 1 = \frac{2m}{3} - 1\notag
    \end{align}
    From $\pi(j)$'s perspective, the best case situation regarding its fairness constraint is if $\pi(j)$ gets assigned to its closest center, and $j$ gets assigned to its farthest one. Given all the previous inequalities, we see that the best possible service for $\pi(j)$ is $\frac{2m}{3} -1$, and the worst possible service for $j$ is $\frac{m}{3} + 1$. We next show that even in this ideal situation for $\pi(j)$, its fairness constraint with $\alpha < 2$ will never be satisfied if $m$ is significantly large. To see this, note that $\frac{2m/3 -1}{m/3 + 1}$ is an increasing function of $m$ and also:
    \begin{align}
        \lim_{m \to \infty}\Big{(}\frac{\frac{2m}{3} -1}{\frac{m}{3} + 1}\Big{)} = \frac{2/3}{1/3} = 2 \notag
    \end{align}
    Therefore, for every given $\alpha < 2$, there exists an $m_a$ such that $\frac{2m_a/3 -1}{m_a/3 + 1} > \alpha$.
    \item $\frac{m}{3} < d(c_1, c_2) \leq \frac{2m}{3}$: In this case, because $m$ is assumed to be significantly large and because the points of $\mathcal{C}_1, \mathcal{C}_2$ alternate in the metric cycle, we can find a point $j \in \mathcal{C}_1$ in the shortest path between $c_1$ and $c_2$, which will be approximately in the middle of the path. Letting $\gamma \in (\frac{1}{3}, \frac{2}{3}]$ such that $d(c_1, c_2) = \gamma m$, we have $\frac{\gamma m}{2} - 1 \leq d(j,c_1), d(j,c_2) \leq \frac{\gamma m}{2} + 1$ (in the example of Figure~\ref{fig-1} we might have $c_1 = j_1, c_2 = j_5$ and $j=j_3$, $\pi(j)=j_{11}$). Regarding the possible assignments for $\pi(j)$ we have:
    \begin{align}
        &d(\pi(j),c_1) \geq d(j,\pi(j)) - d(j,c_1) \geq m - \frac{\gamma m}{2} - 1 = m\Big{(}\frac{2-\gamma}{2}\Big{)} - 1 \notag \\
        &d(\pi(j),c_2) \geq d(j,\pi(j)) - d(j,c_2) \geq m - \frac{\gamma m}{2} - 1 = m\Big{(}\frac{2-\gamma}{2}\Big{)} - 1 \notag
    \end{align}
    Again we will focus on the best case situation for $\pi(j)$, which according to the previous analysis is $\pi(j)$ getting assigned to a center at distance $\frac{m(2-\gamma)}{2} - 1$ from it, and $j$ getting assigned to a center at distance $\frac{\gamma m}{2} + 1$. Therefore, we consider the ratio $\frac{m(2-\gamma)/2 - 1}{\gamma m / 2 + 1}$, and we are going to prove that even in this ideal case for $\pi(j)$, its fairness constraint for $\alpha < 2$ will not be satisfiable if $m$ is suffieciently large. At first, because $\frac{2-\gamma}{2}, \frac{\gamma}{2} > 0$ the previous ratio will be an increasing function of $m$. In addition,
    \begin{align}
        \lim_{m \to \infty}\Big{(}\frac{m(2-\gamma)/2 - 1}{\gamma m / 2 + 1}\Big{)} = \frac{(2-\gamma)/2}{\gamma/2} = \frac{2-\gamma}{\gamma} \geq 2\notag
    \end{align}
    The last inequality follows since $\frac{2-\gamma}{\gamma}$ is a decreasing function, and for $\gamma \in (\frac{1}{3},\frac{2}{3}]$ we have $\frac{2-\gamma}{\gamma} \in [2, 5)$. Hence, for every $\alpha < 2$ there exists an $m_b$ such that $\frac{m_b(2-\gamma)/2 - 1}{\gamma m_b / 2 + 1} > \alpha$.
    \item $\frac{2m}{3} < d(c_1, c_2) \leq m$: Because $m$ is assumed to be significantly large and because the points of $\mathcal{C}_1, \mathcal{C}_2$ alternate in the metric cycle, we can find a point $j \in \mathcal{C}_2$ in the shortest path between $c_1$ and $c_2$, which will be approximately in the middle of the path. Letting $\gamma \in (\frac{2}{3}, 1]$ such that $d(c_1, c_2) = \gamma m$, we have $\frac{\gamma m}{2} - 1 \leq d(j,c_1), d(j,c_2) \leq \frac{\gamma m}{2} + 1$ (in Figure~\ref{fig-1} we might have $c_1=j_{2}, c_2=j_{10}$ and $j = j_{14}$, $\pi(j) = j_{10}$). Consider now $\pi(j)$, and without loss of generality assume that $d(\pi(j),c_1) \geq d(\pi(j),c_2)$ (when $d(\pi(j), c_1) \leq d(\pi(j),c_2)$ the situation is symmetric, with the roles of $c_1$, $c_2$ switched.). 
    
    At first, suppose that $\pi(j)$ is a point in the shortest path between $c_1$ and $c_2$ (in the example of Figure~\ref{fig-1} $c_1=j_{2}, c_2=j_{10}$ and $j = j_{14}$ would result in that). Thus, because $d(j, \pi(j)) = m/2, ~d(\pi(j),c_1) \geq d(\pi(j),c_2)$ and $d(c_1, c_2) \leq m$, we can focus on the line segment $c_1, j, \pi(j), c_2$, where the triangle inequality holds with equality. Here we get, $$d(\pi(j), c_2) = d(j,c_2) - d(j,\pi(j)) \leq \frac{\gamma m}{2} + 1 - \frac{m}{2} = \frac{(\gamma - 1)m}{2} + 1 \leq 1$$ In addition, $$d(\pi(j),c_1) = d(j,\pi(j)) + d(j,c_1) \geq \frac{m}{2} + \frac{\gamma m}{2} - 1 = \frac{(1+\gamma)m}{2} -1$$
    
    The second case we consider is when $\pi(j)$ is not on the shortest path between $c_1$ and $c_2$ (in Figure~\ref{fig-1} take for instance $c_1=j_1, c_2=j_{11}$ and hence $j=j_{14}$ and $\pi(j) = j_{10}$). In that scenario, because $d(\pi(j),c_1) \geq d(\pi(j),c_2)$, we turn our attention to the line segment $c_1, j, c_2, \pi(j)$, where the triangle inequality holds with equality. Here we have $$d(\pi(j), c_2) = d(j,\pi(j)) - d(j,c_2) \leq \frac{m}{2} - \frac{\gamma m}{2} + 1 = \frac{(1-\gamma)m}{2} + 1$$ In addition, $$d(\pi(j), c_1) = d(j,\pi(j)) + d(j,c_1) \geq \frac{m}{2} + \frac{\gamma m}{2} - 1 = \frac{(1+\gamma)m}{2} - 1$$ 
    
    Therefore, in every case we have the following:
    \begin{align}
        &d(\pi(j),c_1) \geq \frac{(1+\gamma)m}{2} - 1 \text{ and }
        d(\pi(j),c_2) \leq \frac{(1-\gamma)m}{2} + 1 \label{sec2:aux-1}
    \end{align}
    
    Now that we have the bounds (\ref{sec2:aux-1}) for the assignment distance of $\pi(j)$ to both centers, we proceed with the final case analysis.
    
    Suppose that $\pi(j)$ gets assigned to $c_1$. Then from $\pi(j)$'s perspective, the best possible situation is if its own assignment distance is exactly $\frac{(1+\gamma)m}{2} -1$, and $j$ gets an assignment distance of $\frac{\gamma m}{2} + 1$. In this case, the ratio $\frac{(1+\gamma)m/2 - 1}{\gamma m / 2 + 1}$ is an increasing function of $m$, because $(1+\gamma)/2, \gamma/2 > 0$. In addition we have:
    \begin{align}
        \lim_{m \to \infty}\frac{(1+\gamma)m/2 - 1}{\gamma m / 2 + 1} = \frac{1+\gamma}{\gamma} \geq 2 \notag
    \end{align}
    The last inequality is because $\frac{1+\gamma}{\gamma}$ is a decreasing function and $\gamma \leq 1$. Hence, for every $\alpha < 2$, there exists an $m_c$ such that $\frac{(1+\gamma)m_c/2 - 1}{\gamma m_c / 2 + 1} > \alpha$. Thus, even in the ideal situation for $\pi(j)$, if $m$ is larger than $m_c$ its fairness constraint for $\alpha < 2$ will be unsatisfiable.
    
    On the other hand, suppose that $\pi(j)$ gets assigned to $c_2$. Then from $j$'s perspective, the best possible situation is if it gets an assignment distance of $\frac{\gamma m}{2} - 1$, and $\pi(j)$ has assignment distance exactly $\frac{(1-\gamma)m}{2} + 1$. In this case, the ratio $\frac{\gamma m / 2 - 1}{(1-\gamma)m/2 + 1}$ is an increasing function of $m$, because $(1-\gamma)/2, \gamma/2 > 0$. Also:
    \begin{align}
        \lim_{m \to \infty}\frac{\gamma m / 2 - 1}{(1-\gamma)m/2 + 1} = \frac{\gamma}{1-\gamma} > 2 \notag
    \end{align}
    The last inequality is because $\frac{\gamma}{1-\gamma}$ is an increasing function and $\gamma >  2/3$. Hence, for every $\alpha < 2$, there exists an $m_d$ such that $\frac{\gamma m_d / 2 - 1}{(1-\gamma)m_d/2 + 1} > \alpha$. Thus, even in the ideal situation for $j$, if $m$ is larger than $m_d$, $j$'s fairness constraint for $\alpha < 2$ will be unsatisfiable.
\end{itemize}
The analysis is exhaustive, because the maximum distance between two points in the metric is $m$. Further, we see that if we set $m = 4\max\{m_a, m_b, m_c, m_d\}$, then in every possible scenario there will exist a point whose fairness constraint for $\alpha < 2$ will not be satisfiable.
\end{proof}

Moving on, we show that for $\alpha \geq 2$ there is always a feasible solution to both our problems, and hence we settle the important question of what is the smallest value of $\alpha$ for which \prob{-PP} and \prob{-AG} are well-defined.

\begin{lemma}\label{thm:a=2-aux}
Consider a set of points $\mathcal{C}$ in a metric space with distance function $d$, where $|\mathcal{C}| \geq 2$. Then there exists an efficient way of finding two distinct points $c_1, c_2 \in \mathcal{C}$ and an assignment $\phi: \mathcal{C} \mapsto \{c_1,c_2\}$, such that for every $j \in \mathcal{C}$ we have $\frac{d(c_1, c_2)}{2} \leq d(j,\phi(j)) \leq d(c_1, c_2)$.
\end{lemma}

\begin{proof}
At first, choose $c_1, c_2$ to be the two points of $\mathcal{C}$ that are the furthest apart, i.e. $(c_1, c_2) = \argmax_{x,y \in \mathcal{C}}d(x,y)$. Then, for every $j \in \mathcal{C}$ set $\phi(j) = \argmax_{c \in \{c_1,c_2\}}d(j,c)$. In other words, given the chosen centers, each point is assigned to the center that is furthest from it in the metric. Let also $\bar{\phi}(j)$ be the center to which $j$ is not assigned to. For any $j \in \mathcal{C}$, combining the triangle inequality and the fact that $d(j,\bar{\phi}(j)) \leq d(j,\phi(j))$, will give us: 
\begin{align}
    d(c_1,c_2) \leq d(j,\phi(j)) + d(j, \bar{\phi}(j)) \leq 2d(j,\phi(j)) \implies d(c_1,c_2)/2 \leq d(j, \phi(j)) \notag
\end{align}
Finally, by the way we chose $c_1$ and $c_2$ we also get $d(j,\phi(j)) \leq d(c_1, c_2)$.
\end{proof}

\begin{theorem}\label{thm:a=2}
For both \prob{-PP} and \prob{-AG}, every instance with $\alpha \geq 2$ always admits a feasible solution.
\end{theorem}

\begin{proof}
Suppose that as an instance to either problem we are given a set of points $\mathcal{C}$ together with their associated similarity sets $\mathcal{S}_j$, $k\geq 2$ and $\alpha \geq 2$. W.l.o.g. we can assume that $|\mathcal{C}| \geq 2$, because otherwise the statement of the Lemma is trivially true. Since $k \geq 2$, we can use Lemma \ref{thm:a=2-aux} and get a set of two centers $\{c_1, c_2\}$ and an assignment function $\phi: \mathcal{C} \mapsto \{c_1,c_2\}$, such that for all $j \in \mathcal{C}$ we have $d(c_1, c_2)/2 \leq d(j,\phi(j)) \leq d(c_1, c_2)$. In the case of constraint (\ref{cons:PP}), for every $j \in \mathcal{C}$ and any $j' \in \mathcal{S}_j$ we have $d(j,\phi(j)) \leq d(c_1, c_2) \leq 2d(j',\phi(j')) \leq \alpha d(j',\phi(j'))$. Furthermore, since any feasible solution for constraint (\ref{cons:PP}) is also a feasible solution for constraint (\ref{cons:AG}), the proof is concluded. 
\end{proof}

Another structural notion that interests us, is that of the Price of Fairness (PoF). For a given instance of either of our problems, PoF is the ratio of the value of the optimal solution to the problem, over the the optimal unfair value. The latter is defined as the optimal value of the given instance, when we drop the fairness constraint and simply solve $k$-center. As is the case in most fair clustering literature, we show that in general PoF can be arbitrarily large.

\begin{theorem}\label{thm:unbound-PoF}
There exist instances of \prob{-PP} and \prob{-AG} with unbounded PoF.
\end{theorem}

\begin{proof}
Consider an instance with four points $j_1, j_2, j_3, j_4$. For the distances we have $d(j_1, j_2) = d(j_3, j_4) = R$ and $d(j_1, j_3) = d(j_1, j_4) = d(j_2, j_3) = d(j_2, j_4) = D$, where $R \ll D$. Note that this is a valid metric space, where $j_1, j_2$ form a clique that is very far away from the clique of $j_3, j_4$. In addition, we assume $k = 2$ and $\alpha = 2$. For the similarity sets we have $\mathcal{S}_{j_1} = \{j_2\}$, $\mathcal{S}_{j_2} = \{j_1\}$, $\mathcal{S}_{j_3} = \{j_4\}$, $\mathcal{S}_{j_4} = \{j_3\}$.

Observe that the value of the optimal unfair solution is clearly $R$. This is achievable by choosing $j_1, j_3$ as centers. Given this, we see that the instance also satisfies Assumption \ref{icml-asm} since $R \ll D$.

Moving forward, we are going to show that the optimal solution for the fair variants has value $D$ (note that the existence of such a solution is guaranteed by Theorem \ref{thm:a=2}). This implies that PoF is $\frac{D}{R}$, and since $R \ll D$ this ratio can be arbitrarily large. Furthermore, note that because all similarity sets have cardinality $1$, constraints (\ref{cons:PP}) and (\ref{cons:AG}) are equivalent and hence we can solely focus on proving the result for (\ref{cons:PP}).

At first, assume that the optimal fair solution uses only one center. Then, any assignment that uses only one center should necessarily yield a maximum assignment distance of $D$.

Let us now consider the case of the optimal fair solution using two centers. If both these centers are in the same clique, i.e., the centers are either $\{j_1, j_2\}$ or $\{j_3, j_4\}$, then trivially any assignment that uses those sets will lead to a maximum assignment distance of $D$. Therefore, we only need to see what happens when the optimal fair solution places one center in each clique, and without loss of generality let us assume that the chosen centers are $\{j_1, j_3\}$. Focus now on $j_1$. If the optimal solution assigns $j_1$ to itself, i.e., $\phi(j_1) = j_1$, then $d(j_1,\phi(j_1)) = 0$. The latter implies that the fairness constraint for $j_2$ cannot be satisfied. Thus, the optimal must set $\phi(j_1) = j_3$, hence leading to a maximum assignment distance of $D$.
\end{proof}

\section{Approximation algorithms for \normalfont{\prob{-PP}} \textbf{and} \normalfont{\prob{-AG}}}\label{sec:alg}

Suppose that we are given an instance of \prob{} with $\alpha, k \geq 2$, and we are either solving \prob{-PP} or \prob{-AG}. In addition, recall that $R_j = \max_{j' \in \mathcal{S}_j}d(j,j')$, $R_m = \max_{j \in \mathcal{C}}R_j$ and $R^*$ denotes the value of the optimal solution for the corresponding problem. 

In this section we demonstrate a procedure that works under an explicitly given value $R$, with $R \geq R_m$. This process will either return a feasible solution $(S_R, \phi_R)$ with $\max_{j \in \mathcal{C}}d(j,\phi_R(j)) \leq 5R$, or an infeasibility message. The latter message indicates with absolute certainty that $R < R^*$. 

The aforementioned procedure suffices to yield the result of Theorem \ref{thm:alg-res}. Because $R^*$ is always the distance between two points in $\mathcal{C}$, the total number of possible values for it is only polynomial, specifically at most ${|\mathcal{C}| \choose 2}$. Hence, we can run the procedure for all such distances that are at least $R_m$, and in the end keep $(S_R, \phi_R)$ for the minimum guess $R$ for which we did not receive an infeasibility message. If $R_m \leq R^*$, then our returned solution is guaranteed to have value at most $5R^*$, because $R^*$ is one of the target values we tested. On the other hand, when $R_m > R^*$, the iteration with $R_m$ as the guess cannot return an infeasibility message, and thus it will provide a solution of value at most $5R_m$. As a side note, we mention that we can speed up the runtime of this approach by using binary search over the guesses $R$, instead of a naive brute-force search.  

Therefore, apart from the input instance, assume that we are also given a target value $R$ with $R \geq R_m$. Our framework begins by choosing an initial set of centers $S$. The full details of this step are presented in Algorithm \ref{alg-1}. Besides choosing this set $S$, Algorithm \ref{alg-1} also creates a partition $P_1, P_2, \hdots P_T$ of $S$ for some $T \leq |\mathcal{C}|$, and returns sets $G_c \subseteq \mathcal{C}$ for every $c \in S$. 

Initially, all point of $\mathcal{C}$ are considered uncovered ($U = \mathcal{C}$). The algorithm works by trying to expand the current set of centers $P_t$ as much as possible, via finding a new center that is currently uncovered and is within distance $3R$ from some center already placed in $P_t$. If no such point exists, then we never deal with $P_t$ again, and we move on to create $P_{t+1}$ by choosing an arbitrary uncovered point as the first center for it. In additional, every time a center $c$ is chosen, it covers all uncovered points that are within distance $2R$ from it, and these points constitute the set $G_c$. This process is repeated until all points get covered, i.e., until the set $U$ becomes empty.

\begin{algorithm}[t]
$S \gets \emptyset$, $U \gets \mathcal{C}$, $P_0 \gets \emptyset$, $t \gets 0$\;
\While {$U \neq \emptyset$} {
$Q \gets \{c \in U ~|~ \exists c' \in P_t \text{ such that } d(c,c') \leq 3R\}$\;
\uIf {$Q \neq \emptyset$} {
Choose a point $c \in Q$\;
$P_t \gets P_t \cup \{c\}$, $S \gets S \cup \{c\}$, $G_c \gets \{j \in U ~|~ d(j,c) \leq 2R\}$, $U \gets U \setminus G_c$\;
}
\Else {
Choose an arbitrary $c \in U$\;
$t \gets t + 1$\;
$P_t \gets \{c\}$, $S \gets S \cup \{c\}$, $G_c \gets \{j \in U ~|~ d(j,c) \leq 2R\}$, $U \gets U \setminus G_c$\;
}
}
Return the set $S$, the partition $P_1, P_2, \hdots P_t$ of $S$, and the sets $G_c$ for every $c \in S$\;
\caption{Choosing an initial set of centers}\label{alg-1}
\end{algorithm}

For every $c \in S$, let $t(c)$ be the index of the partition set $c$ belongs to, i.e., $c \in P_{t(c)}$. We also define $S_I = \{c \in S:~ |P_{t(c)}| = 1\}$ and $S_{N} = S \setminus S_I$. We interpret the centers of $S_I$ as being \emph{isolated}, since for each $c \in S_I$ its corresponding partition set contains only $c$, i.e., $P_{t(c)} = \{c\}$. On the other hand, the centers of $S_N$ are \emph{non-isolated}, in the sense of having $|P_{t(c)}| > 1$ for each $c \in S_N$. In addition, for every point $j \in \mathcal{C}$, let $\rho(j)$ the center of $S$ that covered $j$, i.e., $j \in G_{\rho(j)}$. Note that $d(j,\rho(j)) \leq 2R$. Finally, let $\mathcal{C}_I = \{j \in \mathcal{C}: ~ \rho(j) \in S_I\}$ and $\mathcal{C}_N = \mathcal{C} \setminus \mathcal{C}_I$, where $\mathcal{C}_I$ are the points that got covered by isolated centers, and $\mathcal{C}_N$ the points that got covered by non-isolated centers.

\begin{observation}\label{obs-1}
For every distinct $c,c' \in S$ we have $d(c,c') > 2R$.
\end{observation}

\begin{observation}\label{obs-2}
For every $c \in S_N$, there exists a different $c' \in S_N$ such that $d(c,c') \leq 3R$.
\end{observation}

\begin{observation}\label{obs-3}
The sets $G_c$ for all $c \in S$, induce a partition of $\mathcal{C}$.
\end{observation}

The three previous observations follow trivially from the definition of Algorithm \ref{alg-1}. However, Observation \ref{obs-2} is of particular importance, since it will allow us to carefully control the assignment distances of points later on, in a way that would satisfy the underlying fairness constraints. 

\begin{lemma}\label{choose-cntr}
For any $c \in S_I$, we have $d(j,j') > R$ for all $j \in G_c$ and all $j' \in \mathcal{C} \setminus G_c$. 
\end{lemma}

\begin{proof}
Focus on such a $c \in S_I$, and for the sake of contradiction assume that there exists a $j \in G_c$ and a $j' \in \mathcal{C} \setminus G_c$ for which $d(j,j') \leq R$. Let $c' \neq c$ the center of $S$ with $c'= \rho(j')$. 

At first, suppose that during the execution of Algorithm \ref{alg-1} $c$ entered $S$ before $c'$. Having $|P_{t(c)}| = 1$ means that when $P_{t(c)} =\{c\}$, the algorithm tried to find a point in $U$ within distance $3R$ from $c$ but failed. However, at that time $j'$ was still in $U$, because $j' \in G_{c'}$ and $c'$ entered $S$ after $c$. In addition $d(j',c) \leq d(j,j') + d(j,c) \leq 3R$, and thus we reached a contradiction.

Now assume that $c'$ entered $S$ before $c$. This implies that $t(c') < t(c)$, because $|P_{t(c)}| = 1$. When the algorithm stopped expanding $P_{t(c')}$, there was not any point of $U$ within distance $3R$ from a center of $P_{t(c')}$. However, at that moment $j$ was still in $U$, because $j \in G_c$ and $t(c') < t(c)$. In addition $d(j,c') \leq d(j,j') + d(j',c') \leq 3R$, and so we once again reach a contradiction.
\end{proof}

By using Lemma \ref{choose-cntr} and the fact that $R \geq R_m$, we immediately get the following.
\begin{corollary}\label{balls-cont}
For every $c \in S_I$, we have $\mathcal{S}_j \subseteq G_c \subseteq \mathcal{C}_I$ for all $j \in G_c$.
\end{corollary}

\begin{corollary}\label{balls-cont-2}
For every $j \in \mathcal{C}_N$, we have $\mathcal{S}_j \cap \mathcal{C}_I = \emptyset$ and hence $\mathcal{S}_j \subseteq \mathcal{C}_N$.
\end{corollary}

In words, Corollary \ref{balls-cont} says that the similarity set of a point $j \in \mathcal{C}_I$ is completely contained in $G_{\rho(j)}$, where of course $\rho(j) \in S_I$ and $G_{\rho(j)} \subseteq \mathcal{C}_I$. Similarly, Corollary \ref{balls-cont-2} says that the similarity set of a point $j \in \mathcal{C}_N$ is completely contained in $\mathcal{C}_N$. 

After computing the set of centers $S$, our approach proceeds by constructing the appropriate assignment function. This will occur in two steps. The first step takes care of the points in $\mathcal{C}_I$, by choosing a new set of centers $S'_I \subseteq \mathcal{C}_I$, and by constructing an assignment $\phi_I: \mathcal{C}_I \mapsto S'_I$. The second step handles the points of $\mathcal{C}_N$ via a mapping $\phi_N: \mathcal{C}_N \mapsto S_N$. This is well-defined, since $\mathcal{C}_I \cap \mathcal{C}_N = \emptyset$. Note now that due to Corollary \ref{balls-cont}, the fairness constraint of a point $j \in \mathcal{C}_I$ is only affected by $\phi_I$, since $\mathcal{S}_j \subseteq G_{\rho(j)} \subseteq \mathcal{C}_I$ and $\mathcal{C}_I \cap \mathcal{C}_N = \emptyset$. Similarly, due to Corollary \ref{balls-cont-2}, the fairness constraint of a $j \in \mathcal{C}_N$ is only affected by $\phi_N$, since $\mathcal{S}_j \subseteq \mathcal{C}_N$ and $\mathcal{C}_I \cap \mathcal{C}_N = \emptyset$. Therefore, we can study the satisfaction of fairness constraints separately on $\mathcal{C}_I$ for $\phi_I$, and on $\mathcal{C}_N$ for $\phi_N$. 

Algorithm~\ref{alg-2} demonstrates the details of the first assignment step. The algorithm operates by trying to ``guess'' if the optimal solution uses exactly one center inside each $G_c$ for $c \in S_I$. If it does, so will our algorithm. If not, then our approach will open exactly two centers, and will subsequently construct an assignment that will satisfy the appropriate fairness constraint.

\begin{algorithm}[t]
$S'_I \gets \emptyset$\;
\For {every $c \in S_I$} {

Check if there exists any $j \in G_c$, such that assigning all points of $G_c$ to $j$ would result in the appropriate fairness constraint being satisfied for each $j' \in G_c$. Note that checking this feasibility condition is well-defined, since Corollary \ref{balls-cont} gives $\mathcal{S}_{j'} \subseteq G_c$ for all $j' \in G_c$. If such a $j$ exists, set $S'_I \gets S'_I \cup \{j\}$ and $\phi_I(j') \gets j$ for all $j' \in G_c$\;

If you could not find such a $j$, use the algorithm of Lemma \ref{thm:a=2-aux} on the points of $G_c$. This will return two points $c_1, c_2 \in G_c$ and an assignment $\phi: G_c \mapsto \{c_1, c_2\}$. Then set $S'_I \gets S'_I \cup \{c_1, c_2\}$ and $\phi_I(j') \gets \phi(j')$ for all $j' \in G_c$\;
}
Return $S'_I$ and $\phi_I$\;
\caption{Assignment for the points of $\mathcal{C}_I$}\label{alg-2}
\end{algorithm}

\begin{lemma}\label{iso-feas}
After the execution of Algorithm \ref{alg-2}, for every $j \in \mathcal{C}_I$ we have that the constructed assignment $\phi_I$ will \textbf{1)} satisfy $j$'s fairness constraint, and \textbf{2)} guarantee $d(j,\phi_I(j)) \leq 4R$.
\end{lemma}

\begin{proof}
At first, due to Observation \ref{obs-3}, Algorithm \ref{alg-2} sets the value $\phi_I(j)$ for each $j \in \mathcal{C}_I$ exactly once. In addition, we know that for every $j \in \mathcal{C}_I$, all points of $\mathcal{S}_j$ will have their assignment set in the same iteration of Algorithm \ref{alg-2}, since $\rho(j) \in S_I$ and by Corollary \ref{balls-cont} we have $\mathcal{S}_j \subseteq G_{\rho(j)}$.

For a point $j \in \mathcal{C}_I$, when $\rho(j)$ is considered by Algorithm \ref{alg-2} there are two possible scenarios. In the first we have $|S'_I \cap G_{\rho(j)}| = 1$. If that happens, all points of $G_{\rho(j)}$ are assigned to the only point of $S'_I \cap G_{\rho(j)}$, and because of the first check of the algorithm we are also sure that the fairness constraint of all of them is satisfied. Otherwise, we have $|S'_I \cap G_{\rho(j)}| = 2$, as a result of running the algorithm of Lemma \ref{thm:a=2-aux} on $G_{\rho(j)}$. By using the assignment guarantees of that algorithm, it is easy to see that the fairness constraints for all $j' \in G_{\rho(j)}$ will again be satisfied. Hence, in both cases the corresponding fairness constraint is satisfied for $j$.

Finally, $d(j,\phi_I(j)) \leq d(j,\rho(j)) + d(\phi_I(j), \rho(j)) \leq 4R$, since $\phi_I(j) \in G_{\rho(j)}$ in each case.
\end{proof}

\begin{lemma}\label{iso-k}
If $R \geq R^*$, then after the execution of Algorithm \ref{alg-2} we will have $|S'_I| + |S_N| \leq k$.
\end{lemma}

\begin{proof}
Let $S^*$ be the optimal set of centers, and $\phi^*$ the corresponding optimal assignment. The following two statements rely on the fact that $R\geq R^*$. First, by Observation \ref{obs-1} note that for two distinct points $c,c' \in S_N$ we must have $\phi^*(c) \neq \phi^*(c')$. Second, due to Lemma \ref{choose-cntr} we also have $\phi^*(c) \notin \mathcal{C}_I$ for every $c \in S_N$. The two previous statements imply $|S_N| \leq |S^* \setminus \mathcal{C}_I|$.

Now focus on $S^* \cap \mathcal{C}_I$, and see that $|S^* \cap \mathcal{C}_I| = \sum_{c \in S_I}|S^* \cap G_c|$ due to Observation \ref{obs-3} and the definition of $\mathcal{C}_I$. Further, due to Lemma \ref{choose-cntr} and the fact that $R \geq R^*$, we have that $|S^* \cap G_c| \geq 1$ for every $c \in S_I$. If $|S^* \cap G_c| = 1$, then Lemma \ref{choose-cntr} implies that the optimal solution assigns all points of $G_c$ to the unique point of $S^* \cap G_c$. This assignment is obviously feasible, and thus the first part of Algorithm~\ref{alg-2} can identify it and give $|S'_I \cap G_c| = 1$. Otherwise, if $|S^* \cap G_c| \geq 2$, then Algorithm \ref{alg-2} ensures that $|S'_I \cap G_c| \leq 2$. Therefore, we get 
\begin{align}
    |S'_I| = \sum_{c \in S_I}|S'_I \cap G_c| \leq \sum_{c \in S_I}|S^* \cap G_c| = |S^* \cap \mathcal{C}_I| \notag
\end{align}

Putting everything together yields 
\begin{align}
    &|S'_I| + |S_N| \leq |S^* \cap \mathcal{C}_I| + |S^* \setminus \mathcal{C}_I| = |S^*| \leq k \qedhere \notag
\end{align}
\end{proof}

Using the contrapositive of Lemma \ref{iso-k}, we see that if $|S'_I| + |S_N| > k$ then $R < R^*$, and hence we can safely return as our answer an infeasibility message.

Before we proceed to the second step of our assignment process, we need some extra notation. For each $c \in S_N$ define $H^1_c = \{j \in \mathcal{C}_N ~|~ d(j,c) \leq R\}$ and $H^2_c = G_c \setminus \big{(}\bigcup_{c' \in S_N}H^1_{c'}\big{)}$. Combining Observation \ref{obs-1}, Observation \ref{obs-3} and the way we constructed the sets $H^1_c, H^2_c$, it is easy to see that for each $j \in \mathcal{C}_N$ exactly one of the following two cases will hold.
\begin{itemize}
    \item The point $j$ belongs to \emph{exactly one} $H^1_c$ for some $c \in S_N$. In addition, $j$ clearly does not belong to any set $H^2_{c'}$ for $c' \in S_N$. In this case, we call $j$ a \emph{type-1} point, and we set $\pi(j) = c$.
    \item The point $j$ belongs to $H^2_{\rho(j)}$. In addition, $j$ does not belong to any $H^1_c$ for $c \in S_N$, and also it does not belong to any $H^2_c$ with $c \neq \rho(j)$. Here we call $j$ a \emph{type-2} point, and set $\pi(j) = \rho(j)$.
\end{itemize}
Further, let $\mathcal{C}^1_N = \{j \in \mathcal{C}_N ~|~ j \text{ is a type-1 point}\}$ and $\mathcal{C}^2_N = \{j \in \mathcal{C}_N ~|~ j \text{ is a type-2 point}\}$. Therefore, $\mathcal{C}^1_N \cap \mathcal{C}^2_N = \emptyset$ and $\mathcal{C}^1_N \cup \mathcal{C}^2_N = \mathcal{C}_N$. Finally, the definition of a type-2 point implies:
\begin{observation}\label{obs-4}
For all $j \in \mathcal{C}^2_N$, we have $d(j,\pi(j)) \leq 2R$ and $d(j,c) > R$ for all $c \in S_N$.
\end{observation}
The distinction between type-1 and type-2 points is necessary for satisfying the fairness constraints. Notice that by construction of $S_N$ type-1 points are more ``privilleged'', since they have an available center within distance at most $R$ from them. On the other hand, type-2 points do not have such an advantage. Therefore, the assignment process should be aware of this discrepancy, so it can favor type-2 points in a controlled way that will satisfy everyone's fairness constraint. 

Algorithm \ref{alg-3} demonstrates the full details of constructing the assignment $\phi_N : \mathcal{C}_N \mapsto S_N$. The high-level intuition behind it follows. At first, we try to provide each point $j$ with an assignment distance in the range $[R,5R]$, something that is possible due to Observation \ref{obs-2}. However, since $\alpha$ might be less than $5$, we are very careful in how we handle the assignment of similar points. The latter is achieved by considering type-1 and type-2 points independently, in a manner that is aware of where the potential similar points of each type may be.

\begin{algorithm}[t]

\For {every $j \in \mathcal{C}_N$} {

\If {$j \in \mathcal{C}^1_N$} {
$\phi_N(j) \gets \argmin_{c \in (S_N \setminus \{\pi(j)\})}d(j,c)$ \tcp*{Case (A)}
}
\If {$j \in \mathcal{C}^2_N$} {    
\uIf {$\exists c \in (S_N \setminus \{\pi(j)\}): d(j,c) \leq 2R$}{
$\phi_N(j) \gets \argmax_{c' \in S_N: d(j,c') \leq 2R}d(j,c')$ \tcp*{Case (B)}
}
\Else{
$\phi_N(j) \gets \argmin_{c' \in (S_N \setminus \{\pi(j)\})}d(j,c')$ \tcp*{Case (C)}
}
}
}
Return the assignment $\phi_N: \mathcal{C}_N \mapsto S_N$\;
\caption{Assignment for the points of $\mathcal{C}_N$}\label{alg-3}
\end{algorithm}

\begin{lemma}\label{alg-lem-1}
For any point $j \in \mathcal{C}_N$ we have:
\begin{itemize}
    \item $d(j,\pi(j)) \leq R < d(j,\phi_N(j)) \leq 4R$, if $j$ gets assigned to $\phi_N(j)$ according to \textbf{Case (A)}.
    \item $R < d(j,\pi(j)) \leq d(j,\phi_N(j)) \leq 2R$, if $j$ gets assigned to $\phi_N(j)$ according to \textbf{Case (B)}.
    \item $d(j,\pi(j)) \leq 2R < d(j,\phi_N(j)) \leq 5R$, if $j$ gets assigned to $\phi_N(j)$ according to \textbf{Case (C)}.
\end{itemize}
\end{lemma}

\begin{proof}
In \textbf{Case (A)} $d(j, \pi(j)) \leq R$ since $j \in H^1_{\pi(j)}$. Also, from the definition of type-1 points, there does not exist any center in $S_N \setminus \{\pi(j)\}$ that is within distance at most $R$ from $j$, and hence $d(j,\phi_N(j)) > R \geq d(j, \pi(j))$. In addition, Observation \ref{obs-2} ensures that there exists a $c \in S_N \setminus \{\pi(j)\}$ such that $d(\pi(j),c) \leq 3R$. Therefore, $d(j,\phi_N(j)) \leq d(j,c) \leq d(j,\pi(j)) + d(\pi(j),c) \leq 4R$. 

The assignment guarantee for \textbf{Case (B)} follows trivially from Observation \ref{obs-4}, and the way the algorithm operates in that situation.

In \textbf{Case (C)} we have $d(j,c) > 2R \geq d(j,\pi(j))$ for all $c \in S_N \setminus \{\pi(j)\}$. In addition, Observation \ref{obs-2} ensures that there exists a $c' \in S_N \setminus \{\pi(j)\}$ such that $d(\pi(j),c') \leq 3R$. Hence, $d(j,\phi_N(j)) \leq d(j,c') \leq d(j,\pi(j)) + d(\pi(j),c') \leq 5R$, where $d(j,\pi(j)) \leq 2R$ follows from Observation \ref{obs-4}.
\end{proof}

Lemma \ref{alg-lem-1} immediately gives an upper bound of $5R$ for the maximum assignment distance. However, it is the rest of the inequalities shown there that allow us to prove satisfaction of the fairness constraints by $\phi_N$. This is achieved in the following Lemma.

\begin{lemma}\label{alg-feas-aux}
For all $j \in \mathcal{C}_N$, we have $d(j,\phi_N(j)) \leq \alpha \cdot d(j',\phi_N(j'))$ for all $j' \in \mathcal{S}_j$.
\end{lemma}

\begin{proof}
Suppose we have some $j \in \mathcal{C}_N$ and some $j' \in \mathcal{S}_j$. The proof of the statement will be based on an exhaustive case analysis. Before we proceed, we mention two inequalities that we will repeatedly use. At first, $d(j,j') \leq d(j', \phi_N(j'))$, because $d(j,j') \leq R_m \leq R$ and by Lemma \ref{alg-lem-1} we have $d(j',\phi_N(j')) > R$. Moreover, $d(j', \pi(j')) \leq d(j', \phi_N(j'))$, again by using Lemma \ref{alg-lem-1}.

\begin{itemize}
    \item Suppose that $j$ is a type-1 point and $j'$ is also a type-1 point.
    
    At first let $\pi(j) \neq \pi(j')$. Then $j$ can potentially be assigned to $\pi(j')$, and therefore we have $d(j,\phi_N(j)) \leq d(j,\pi(j')) \leq d(j,j') + d(j',\pi(j')) \leq 2d(j',\phi_N(j')) \leq \alpha \cdot d(j',\phi_N(j'))$.
    
    Now let $\pi(j) = \pi(j')$. Because $j'$ is a type-1 point and gets assigned according to \textbf{Case (A)}, we know that $\phi_N(j') \neq \pi(j)$, Hence $j$ can potentially be assigned to $\phi_N(j')$. Therefore, $d(j,\phi_N(j)) \leq d(j,\phi_N(j')) \leq d(j,j') + d(j',\phi_N(j')) \leq 2d(j',\phi_N(j')) \leq \alpha \cdot d(j',\phi_N(j'))$. 
    
    \item Suppose that $j$ is a type-1 point and $j'$ is a type-2 point.
    
    At first assume $j'$ received its assignment via \textbf{Case (C)}. Then, by Lemma \ref{alg-lem-1} we know that $d(j', \phi_N(j')) > 2R$. In addition, again by Lemma \ref{alg-lem-1}, we have $d(j, \phi_N(j)) \leq 4R$. Thus, $d(j, \phi_N(j)) \leq 2d(j', \phi_N(j')) \leq \alpha \cdot d(j', \phi_N(j'))$. 
    
    Now assume that $j'$ received its assignment through \textbf{Case (B)}. Therefore, there exists $c \in S \setminus \{\pi(j')\}$ with $d(j', c) \leq 2R$. By the way \textbf{Case (B)} works and Observation \ref{obs-4}, we also have $d(j',\phi_N(j')) \geq \max(d(j',\pi(j')), d(j',c))$. Let us now see what happens when $\pi(j')=\pi(j)$. Then $c \neq \pi(j)$, and thus $j$ can potentially be assigned to $c$. Therefore, $d(j,\phi_N(j)) \leq d(j,c) \leq d(j,j') + d(j',c) \leq d(j,j') + d(j', \phi_N(j')) \leq 2 d(j', \phi_N(j')) \leq \alpha \cdot d(j', \phi_N(j'))$. On the other hand, if $\pi(j') \neq \pi(j)$, then $j$ can potentially get assigned to $\pi(j')$, and thus have $d(j,\phi_N(j)) \leq d(j,\pi(j')) \leq d(j,j') + d(j',\pi(j')) \leq \alpha \cdot d(j', \phi_N(j'))$.
    
    \item Suppose that $j$ is a type-2 point, and also gets its assignment via \textbf{Case (B)}. By Lemma \ref{alg-lem-1} we have $d(j,\phi_N(j)) \leq 2R$ and $d(j',\phi_N(j')) > R$. For $\alpha \geq 2$ the statement trivially follows.
    
    \item Suppose that $j$ is a type-2 point, $j'$ is a type-1 point, and $j$ gets its assignment via \textbf{Case (C)}.
    
    At first, assume that $\phi_N(j') \neq \pi(j)$. In this case $j$ can potentially get assigned to $\phi_N(j')$, and  $d(j,\phi_N(j)) \leq d(j,\phi_N(j')) \leq d(j,j') + d(j',\phi_N(j')) \leq \alpha \cdot d(j',\phi_N(j'))$. 
    
    Now assume that $\phi_N(j') = \pi(j)$. Because $j'$ is a type-1 points and so $\phi_N(j') \neq \pi(j')$, we can infer that $\pi(j) \neq \pi(j')$. Also, $d(j,\pi(j')) \leq d(j,j') + d(j',\pi(j')) \leq 2R$. However, the latter contradicts the assumption that $j$ got its assignment according to \textbf{Case (C)}. Therefore, we know that $\phi_N(j') \neq \pi(j)$ necessarily.
    
    \item Suppose that both $j, j'$ are type-2 points, and $j$ gets its assignment via \textbf{Case (C)}.
    
    At first, assume $\pi(j') \neq \pi(j)$. Then $j$ can potentially get assigned to $\pi(j')$, and therefore $d(j,\phi_N(j)) \leq d(j,\pi(j')) \leq d(j,j') + d(j',\pi(j')) \leq 2d(j', \phi_N(j')) \leq \alpha \cdot d(j', \phi_N(j'))$.
   
    Now let $\pi(j') = \pi(j)$. To begin with, assume that there exists a $c \in S \setminus \{\pi(j)\}$ such that $d(j',c) \leq 2R$. Moreover, because $c \neq \pi(j)$, $j$ can potentially get assigned to $c$, and thus $d(j,\phi_N(j)) \leq d(j,c) \leq d(j,j') + d(j',c) \leq d(j,j') + d(j',\phi_N(j')) \leq \alpha \cdot d(j',\phi_N(j'))$. To get $d(j',c) \leq d(j',\phi_N(j'))$ we simply used the way \textbf{Case (B)} works. Finally, suppose that $\forall c \in S \setminus \{\pi(j)\}$ we have $d(j',c) > 2R$. Then $\phi_N(j') \neq \pi(j)$ and thus $j$ can potentially get assigned to $\phi_N(j')$. Therefore, $d(j,\phi_N(j)) \leq d(j,\phi_N(j')) \leq d(j,j') + d(j', \phi_N(j')) \leq \alpha \cdot d(j,\phi_N(j'))$. \qedhere
\end{itemize}
\end{proof}

Combining Lemmas \ref{alg-lem-1} and \ref{alg-feas-aux} we immediately get the following.
\begin{lemma}\label{non-iso-feas}
After the execution of Algorithm \ref{alg-3}, for every $j \in \mathcal{C}_N$ we have that the constructed assignment $\phi_N$ will \textbf{1)} satisfy $j$'s fairness constraint, and \textbf{2)} guarantee $d(j,\phi_N(j)) \leq 5R$.
\end{lemma}

Finally, by combining Lemmas \ref{iso-feas}, \ref{non-iso-feas} and \ref{iso-k} with the fact that the number of centers we use is $|S'_I|+|S_N|$, we see that we provide a procedure that for a guess $R \geq R_m$ works as follows. It either returns a feasible solution with maximum assignment distance $5R$, or returns an infeasibility message that indicates $R < R^*$. As mentioned earlier, this concludes the proof of Theorem \ref{thm:alg-res}.

\subsection{Cases with bounded PoF}\label{sec:PoF}

As we have already shown in Theorem \ref{thm:unbound-PoF}, the Price of Fairness for both variants of \prob{} can in general be unbounded. However, we are going to demonstrate that in certain scenarios we can provably achieve solutions with bounded PoF. This means that the objective function value of the solution will be comparable to the optimal unfair value, up to some constant factor.

For the given instance of \prob{}, let $R^*_{unf}$ be the value of the optimal $k$-center solution, when we drop the fairness constraints from the problem's requirements.

The first scenario we study is a small modification to our main algorithm, which consists of only changing Algorithm \ref{alg-2}, and thus the construction of $S'_I$ and $\phi_I$. Specifically, if for some $c \in S_I$ we have $|G_c| = 1$, then we use $c$ as a center and set $\phi_I(c) = c$. If for some $c \in S_I$ we have $|G_c| \geq 2$, then we immediately use the procedure of Lemma \ref{thm:a=2-aux}, without checking if only one point of $G_c$ can yield a feasible solution. This modification yields the results of Theorem \ref{thm:pof-1}. 

\begin{proof}[Proof of Theorem \ref{thm:pof-1}]
At first, note that due to Assumption \ref{icml-asm} we have $R_m \leq \psi R^*_{unf}$, and hence the guess $\psi R^*_{unf}$ will be among the ones we test; recall that we test guesses $R \in [R_m, \max_{j,j'}d(j,j')]$. Assume for now that $\psi \geq 1$. For the iteration where the guess is $\psi R^*_{unf}$, Lemmas \ref{iso-feas} and \ref{non-iso-feas} will clearly hold, thus ensuring that the returned solution has value $5\psi R^*_{unf}$, and the constructed assignment satisfies all fairness constraints. The only thing left to analyze is the number of centers we end up using when the guess is $\psi R^*_{unf}$. Combining Observation \ref{obs-1}, the fact that $\psi \geq 1$ and the fact that the optimal unfair solution uses at most $k$ centers, we immediately get $|S_I| + |S_N| \leq k$. On the other hand, observe that the number of centers our modified algorithm uses is in the worst case is $2|S_I| + |S_N|$, and therefore at most $2k$. 

When $\psi < 1$, then we know for sure that $R^{*}_{unf}$ will be among the tested guesses. In that case, the previous analysis follows through, with the only difference being that now the maximum radius of our returned solution would be $5 R^*_{unf}$.

Finally, to conclude the proof, we just need to make sure that for a radius guess that resulted in $|S_I| + |S_N| > k$, we return an infeasibility message.
\end{proof}

Although the result of Theorem \ref{thm:pof-1} is interesting in the sense of showing a scenario with bounded PoF, it is not a true approximation algorithm, because we end up violating the number of chosen centers by a multiplicative factor of $2$. We are now going to demonstrate another case, where we achieve a true feasible solution to \prob{-AG}, that additionally enjoys a bounded PoF.

In this scenario, the radius $R_j$ is the same for all points, i.e., for all $j \in \mathcal{C}$ we have $R_j = R_d$ for some $R_d$. Our algorithm here is actually identical to the one presented in the previous subsection, and the difficulty in proving Theorem \ref{thm:pof-2} for it lies only on the analysis.

\begin{proof}[Proof of Theorem \ref{thm:pof-2}]
At first, note that due to Assumption \ref{icml-asm} we have $R_d \leq \psi R^*_{unf}$, and hence the guess $\psi R^*_{unf}$ will be among the ones we test. As in the proof of Theorem \ref{thm:pof-1} we can solely focus on the $\psi \geq 1$ case. For the iteration of $\psi R^*_{unf}$, Lemma \ref{non-iso-feas} clearly holds. We will show that Lemma \ref{iso-feas} will hold as well, and furthermore that Algorithm \ref{alg-2} will always pick \emph{just one} center in each $G_c$ for $c \in S_I$. This will immediately imply that the returned solution has value at most $5 \psi R^*_{unf}$, all constraints (\ref{cons:AG}) are satisfied, and the centers we end up using are exactly $|S_I| + |S_N|$. Finally, note that by Observation \ref{obs-1}, the fact that $\psi \geq 1$ and the fact that the optimal unfair solution uses at most $k$ centers, we will also have $|S_I| + |S_N| \leq k$.

Therefore, all we need to show is that for every $c \in S_I$, Algorithm \ref{alg-2} is able to find exactly one center that satisfies constraint (\ref{cons:AG}) for all $j \in G_c$ (recall that $\mathcal{S}_j \subseteq G_c$). To do that, we prove that there exists an $x \in G_c$, such that that for all $j \in G_c$ we have $\sum_{j' \in \mathcal{S}_{j}}d(j,j') \leq \sum_{j' \in \mathcal{S}_{j}}d(j',x)$. This suffices to prove the desired statement. To see why, assume that we make $x$ the chosen center of $G_c$, and assign all points of $G_c$ to it. Then for any point $j \in G_c$ and any $j' \in \mathcal{S}_{j}$ we have $d(j,x) \leq d(j,j') + d(j',x)$ by the triangle inequality. Summing over all $j' \in \mathcal{S}_{j}$ and using the property of $x$ gives:
\begin{align}
    d(j,x) &\leq \frac{1}{|\mathcal{S}_{j}|}\sum_{j' \in \mathcal{S}_{j}}d(j,j') + \frac{1}{|\mathcal{S}_{j}|}\sum_{j' \in \mathcal{S}_{j}}d(j',x) \leq \frac{2}{|\mathcal{S}_{j}|}\sum_{j' \in \mathcal{S}_{j}}d(j',x) \leq \frac{\alpha}{|\mathcal{S}_{j}|}\sum_{j' \in \mathcal{S}_{j}}d(j',x) \notag 
\end{align}

For the sake of contradiction, assume now that for all $x \in G_c$ there exists a point $j \in G_c$ such that $\sum_{j' \in \mathcal{S}_{j}}d(j,j') > \sum_{j' \in \mathcal{S}_{j}}d(j',x)$. Based on this, we can create a dependency graph, where every point of $G_c$ is a vertex, and there is a directed edge from $x$ to $j$ if $\sum_{j' \in \mathcal{S}_{j}}d(j,j') > \sum_{j' \in \mathcal{S}_{j}}d(j',x)$. The assumption for the contradiction implies that this dependency graph will contain a directed cycle $x_1, x_2, \hdots, x_r$, for which we have $\sum_{j' \in \mathcal{S}_{x_t}}d(x_t,j') > \sum_{j' \in \mathcal{S}_{x_t}}d(j',x_{t-1})$ for all $t \in [2,r+1]$, assuming that $x_{r+1}= x_1$. If we add all the above inequalities we get
\begin{align}
    \sum^{r+1}_{t=2} \sum_{j' \in \mathcal{S}_{x_t}}d(j',x_t) > \sum^{r+1}_{t=2} \sum_{j' \in \mathcal{S}_{x_t}}d(j',x_{t-1}) \notag
\end{align}
Now focus on any $j'$, and see that its contribution in the LHS of the above inequality is $A = \sum_{t: j' \in \mathcal{S}_{x_t}}d(j',x_t)$, and in the RHS is $B = \sum_{t: j' \in \mathcal{S}_{x_t}}d(j',x_{t-1})$. We argue that $A > B$ is impossible, and thus reach a contradiction. If $A > B$, we can first subtract from both $A$ and $B$ the common terms appearing in the sums. Then, in what is left of $A$ we will only have terms $d(j',x_t)$ being added, for $j' \in \mathcal{S}_{x_t}$. In what is left of $B$ we will only have terms $d(j',x_{t-1})$ being added, but for which $j' \notin \mathcal{S}_{x_{t-1}}$. Note also that the number of leftover terms is the same in both $A$ and $B$. Moreover, since the similarity radius is the same for all points, for any two points $z,y \in G_c$ with $j' \in \mathcal{S}_z$ and $j' \notin \mathcal{S}_y$, we have $d(j',z) < d(j',y)$. Hence we reached the desired contradiction.
\end{proof}
\section{Solving the assignment problem}\label{sec:assgn}

In this section we address the assignment problem for \prob{}. Specifically, for an instance with $\alpha, k \geq 2$, if we are given the set of centers $S^*$ used in the optimal solution, can we efficiently find the optimal assignment $\phi^*: \mathcal{C} \mapsto S^*$? In other words, if $R^*$ is the value of the optimal solution, we want to compute $\phi^*$ such that \textbf{1)} $\phi^*$ satisfies the appropriate fairness constraint for all points, and \textbf{2)} for every $j \in \mathcal{C}$ we have $d(j,\phi^*(j)) \leq R^*$. In what follows, we demonstrate in full detail a procedure that achieves this for \prob{-PP}. A similar process can handle \prob{-AG}, but for the sake of not repeating the same arguments, we are only going to sketch this.

Before we proceed with our assignment algorithm for \prob{-PP}, note that w.l.o.g. we can always assume that the optimal value $R^*$ is known. This is because there are only polynomially many options for it, and thus we can efficiently guess the optimal one. Our process is presented in Algorithm \ref{alg-4}, and it works iteratively. The high-level idea is that it always maintains an assignment of value at most $R^*$, and in each iteration it corrects one violated fairness constraint. As we show later, a polynomial number of iterations suffices in order to reach a feasible assignment.

\begin{algorithm}[t]
For every $j \in \mathcal{C}$ set $\phi(j) \gets \argmax_{i \in S^*: d(i,j) \leq R^*}d(i,j)$\;
\While{there exists a $j \in \mathcal{C}$ with a $j' \in \mathcal{S}_j$ such that $d(j, \phi(j)) > \alpha d(j', \phi(j')$} {
Find such a pair $j \in \mathcal{C}$ and $j' \in \mathcal{S}_j$\;
Let $\Delta_{j,j'} = \{i \in S^*: d(i,j) < d(j,\phi(j)) \text{ and } d(i,j) \leq \alpha d(j', \phi(j'))\}$\;
Set $\phi(j) \gets \argmax_{i \in \Delta_{j,j'}}d(i,j)$\;
}
Return $\phi$\;
\caption{Solving the assignment problem for \prob{-PP}}\label{alg-4}
\end{algorithm}

\begin{lemma}\label{s4-lem-1}
Every time the condition of the while loop in Algorithm \ref{alg-4} is checked, we have $d(\phi(j),j) \geq d(\phi^*(j),j)$ for every $j \in \mathcal{C}$.
\end{lemma}

\begin{proof}
We are going to prove this via induction. For the first time we check the condition, the statement is obviously true by the way we initialized the mapping $\phi$ before the start of the loop, and the fact that $d(j,\phi^*(j)) \leq R^*$ for all $j \in \mathcal{C}$.

Consider now the $t^{\textbf{th}}$ time we check the condition, for which by the inductive hypothesis the statement of the lemma holds. If at that time no violated fairness constraint is found, then we are done. Hence, we need to focus on the case where the main body of the while loop is executed, and show that after the changes that occur in $\phi$, the statement will still be satisfied for the $(t+1)^{\textbf{th}}$ time we will check the condition.

Let $j_t$ be the point chosen at that iteration, with $j'_t \in \mathcal{S}_{j_t}$ the point with $d(j_t, \phi(j_t)) > \alpha d(j'_t, \phi(j'_t))$. By the inductive hypothesis we have $d(j'_t, \phi(j'_t)) \geq d(j'_t, \phi^*(j'_t))$. Combining the two previous inequalities gives $d(j_t,\phi(j_t)) > \alpha d(j'_t, \phi^*(j'_t))$. Now because the optimal assignment satisfies $d(j_t, \phi^*(j_t)) \leq \alpha d(j'_t, \phi^*(j'_t))$, we finally get $d(j_t, \phi(j_t)) > d(j_t, \phi^*(j_t))$. In addition, we have $d(j_t, \phi^*(j_t)) \leq \alpha d(j'_t, \phi^*(j'_t)) \leq \alpha d(j'_t, \phi(j'_t))$. Therefore, we see that $\phi^*(j_t) \in \Delta_{j_t,j'_t}$. Let now $\phi'(j_t)$ be the updated assignment for $j_t$ after the end of the iteration. From the way we update the assignment for $j_t$ and the fact that $\phi^*(j_t) \in \Delta_{j_t,j'_t}$, we infer that $d(\phi'(j_t),j_t) \geq d(\phi^*(j_t),j_t)$.
\end{proof}

\begin{theorem}
Algorithm \ref{alg-4} terminates within $|\mathcal{C}| |S^*|$ iterations, and the final assignment $\phi$ satisfies: \textbf{1)} $d(j,\phi(j)) \leq R^*$ for all $j \in \mathcal{C}$, and \textbf{2)} $d(j,\phi(j)) \leq \alpha d(j',\phi(j')$ for all $j \in \mathcal{C}$ and $j' \in \mathcal{S}_j$.
\end{theorem}

\begin{proof}
From the condition of the while loop we know that when the algorithm terminates, the fairness constraints will be satisfied by the mapping $\phi$. Also, because we never assign a point to a center that is further than $R^*$ from it, we know that $\phi$ achieves the optimal value.

Now we are going to count the total possible number of iterations. We do that by considering how many times we changed the assignment of every single point $j$, i.e., how many times an iteration tried to fix one of $j$'s violated constraints. By Lemma \ref{s4-lem-1}, we see that for any $j$ the minimum possible assignment distance we can provide to it is $d(j,\phi^*(j))$. Observe that if at any moment $d(j,\phi(j)) = d(j,\phi^*(j))$, then Lemma \ref{s4-lem-1} guarantees that $j$'s assignment will never change again. This is because for every $j' \in \mathcal{S}_j$ we always have $d(j',\phi(j')) \geq d(j',\phi^*(j'))$, and thus using the properties of the optimal assignment we get $d(j,\phi(j)) = d(j,\phi^*(j)) \leq \alpha d(j',\phi^*(j')) \leq \alpha d(j',\phi(j'))$. 

On the other hand, if at some point $d(j,\phi(j)) > d(j,\phi^*(j))$, then one of $j$'s fairness constraints might be violated, and hence we might end up using an iteration to fix it. In this case, let $j' \in \mathcal{S}_j$ the point causing the problematic situation. In addition, note that Lemma \ref{s4-lem-1} and the properties of the optimal solution ensure that $d(j_t, \phi^*(j_t)) \leq \alpha d(j'_t, \phi^*(j'_t)) \leq \alpha d(j'_t, \phi(j'_t))$. Thus, for this iteration $\phi^*(j) \in \Delta_{j,j'}$, and the new assignment distance of $j$ will be strictly smaller than the one it had at the beginning of the iteration. Thus, $j$ can be chosen in at most $|S^*|$ iterations.
\end{proof}

The assignment procedure for \prob{-AG} is almost identical to Algorithm \ref{alg-4}, with the only difference being that we should instead be looking for violated constraints (\ref{cons:AG}). In addition, the analysis of that algorithm remains identical to that of Algorithm \ref{alg-4}.
\section{Experimental evaluation}\label{sec:exps}

We implemented all algorithms in Python 3.8 and ran our experiments on Intel Xeon (Ivy Bridge) E3-12 @ 2.4 GHz with 20 cores and 96 GB 1200 MHz DDR4 memory. Our code can be found \href{https://github.com/chakrabarti/equitable_clustering}{here}.

\textbf{Datasets:} We used 5 datasets from the UCI Machine Learning Repository \cite{Dua:2019}, namely: \textbf{(1)} Bank-4,521 points \cite{moro}, \textbf{(2)} Adult-32,561 points \cite{kohavi}, \textbf{(3)} Creditcard-30,000 points \cite{yeh}, \textbf{(4)} Census1990-2,458,285 points \cite{meek} and \textbf{(5)} Diabetes-101,766 points \cite{strack}. From Adult, Creditcard, Census and Diabetes we uniformly subsampled $25,000$ points, and performed our experiments with respect to those sampled sets. In order to construct the distances between points, we removed non-numeric features, standardized each of the remaining features, took the Euclidean distances between these modified points, and then normalized the distances to be in $[0,1]$ for each dataset (by dividing the distances for a given dataset by the maximum distance between any two points). 

\textbf{Algorithms:} We first implemented the two versions of the algorithm of Theorem \ref{thm:alg-res}, one solving \prob{-AG} and the other \prob{-PP}. We call Alg-AG the variant solving \prob{-AG}, and Alg-PP the variant solving \prob{-PP}. Furthermore, we implemented the algorithm of Theorem \ref{thm:pof-1} and we refer to this as Pseudo-PoF-Alg. Finally, as baselines we used our own implementations of two ``unfair'' $k$-center algorithms, specifically the $2$-approximation of \cite{Hochbaum1985} and the $2$-approximation of \cite{Gonzalez1985}.

\textbf{Range of $k$ and value of fairness parameter $\alpha$:} We ran all of our experiments for every value of $k$ in $\{2,4,8,16,32,64,128\}$, and in all our simulations we set $\alpha = 2$ for constraints (\ref{cons:PP}) and (\ref{cons:AG}). We did not test any other value for $\alpha$, since in practice $\alpha > 2$ is unsuitable if reasonably strong fairness considerations are at play. 

\textbf{Constructing the similarity sets: } For each combination of dataset and value of $k$ that we are interested in, we need to construct the similarity sets $\mathcal{S}_j$, such that they satisfy Assumption \ref{icml-asm}. Our first step in doing so, was utilizing the filtering procedure from \cite{Hochbaum1985}, which for a given instance (combination of a dataset and a value $k$) returns a value $R_f$. If $R^*_{unf}$ is the value of the optimal ``unfair'' $k$-center solution for the instance, the aforementioned filtering guarantees that $R_f \leq R^*_{unf}$. Then, for each point $j$ we drew $R_j$ uniformly at random from $[0,2R_f]$, and then set $\mathcal{S}_j =  \{j'~|~ d(j,j') \leq R_j\}$. There were two reasons for constructing the sets $\mathcal{S}_j$ in this way. At first, this approach agrees with the canonical case for $\psi$. As described in Section \ref{sec:asm}, $\psi = 2$ is the most well-justified instantiation of Assumption \ref{icml-asm}. Second, this approach forces non-uniformity in the values of $R_j$, and thus we are able to test our algorithms in the most general setting (for instance the uniform setting described in Theorem \ref{thm:pof-2} is more restricted and less realistic). 

\textbf{Evaluated Metrics:} Let $S$ be the set of chosen centers and $\phi: \mathcal{C} \mapsto S$ the corresponding assignment function, that constituted the solution we got when we ran some particular algorithm on some problem instance. The quantities we evaluate are:
\begin{itemize}
    \item \textbf{Maximum assignment distance} $(\max_{j \in \mathcal{C}}d(j,\phi(j))$: This is the actual objective function value of the returned solution.
    \item \textbf{Satisfaction of constraint (\ref{cons:PP})}:
    Here for each $j$ we define $f^{\text{PP}}_j = \max_{j' \in \mathcal{S}_j}\frac{d(j,\phi(j))}{d(j',\phi(j'))}$.
    \item \textbf{Satisfaction of constraint (\ref{cons:AG})}: Here for each $j$ we define $f^{\text{AG}}_j = \frac{|\mathcal{S}_j|d(j,\phi(j))}{\sum_{j' \in \mathcal{S}_j}d(j',\phi(j'))}$.
\end{itemize}

We now present our results that involve running all 5 mentioned algorithms on the Adult dataset. The corresponding plots for the other four datasets can be found in Appendix \ref{sec:appendix}, and they exhibit the exact behavior as the ones displayed here. In addition, the maximum runtime encountered in all our simulations was approximately 30 minutes (running Alg-PP on Census1990), and the bottleneck in all executions was computing the pairwise distances and not running the algorithms.

\begin{figure}[t]
\centering
\includegraphics[scale=0.2]{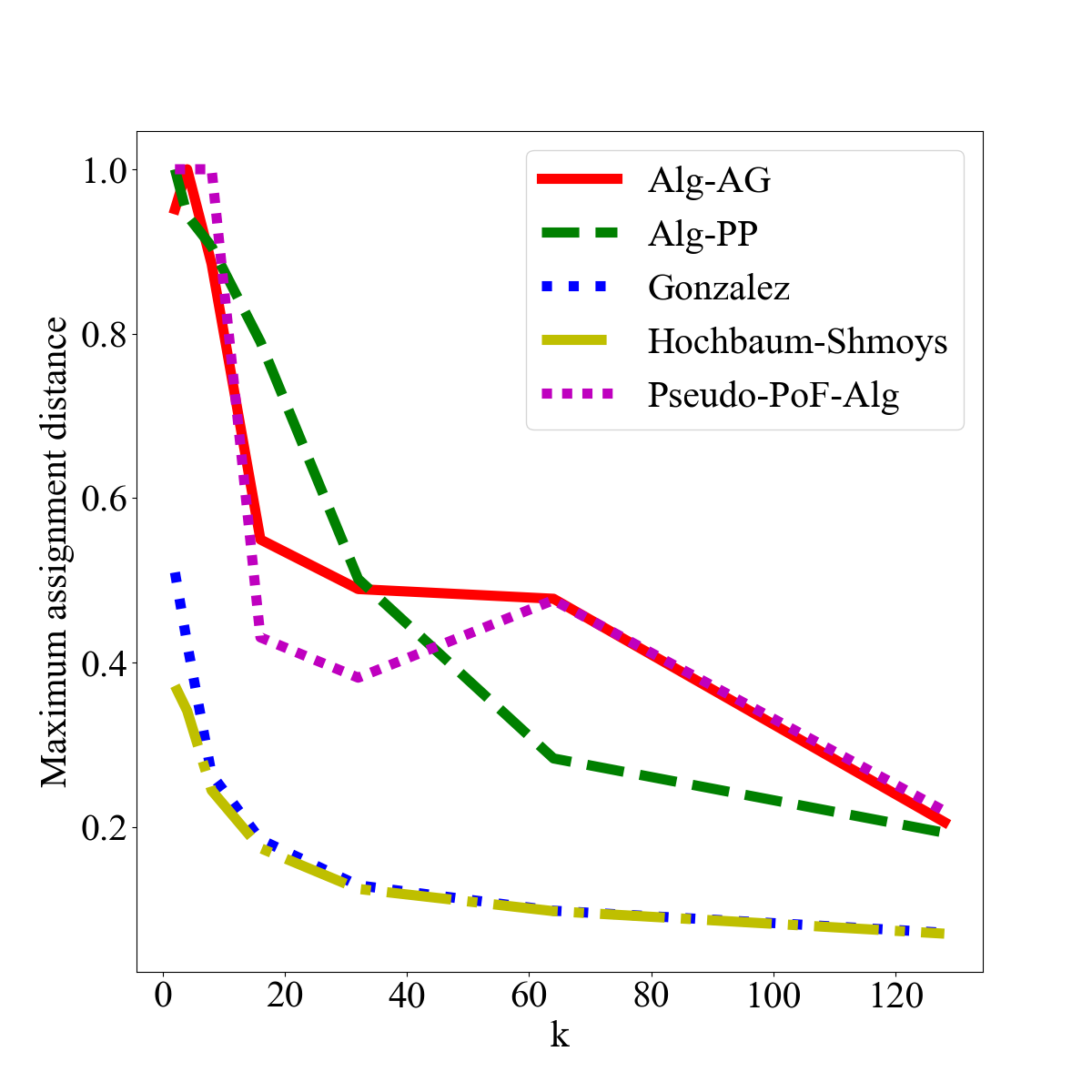}
\caption{Maximum assignment distance for all algorithms}
\label{adult-1}
\end{figure}

In Figure \ref{adult-1} we present the maximum assignment distance as a function of $k$ for all algorithms. At first, we observe that even our algorithms with no PoF guarantees, i.e., Alg-PP and Alg-AG, perform very well in terms of an empirical PoF with respect to the baseline solutions. In addition, we want to compare the objective values of Alg-PP and Alg-AG. Recall that since a solution to \prob{-PP} also constitutes a solution to \prob{-AG}, we are theoretically expecting Alg-AG to perform better. However, we see that in practice there is no clear-cut winner, and hence the use of Alg-PP is highly recommended, since the notion of fairness guaranteed by that algorithm is much stronger. 

\begin{figure}[h]
\begin{subfigure}{.33\textwidth}
  \centering
  \includegraphics[width=.99\linewidth]{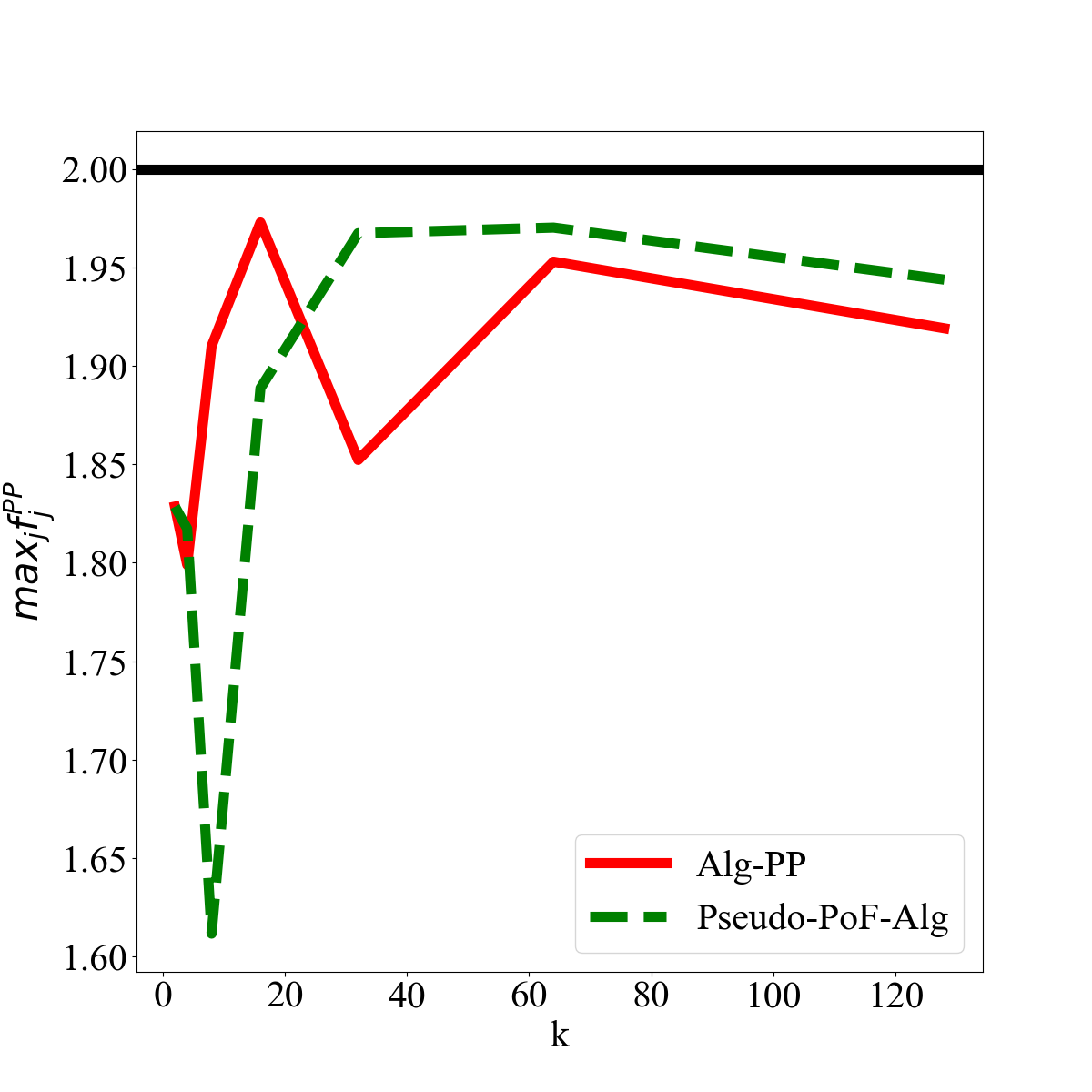}
  \caption{}
  \label{adult-2a}
\end{subfigure}%
\begin{subfigure}{.33\textwidth}
  \centering
  \includegraphics[width=.99\linewidth]{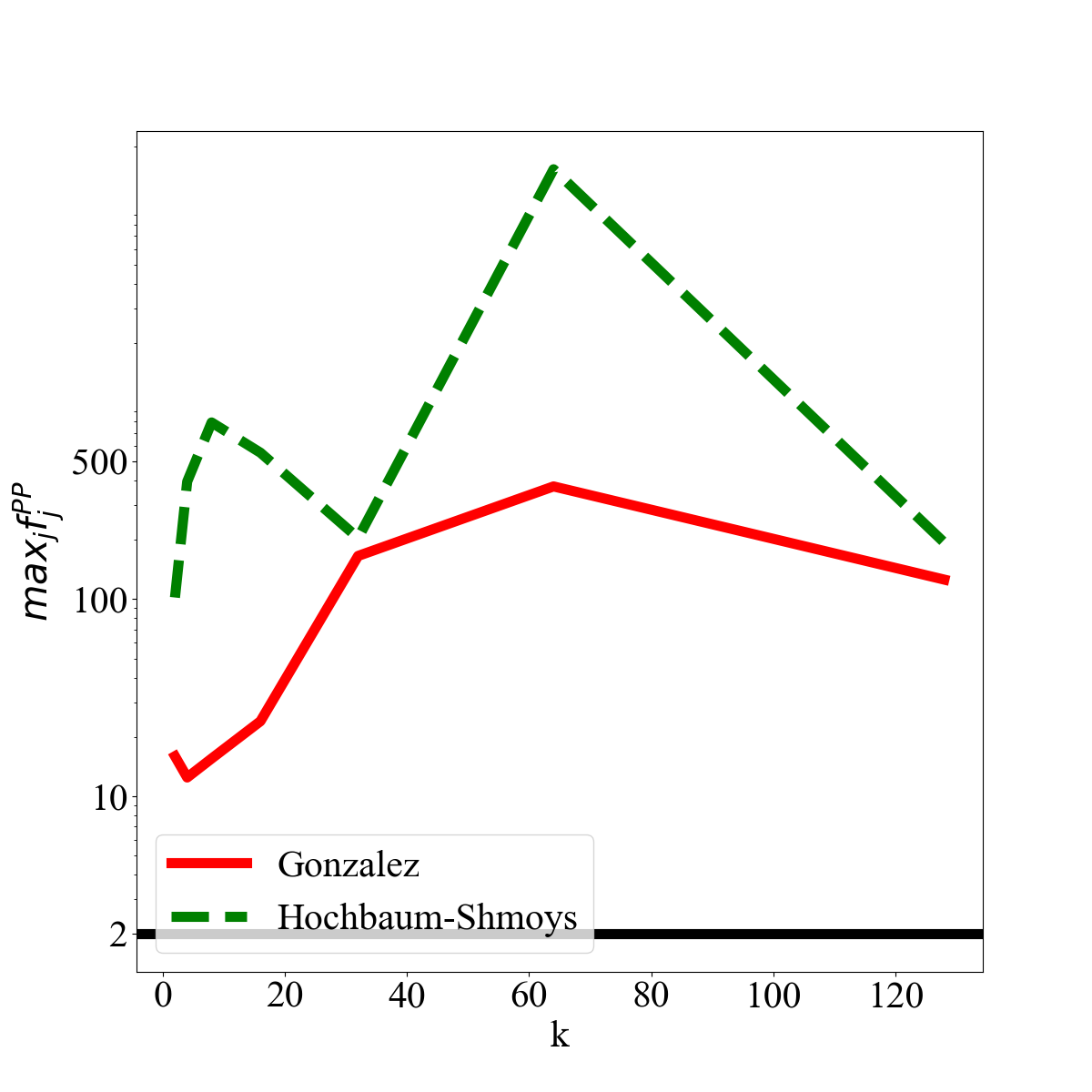}
  \caption{}
  \label{adult-2b}
\end{subfigure}
\begin{subfigure}{.33\textwidth}
  \centering
  \includegraphics[width=.99\linewidth]{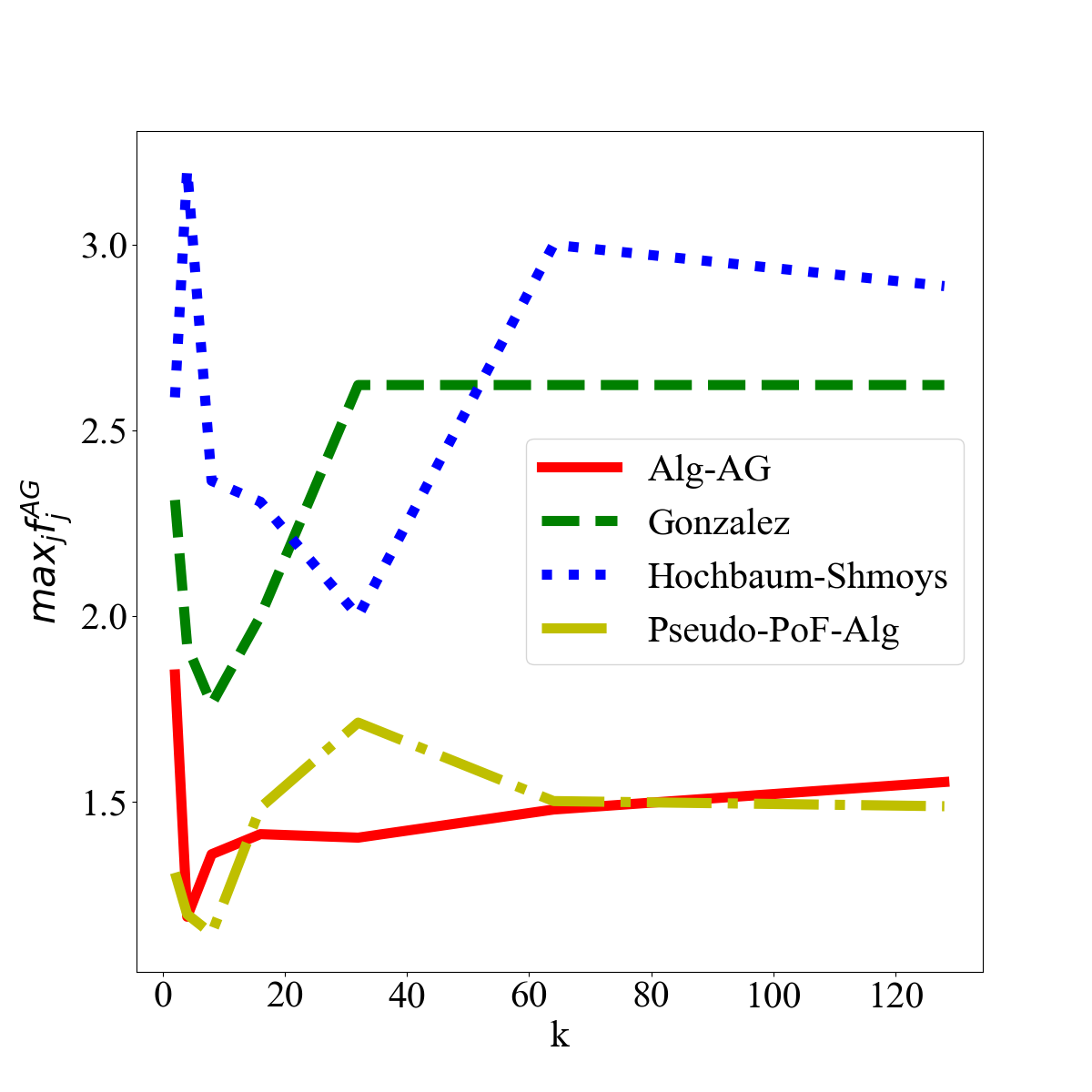}
  \caption{}
  \label{adult-2c}
\end{subfigure}
\caption{Satisfaction of fairness constraints}
\label{adult-2}
\end{figure}

In Figure \ref{adult-2} we demonstrate how all algorithms perform in terms of the fairness constraints.\footnote{In these plots, for the two baseline algorithms we excluded points with $f^{PP}_j = +\infty$ or $f^{AG}_j = +\infty$ in the computation of $\max_{j} f^{PP}_j$ and $\max_{j} f^{AG}_j$. In other words, we were very lenient with the two baselines.} Figure \ref{adult-2a} shows $\max_{j} f^{PP}_j$ as a function of $k$ for our two algorithms for \prob{-PP}, i.e., Alg-PP and Pseudo-PoF-Alg. Here we see that as the theory suggests, our algorithms always satisfy constraint (\ref{cons:PP}) and have $\max_j f^{PP}_j \leq 2$. On the other hand, Figure \ref{adult-2b} shows $\max_{j} f^{PP}_j$ as a function of $k$ for the baselines. Here we see that the baselines are far from satisfying constraint (\ref{cons:PP}), and specifically that there exist points that are treated very unfairly. Finally, Figure \ref{adult-2c} shows $\max_{j} f^{AG}_j$ as a function of $k$ for all algorithms that can be potentially used for \prob{-AG}. Here we see that our algorithms again satisfy the corresponding constraint (\ref{cons:AG}), and furthermore have a better $\max_{j} f^{AG}_j$ value compared to the baselines. Finally, in the AG case the baselines seem to perform much better compared to the PP case, and this is reasonable because the notion of fairness described by (\ref{cons:AG}) is much weaker. Nonetheless, in most cases the baselines are not able to satisfy (\ref{cons:AG}). 

\begin{figure}[h]
\begin{subfigure}{.33\textwidth}
  \centering
  \includegraphics[width=.99\linewidth]{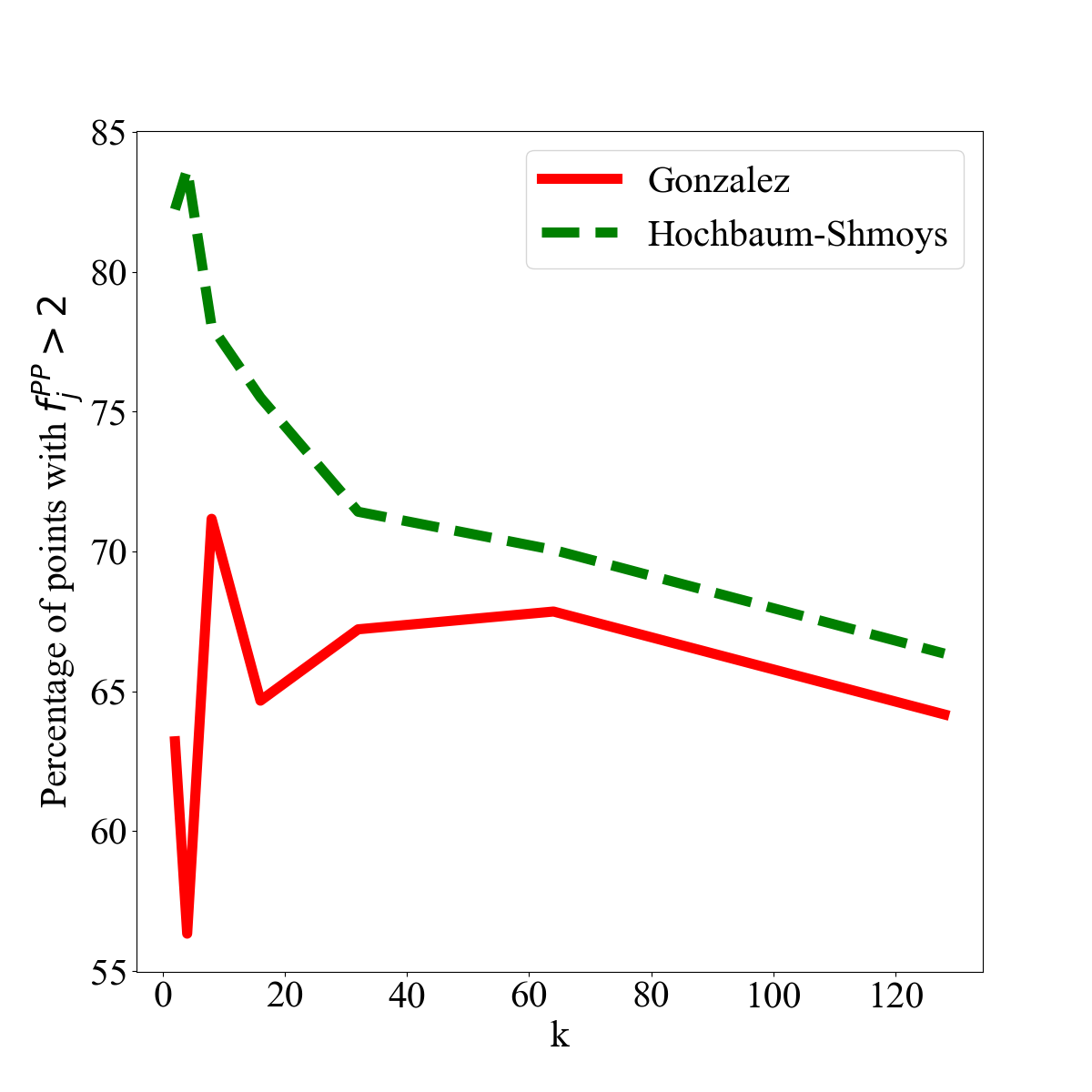}
  \caption{}
  \label{adult-3a}
\end{subfigure}%
\begin{subfigure}{.33\textwidth}
  \centering
  \includegraphics[width=.99\linewidth]{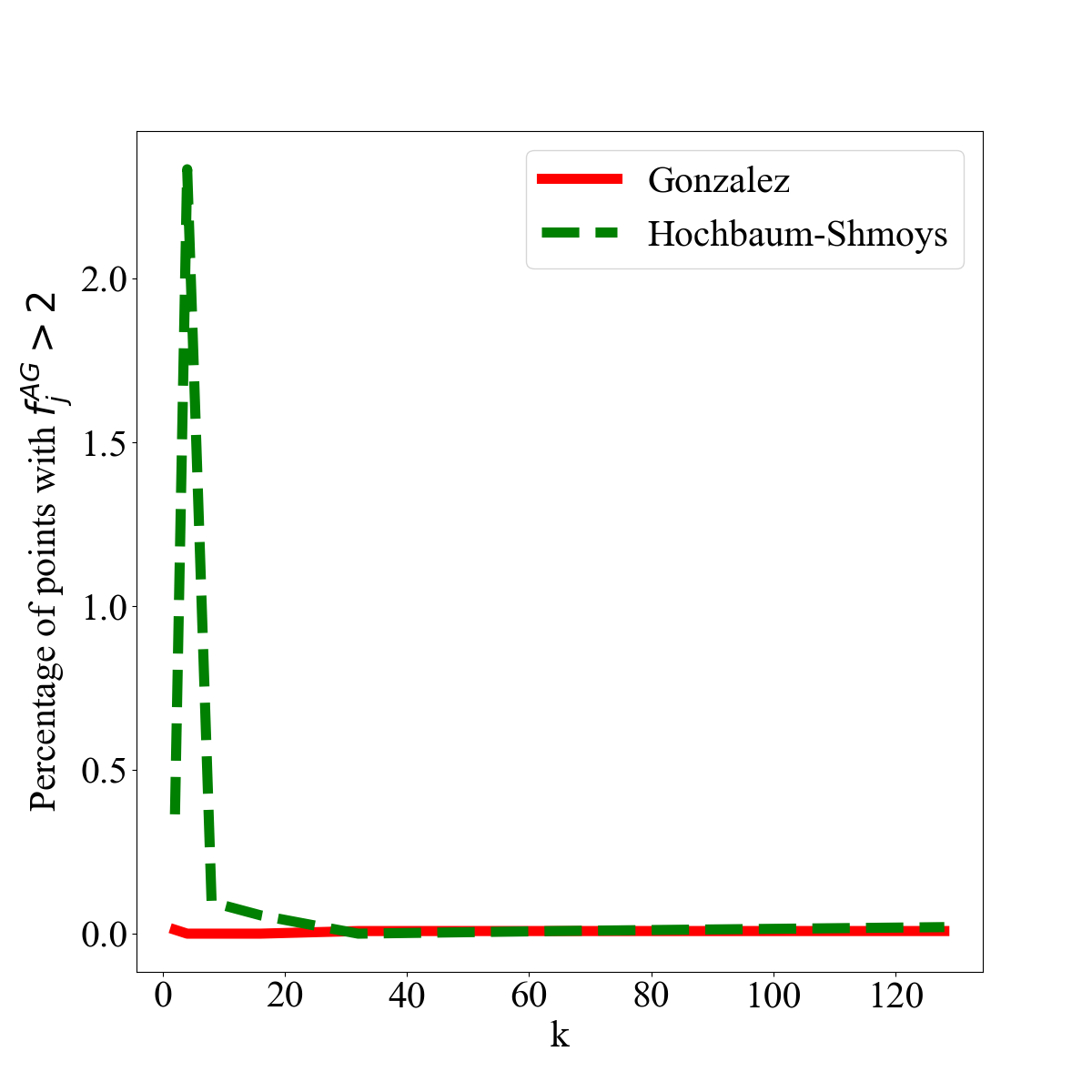}
  \caption{}
  \label{adult-3b}
\end{subfigure}
\begin{subfigure}{.33\textwidth}
  \centering
  \includegraphics[width=.99\linewidth]{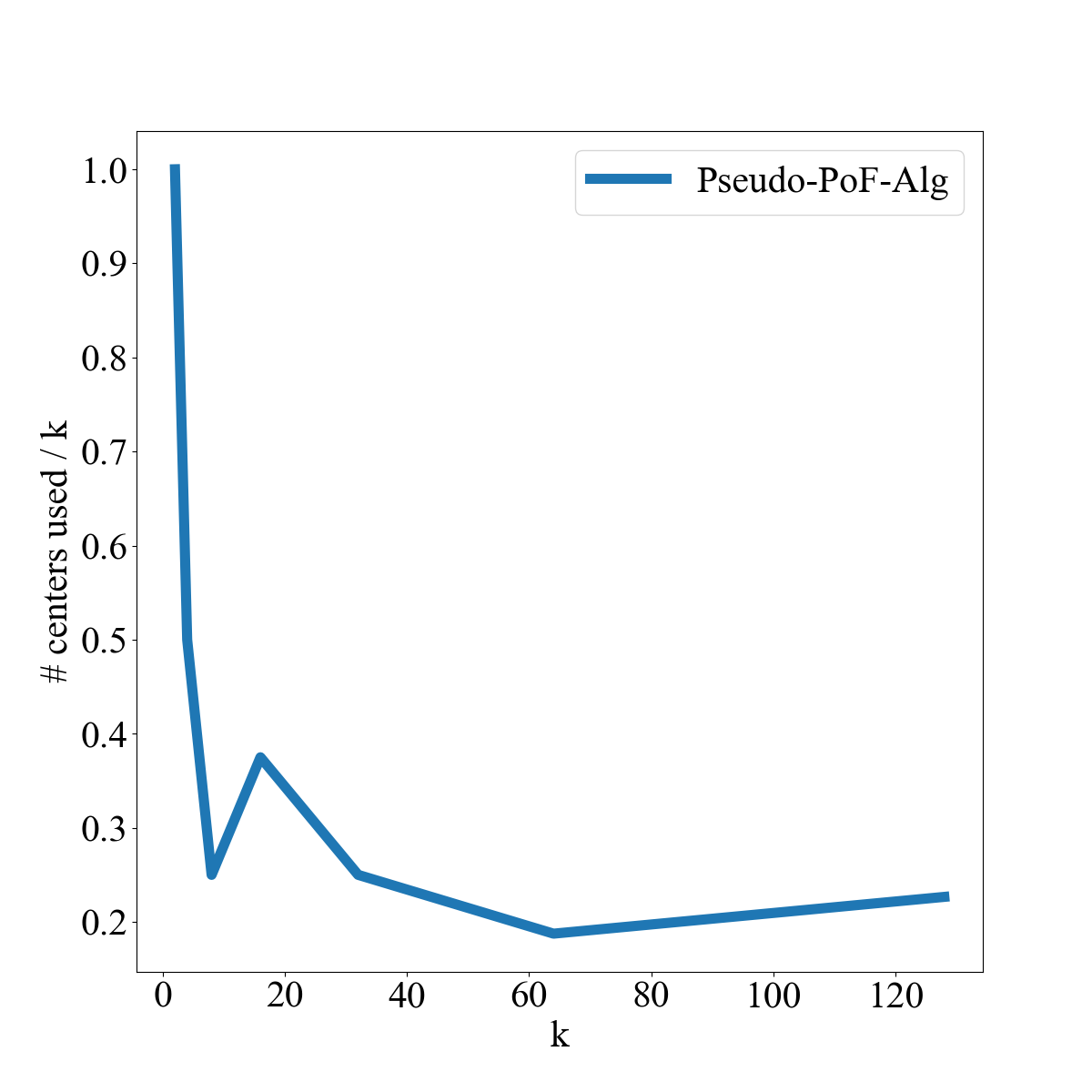}
  \caption{}
  \label{adult-3c}
\end{subfigure}
\caption{Amount of constraint violation}
\label{adult-3}
\end{figure}

In Figures \ref{adult-3a} and \ref{adult-3b} we are interested in the percentage of points for which baselines do not satisfy the appropriate fairness constraint. Specifically, Figure \ref{adult-3a} demonstrates that for the stronger notion of PP-fairness, a substantial percentage of points gets unfair treatment ($f^{PP}_j > 2$). On the other hand, for the weaker notion of fairness captured by (\ref{cons:AG}), the two baselines do much better. Nonetheless, even if one is interested only in the weaker AG concept of fairness, they should not use the baselines. \emph{Even one unfairly treated point goes against the very nature of individual fairness}.

Finally, in Figure \ref{adult-3c} we see by how much Pseudo-PoF-Alg violates the constraint $|S| \leq k$ on the set of chosen centers (recall that in theory Pseudo-PoF-Alg yields $|S| \leq 2k$). Here we plot the ratio of the number of centers used by the algorithm over the given value $k$, and see that in practice Pseudo-PoF-Alg does not actually incur any violation. 
\\ \\ 
\noindent\textbf{Acknowledgments.}

We would like to thank the anonymous AISTATS reviewers for their constructive feedback. Darshan Chakrabarti, John Dickerson, and Seyed Esmaeili were supported in part by NSF CAREER Award IIS-1846237, NSF D-ISN Award \#2039862, NSF Award CCF-1852352, NIH R01 Award NLM-013039-01, NIST MSE Award \#20126334, DARPA GARD \#HR00112020007, DoD WHS Award \#HQ003420F0035, and a Google Faculty Research award.
Aravind Srinivasan was supported in part by NSF awards CCF-1422569, CCF-1749864, and CCF-1918749, as well as research awards from Adobe, Amazon, and Google. Leonidas Tsepenekas was  supported in part by NSF awards CCF-1749864 and CCF-1918749, and by research awards from Amazon and Google. 
\printbibliography
\appendix
\section{Explicitly enforcing Assumption \ref{icml-asm}}\label{sec:asm-enf}

It is reasonable to assume that there will be situations in which a central planner is not certain that Assumption \ref{icml-asm} holds. Furthermore, there may also be cases where the sets $\mathcal{S}_j$ are not explicitly provided, e.g., because individuals have a fuzzy understanding of similarity and cannot accurately determine their most comparable points. Nonetheless, even under such conditions, the central planner can help the points construct the sets $\mathcal{S}_j$, in way that is explainable and will also satisfy the necessary assumption. This is clearly described in what follows.

The planner can first compute a nearly-tight lower bound $R_f$ for $R^{*}_{unf}$ (note that computing $R^{*}_{unf}$ exactly is NP-hard). This can be done efficiently in multiple ways, for example by using the thresholding technique of \cite{Hochbaum1985}. Afterwards, the planner publishes $R_f$ and informs the agents that even under optimal conditions, the points that are considered similar to each of them are only within distance $\psi R_f$, for some small constant $\psi$. Then, the points are asked to independently construct their similarity sets, such that $\mathcal{S}_j \subseteq \{j' \in \mathcal{C} : d(j,j') \leq \psi R^{*}_{unf}\}$. 

This strategy certainly enjoys explainability merits. Besides having the planner compute, publish and clarify the meaning of $\psi R_f$ to the points, it also gives the planner a valid justification to turn down requests for $\mathcal{S}_j$ that do not satisfy Assumption \ref{icml-asm}, by clearly explaining to such an agent $j$ why this choice is unreasonable.

\pagebreak
\section{Additional experimental results}\label{sec:appendix}

\textbf{Experimental results for Bank:}

\begin{figure}[h]
\centering
\includegraphics[scale=0.14]{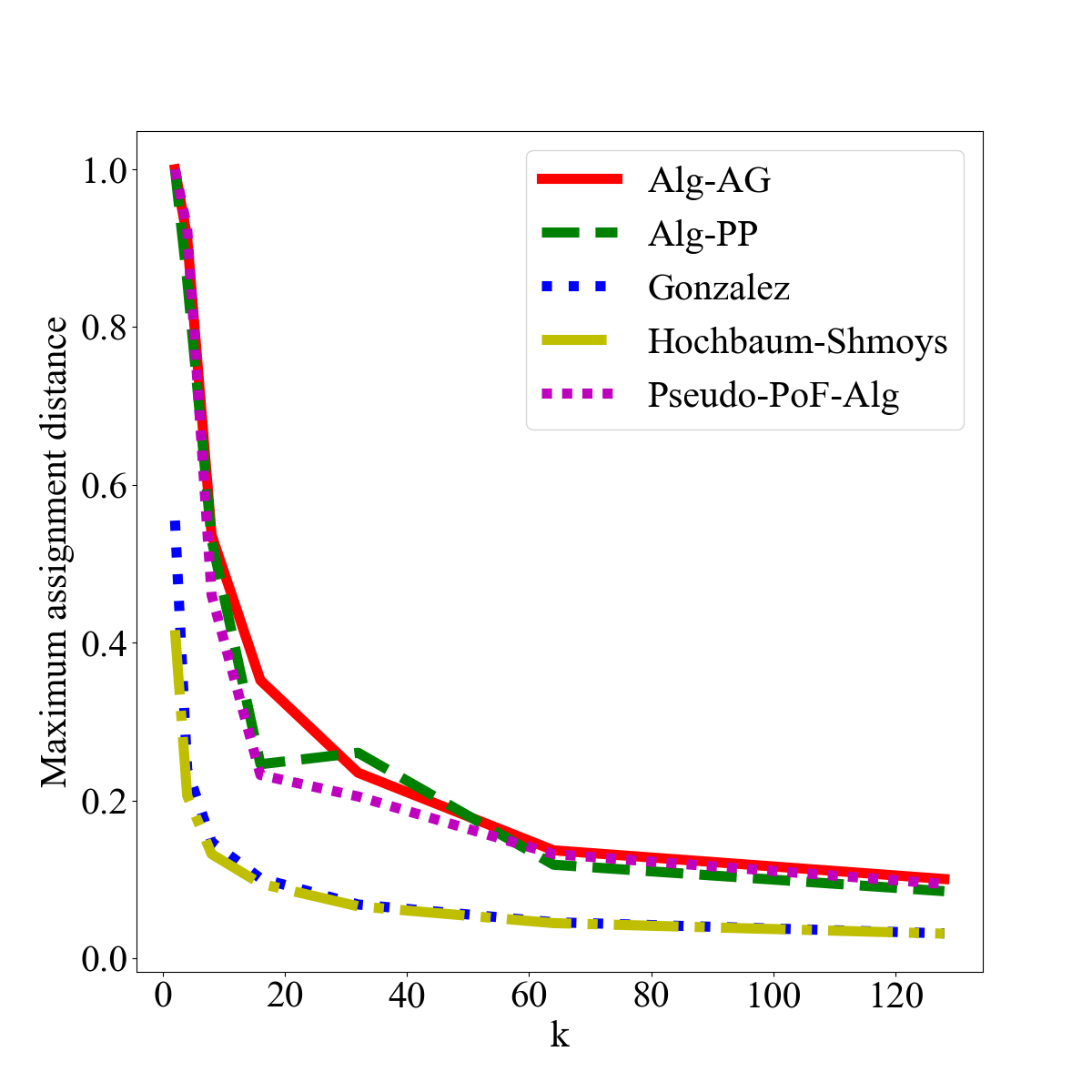}
\caption{Maximum assignment distance for all algorithms}
\label{bank-1}
\end{figure}

\begin{figure}[h]
\begin{subfigure}{.33\textwidth}
  \centering
  \includegraphics[width=.74\linewidth]{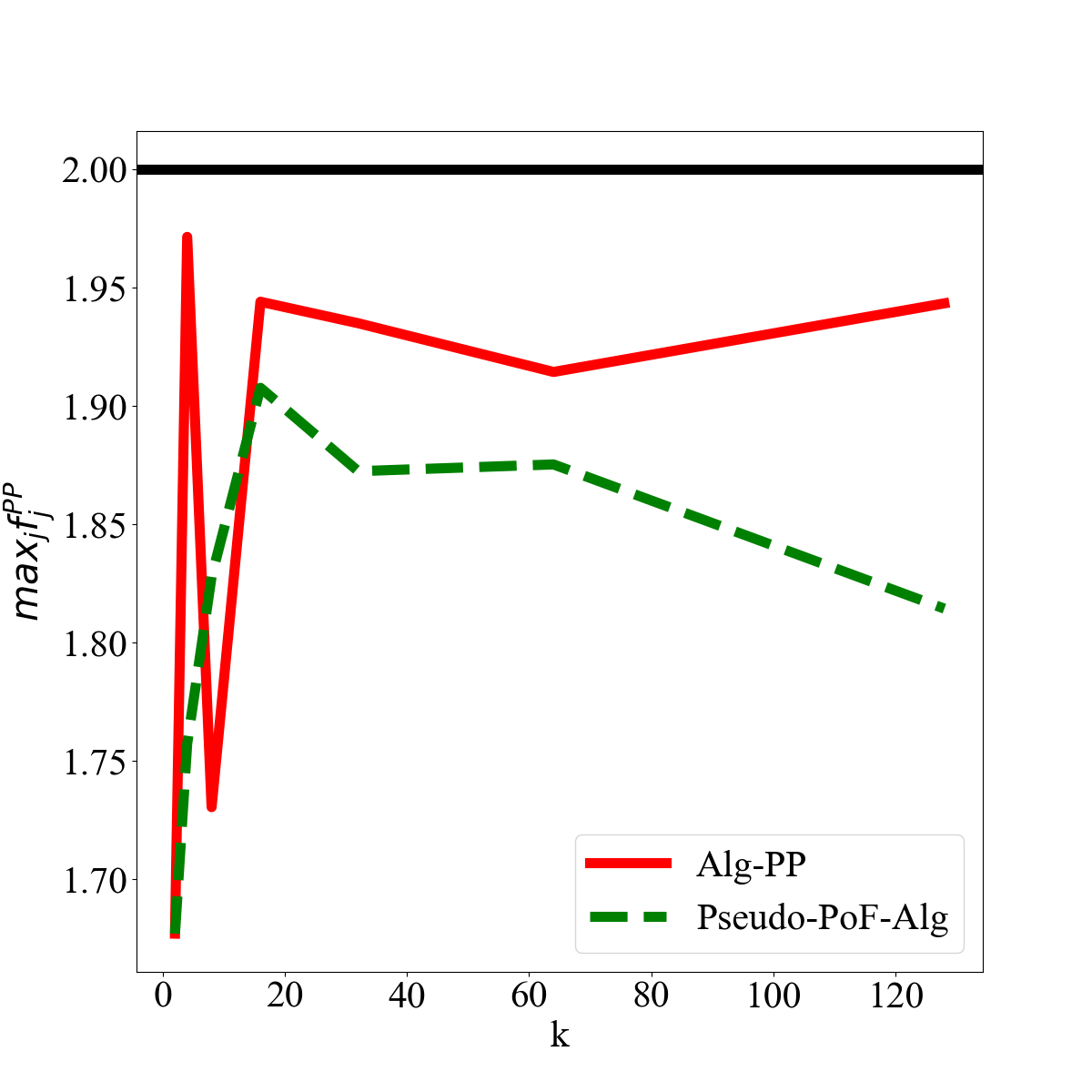}
  \caption{}
  \label{bank-2a}
\end{subfigure}%
\begin{subfigure}{.33\textwidth}
  \centering
  \includegraphics[width=.74\linewidth]{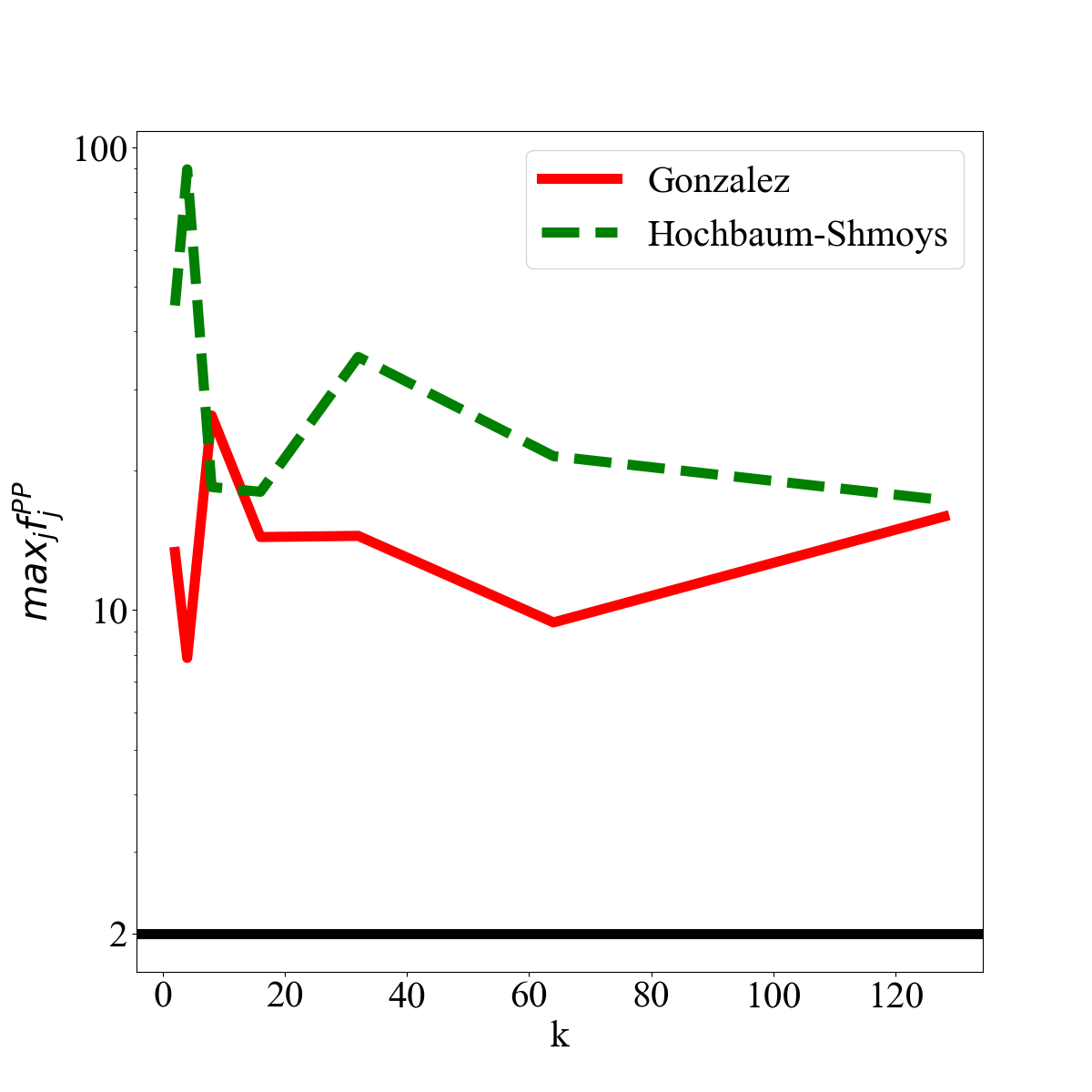}
  \caption{}
  \label{bank-2b}
\end{subfigure}
\begin{subfigure}{.33\textwidth}
  \centering
  \includegraphics[width=.74\linewidth]{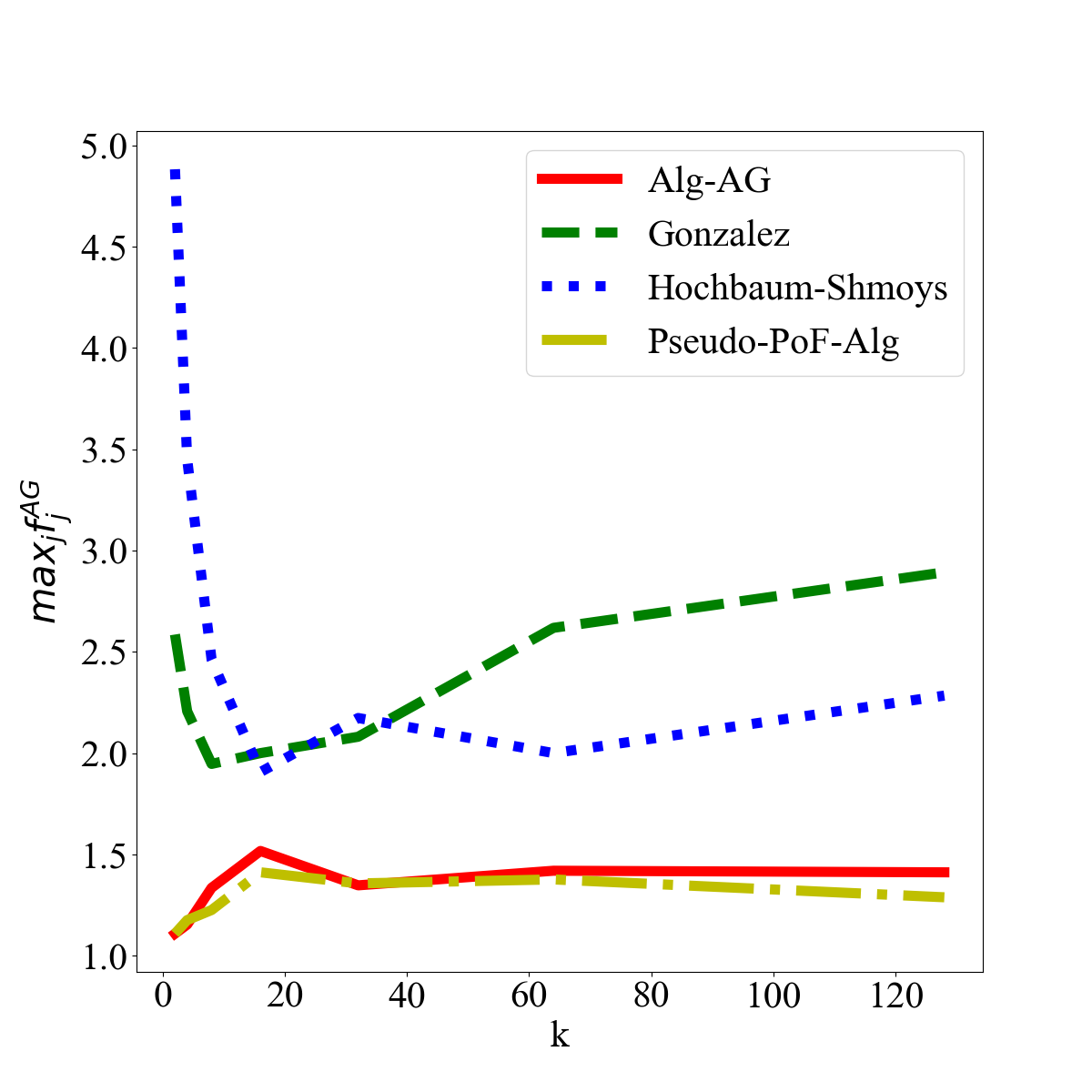}
  \caption{}
  \label{bank-2c}
\end{subfigure}
\caption{Satisfaction of fairness constraints}
\label{bank-2}
\end{figure}

\begin{figure}[h]
\begin{subfigure}{.33\textwidth}
  \centering
  \includegraphics[width=.74\linewidth]{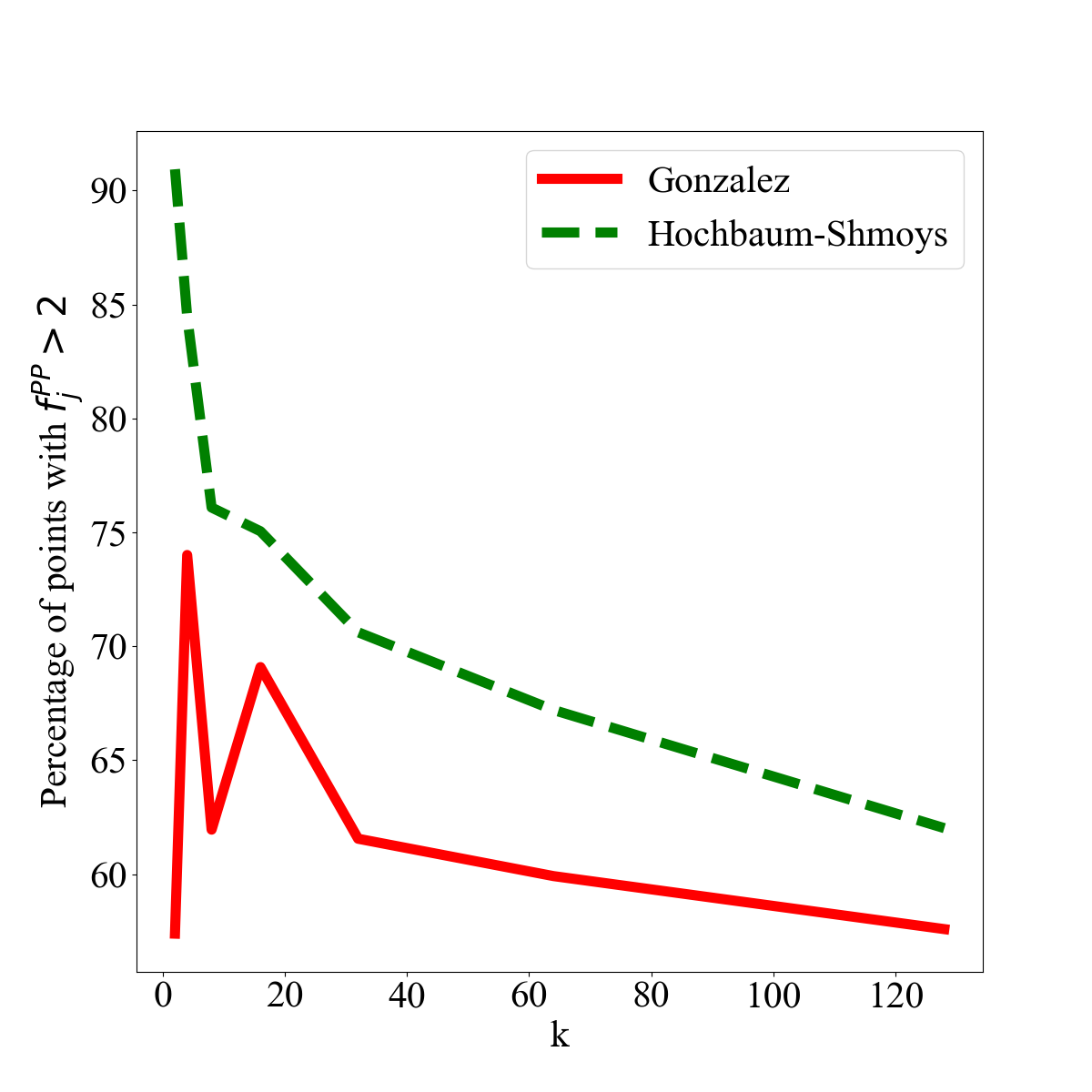}
  \caption{}
  \label{bank-3a}
\end{subfigure}%
\begin{subfigure}{.33\textwidth}
  \centering
  \includegraphics[width=.74\linewidth]{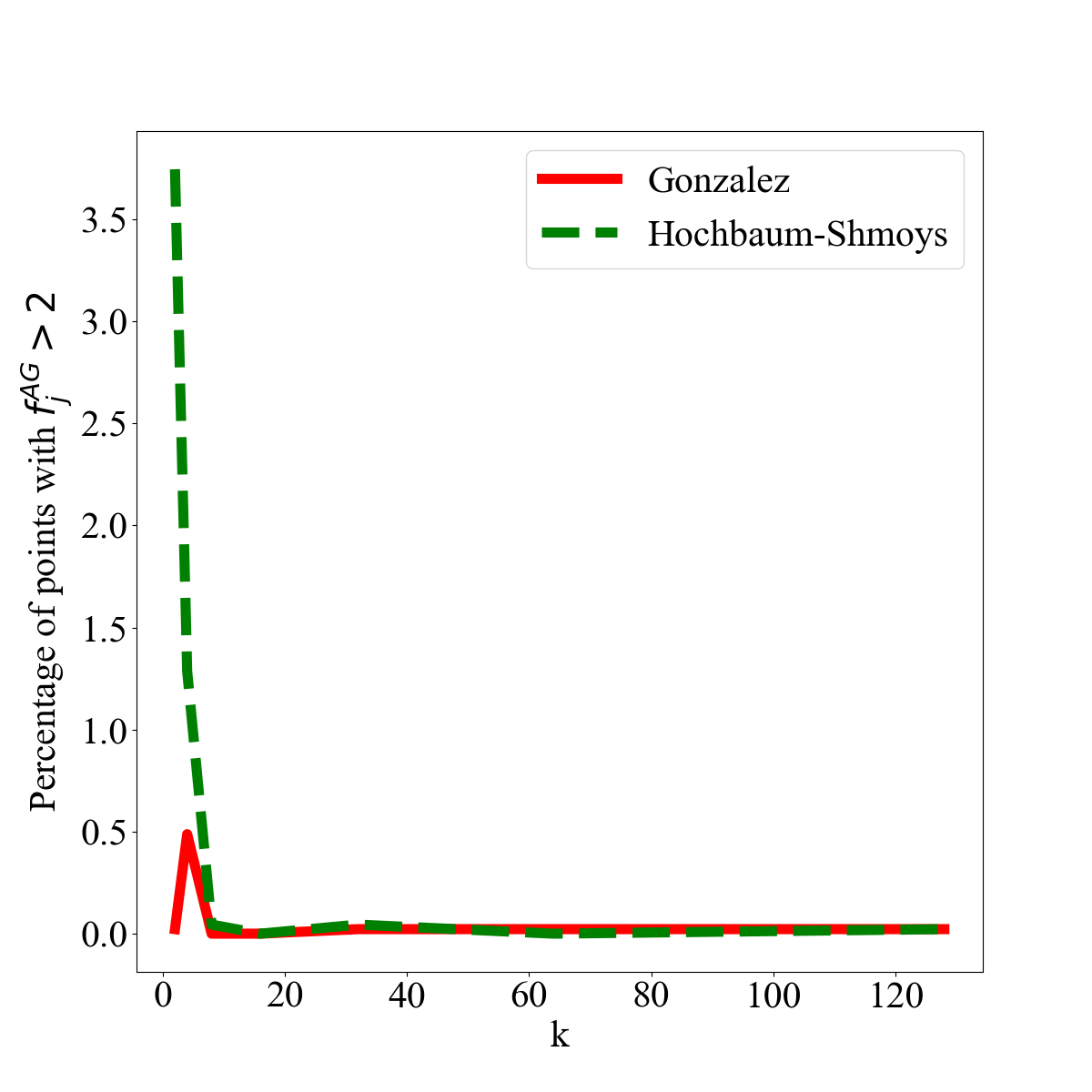}
  \caption{}
  \label{bank-3b}
\end{subfigure}
\begin{subfigure}{.33\textwidth}
  \centering
  \includegraphics[width=.74\linewidth]{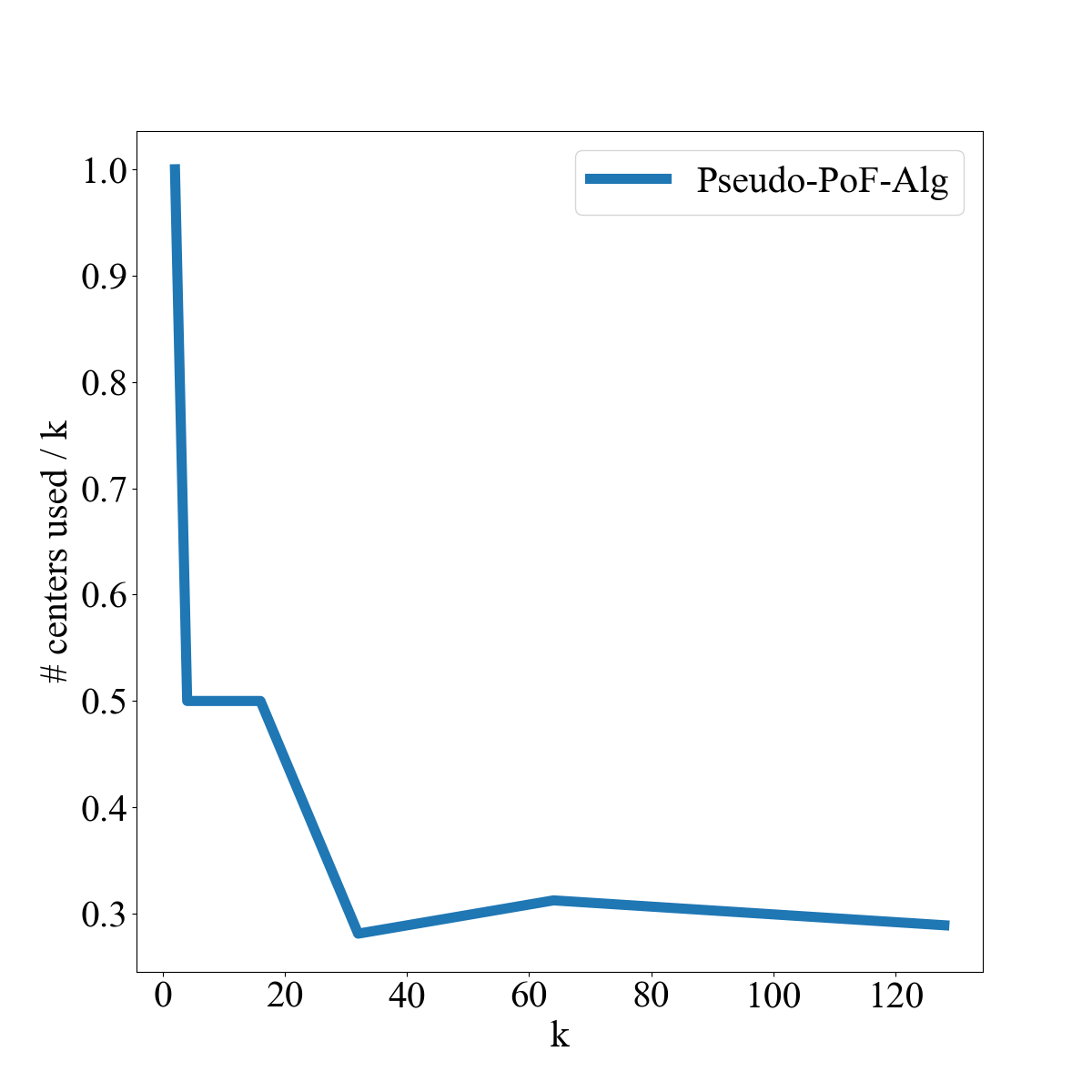}
  \caption{}
  \label{bank-3c}
\end{subfigure}
\caption{Amount of constraint violation}
\label{bank-3}
\end{figure}
\pagebreak

\textbf{Experimental results for Creditcard:}

\begin{figure}[h!]
\centering
\includegraphics[scale=0.18]{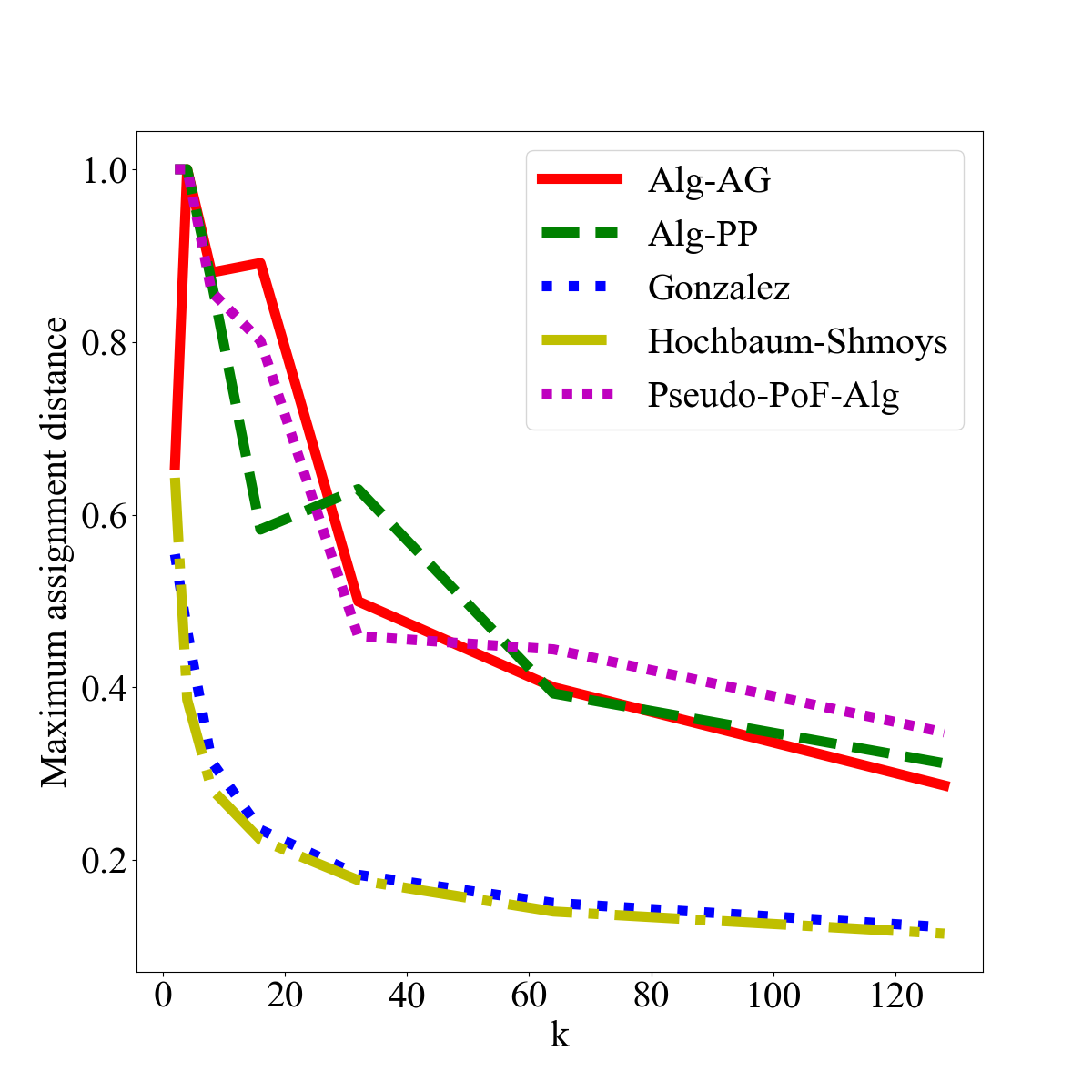}
\caption{Maximum assignment distance for all algorithms}
\label{creditcard-1}
\end{figure}

\begin{figure}[h!]
\begin{subfigure}{.33\textwidth}
  \centering
  \includegraphics[width=.99\linewidth]{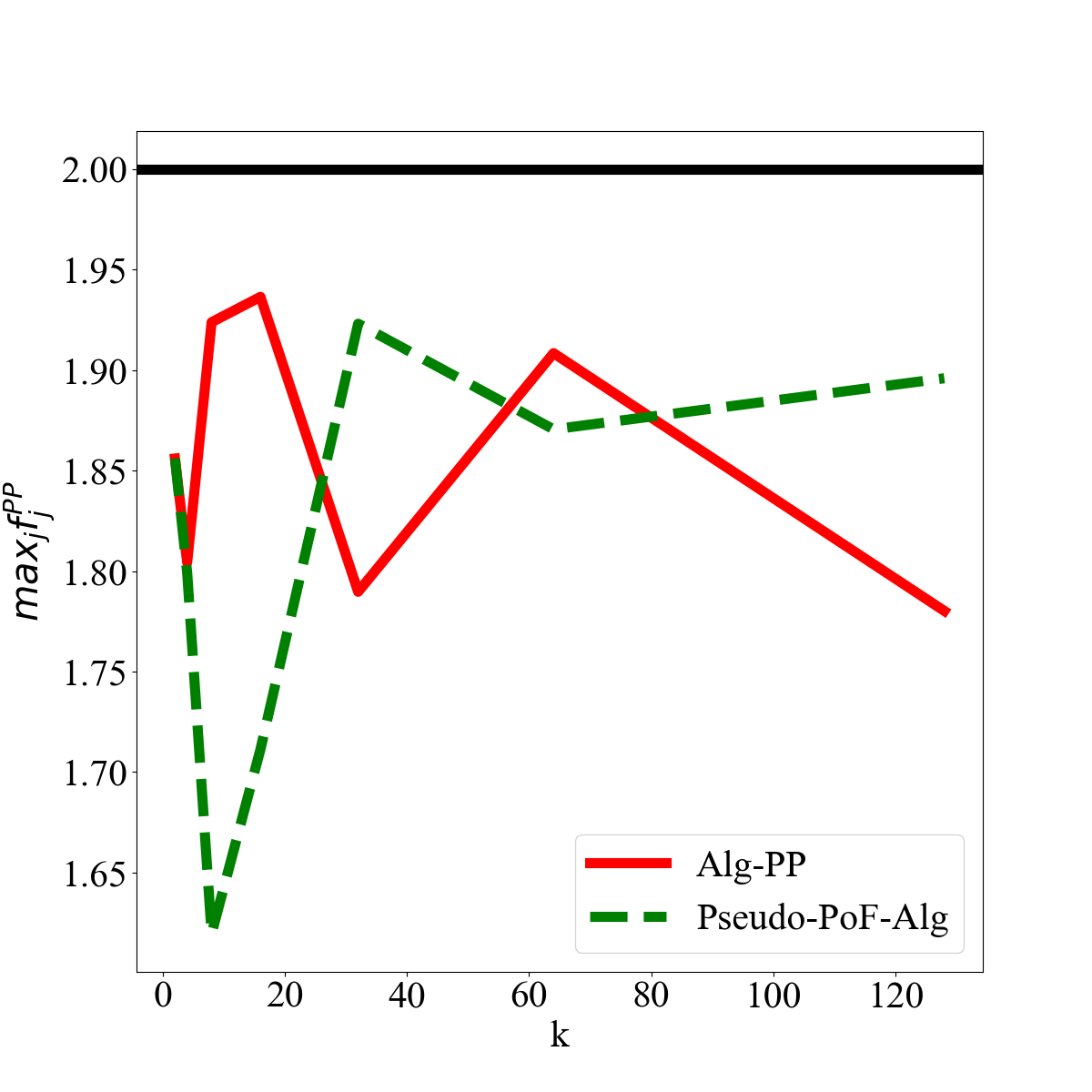}
  \caption{}
  \label{creditcard-2a}
\end{subfigure}%
\begin{subfigure}{.33\textwidth}
  \centering
  \includegraphics[width=.99\linewidth]{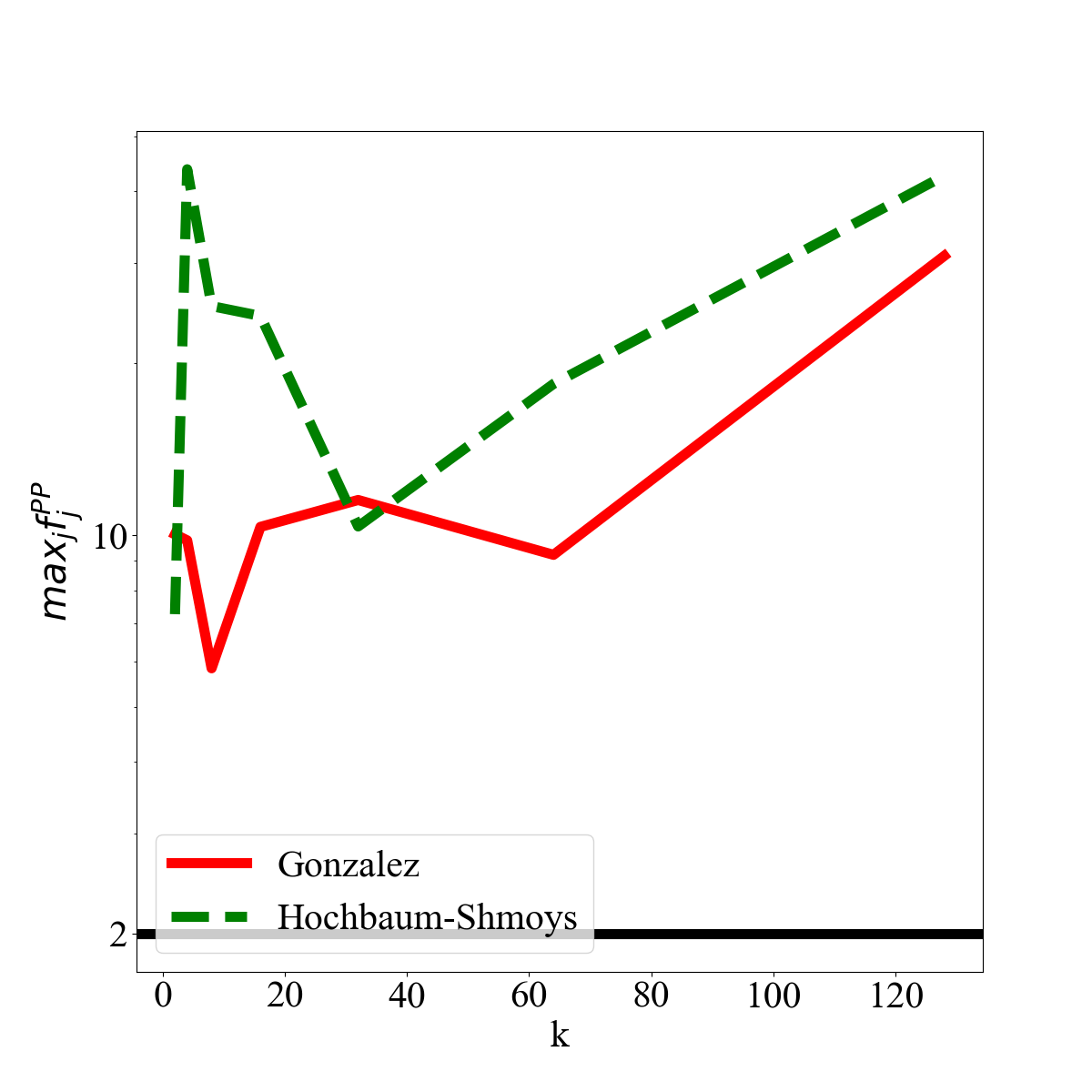}
  \caption{}
  \label{creditcard-2b}
\end{subfigure}
\begin{subfigure}{.33\textwidth}
  \centering
  \includegraphics[width=.99\linewidth]{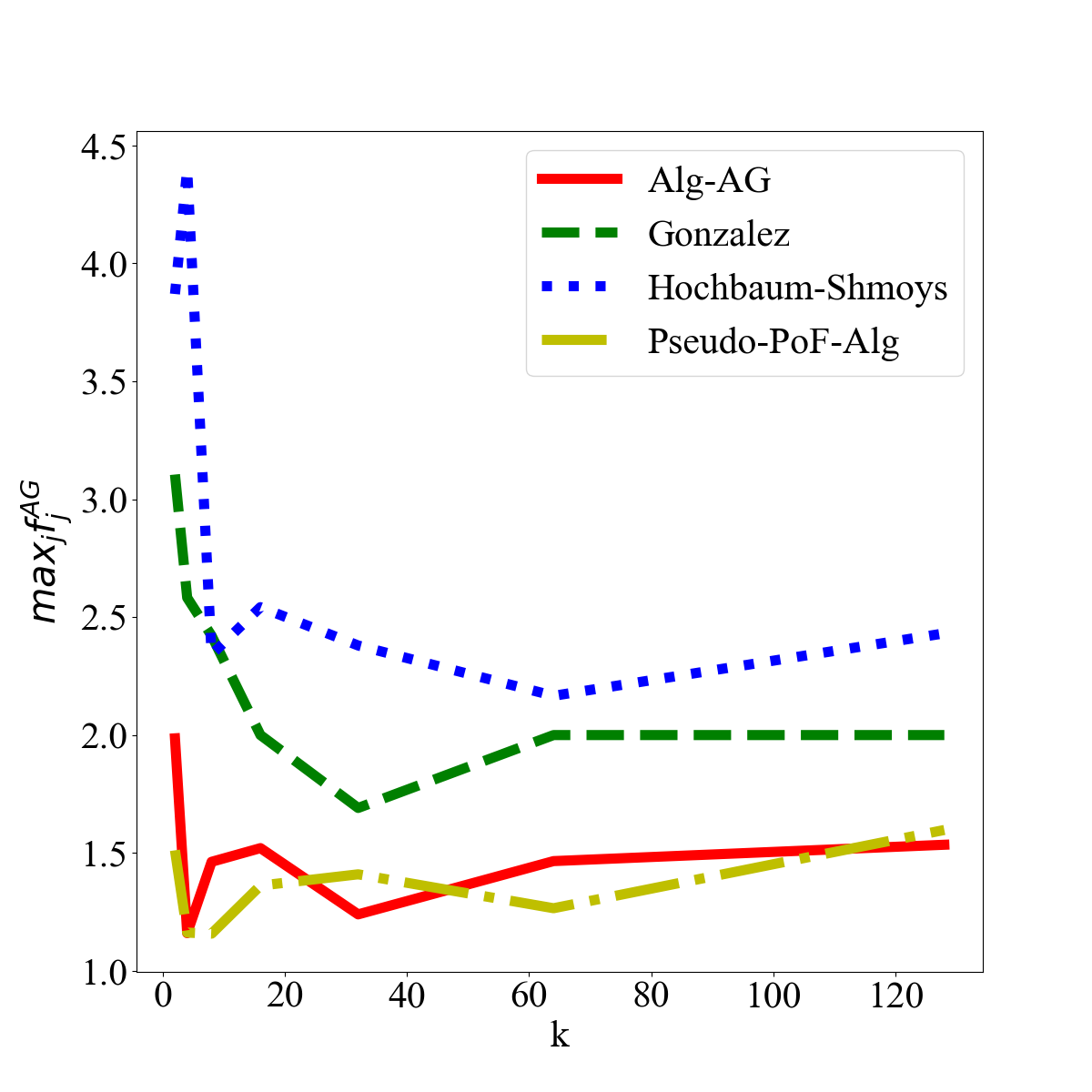}
  \caption{}
  \label{creditcard-2c}
\end{subfigure}
\caption{Satisfaction of fairness constraints}
\label{creditcard-2}
\end{figure}

\begin{figure}[h!]
\begin{subfigure}{.33\textwidth}
  \centering
  \includegraphics[width=.99\linewidth]{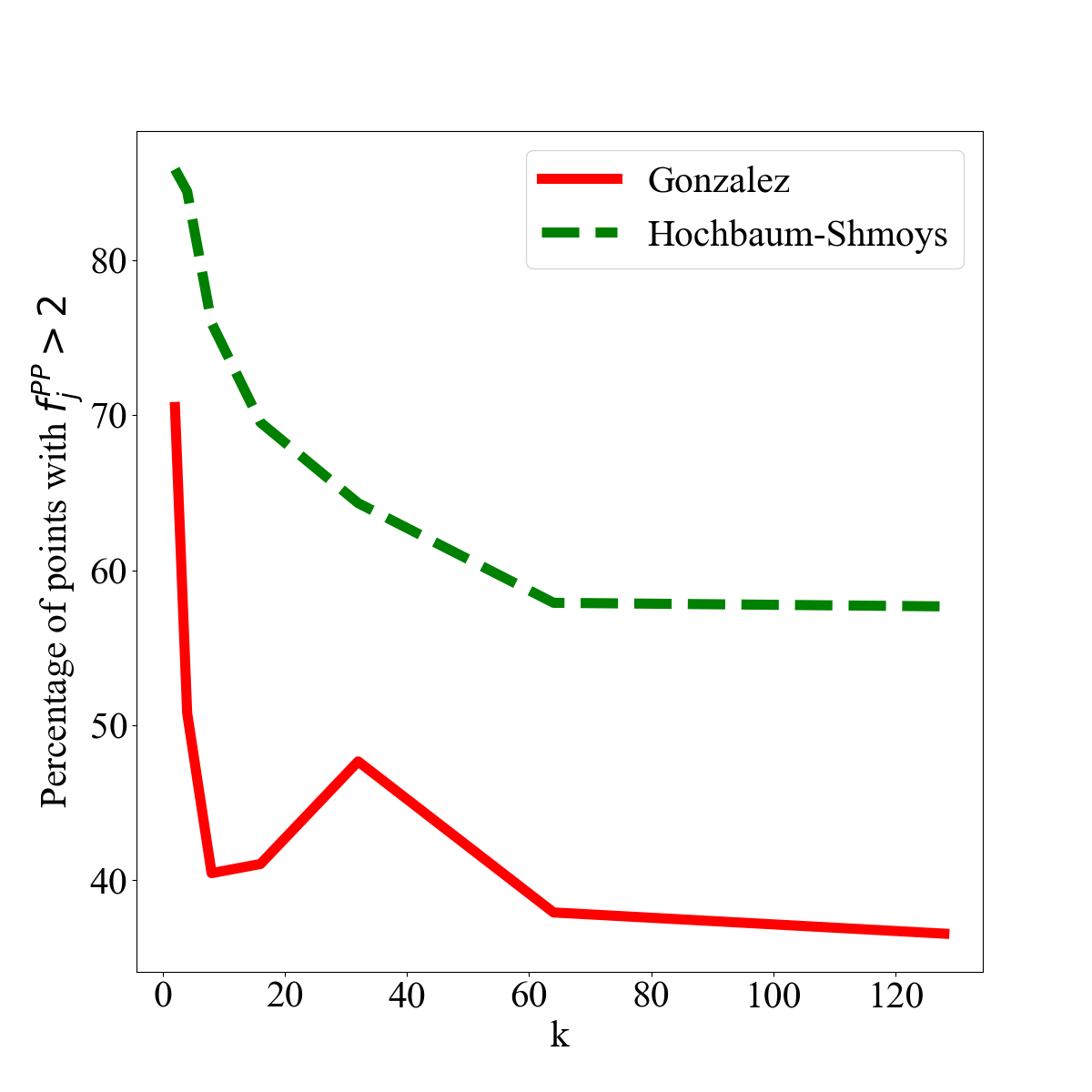}
  \caption{}
  \label{creditcard-3a}
\end{subfigure}%
\begin{subfigure}{.33\textwidth}
  \centering
  \includegraphics[width=.99\linewidth]{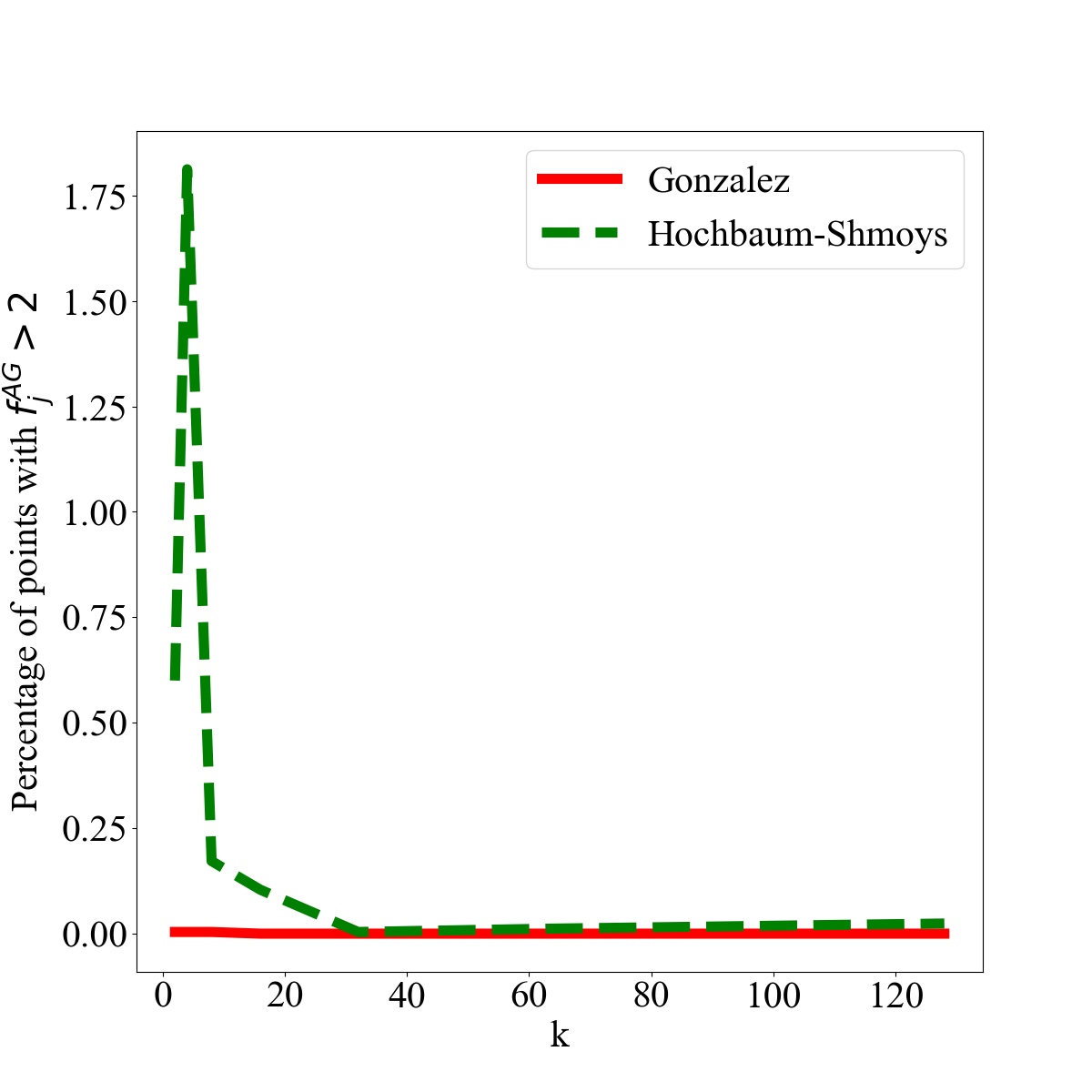}
  \caption{}
  \label{creditcard-3b}
\end{subfigure}
\begin{subfigure}{.33\textwidth}
  \centering
  \includegraphics[width=.99\linewidth]{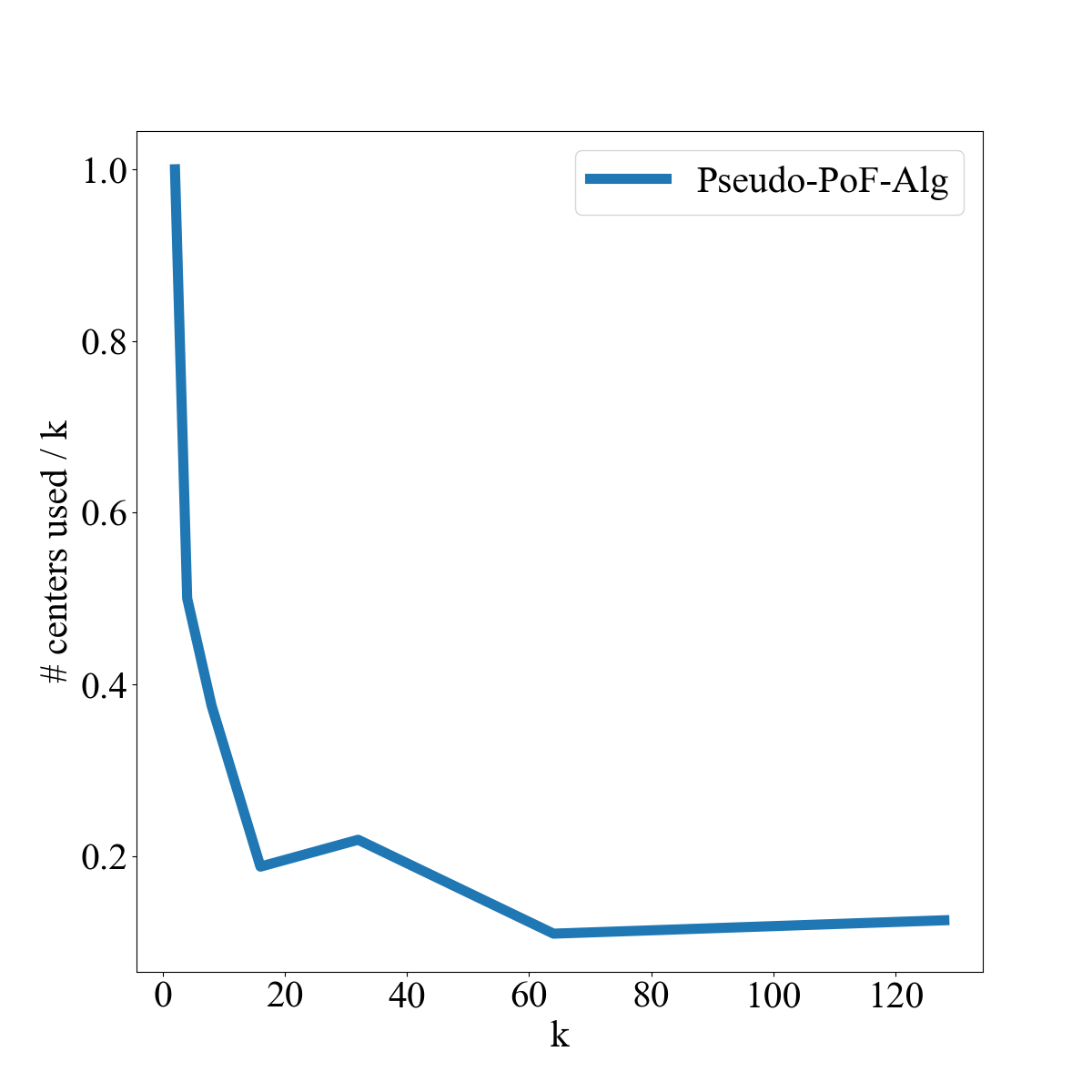}
  \caption{}
  \label{creditcard-3c}
\end{subfigure}
\caption{Amount of constraint violation}
\label{creditcard-3}
\end{figure}

\pagebreak

\textbf{Experimental results for Census1990:}

\begin{figure}[h!]
\centering
\includegraphics[scale=0.18]{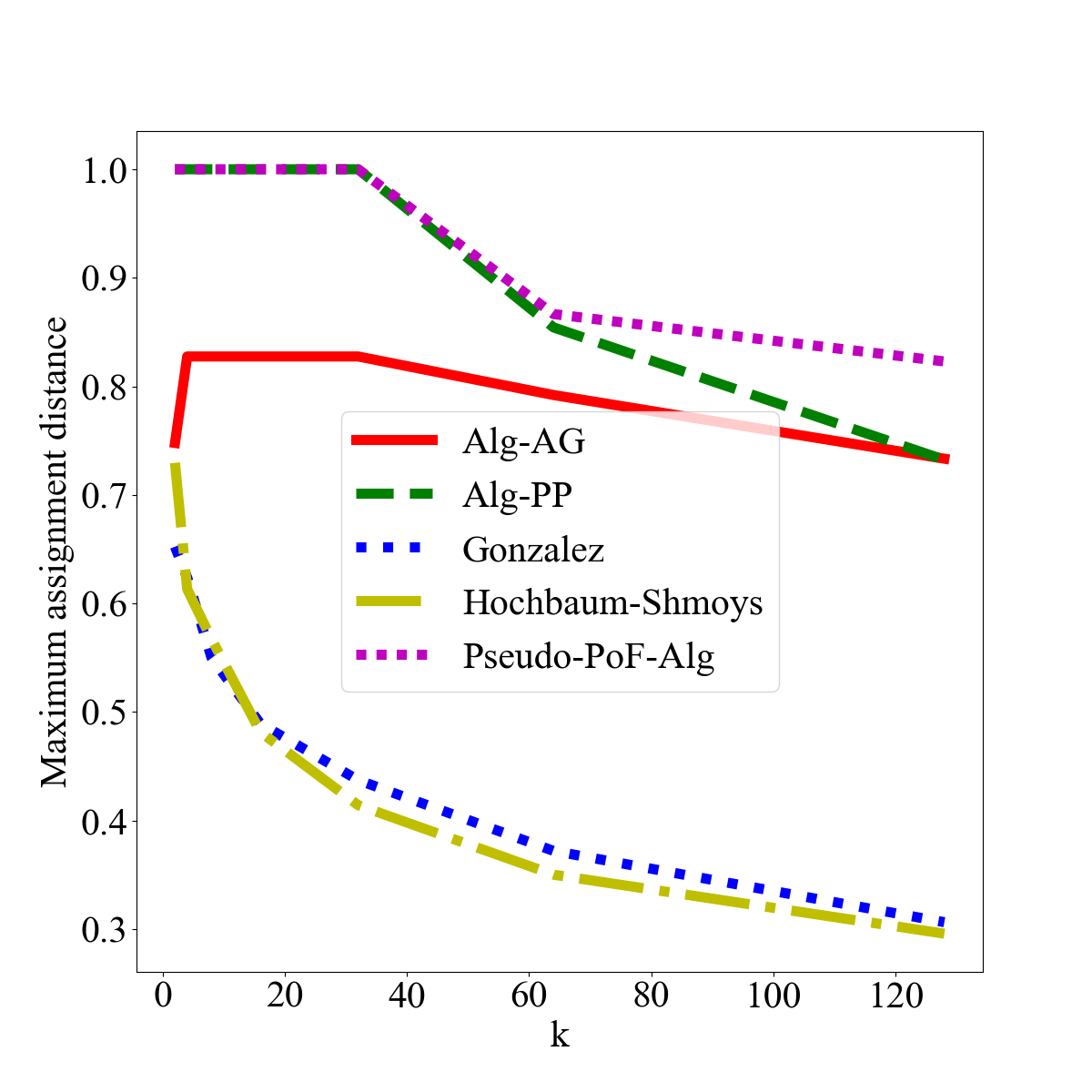}
\caption{Maximum assignment distance for all algorithms}
\label{census1990-1}
\end{figure}

\begin{figure}[h!]
\begin{subfigure}{.33\textwidth}
  \centering
  \includegraphics[width=.99\linewidth]{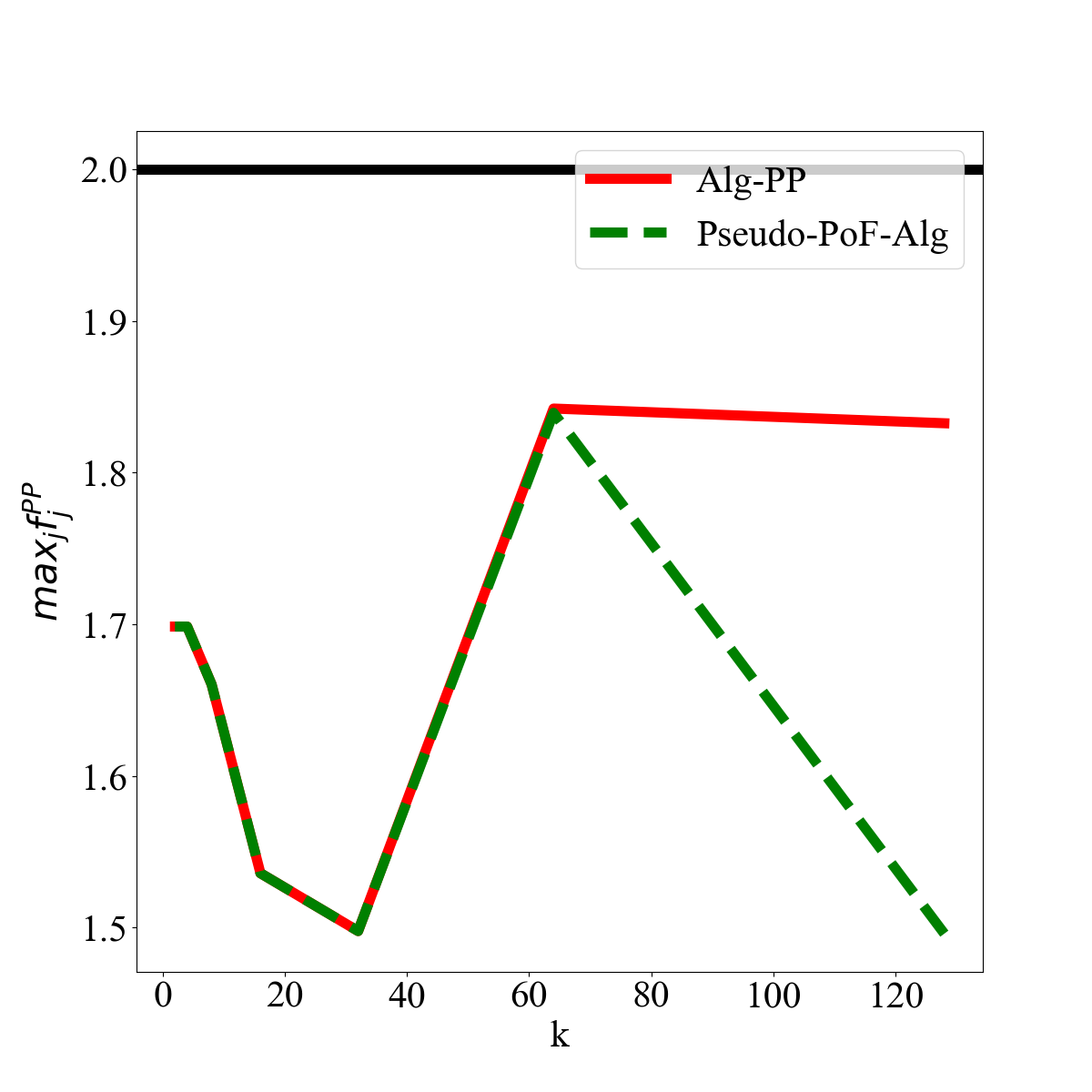}
  \caption{}
  \label{census1990-2a}
\end{subfigure}%
\begin{subfigure}{.33\textwidth}
  \centering
  \includegraphics[width=.99\linewidth]{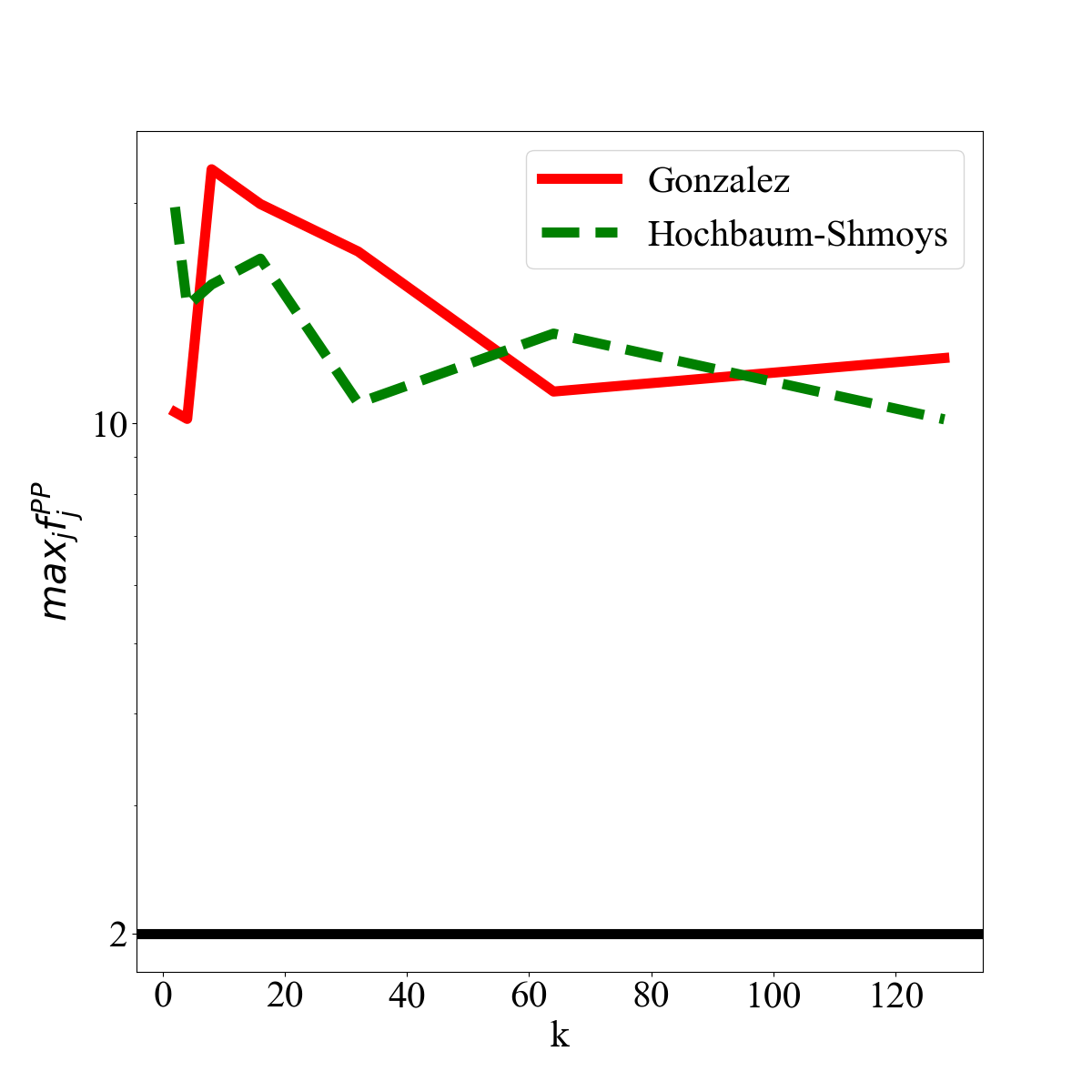}
  \caption{}
  \label{census1990-2b}
\end{subfigure}
\begin{subfigure}{.33\textwidth}
  \centering
  \includegraphics[width=.99\linewidth]{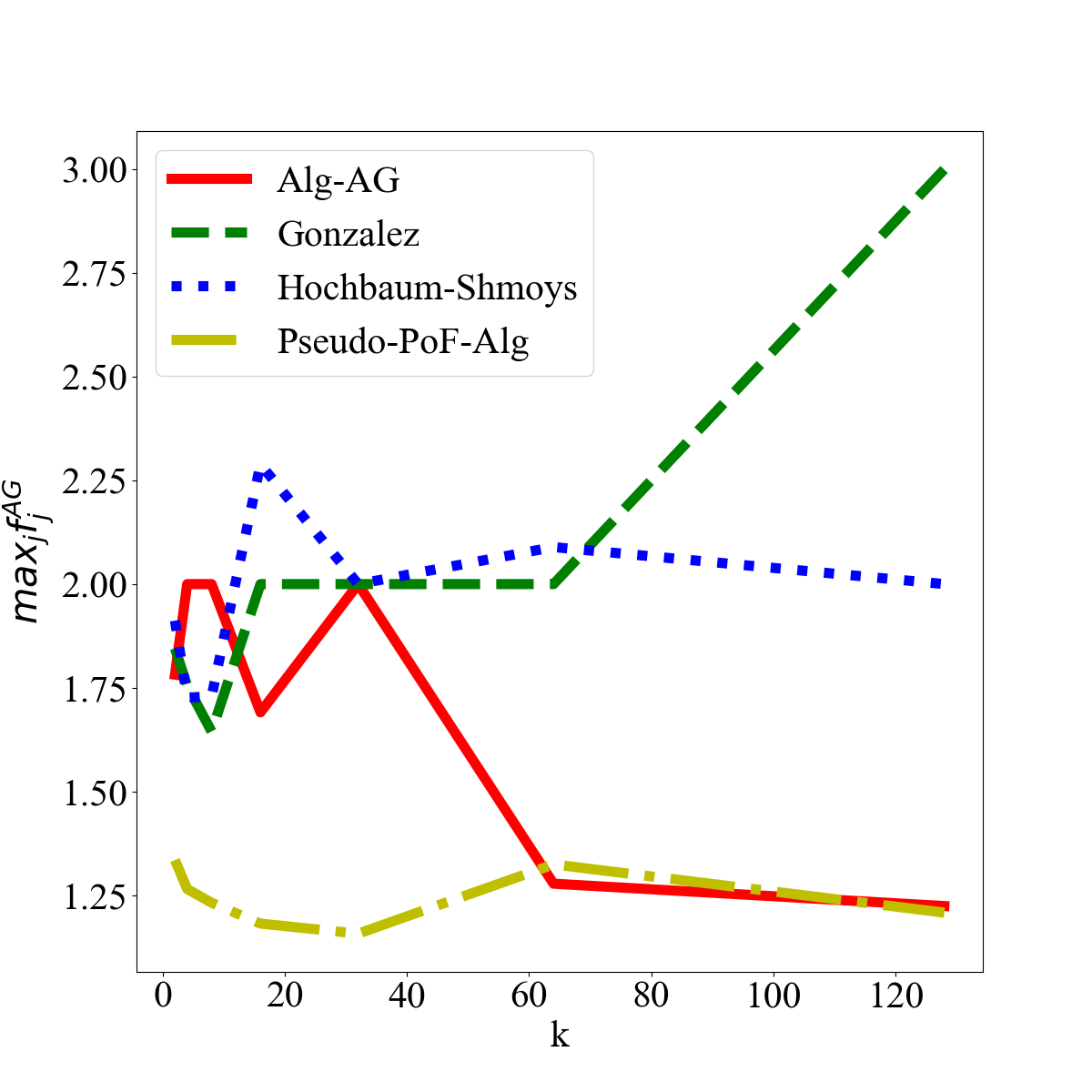}
  \caption{}
  \label{census1990-2c}
\end{subfigure}
\caption{Satisfaction of fairness constraints}
\label{census1990-2}
\end{figure}

\begin{figure}[h!]
\begin{subfigure}{.33\textwidth}
  \centering
  \includegraphics[width=.99\linewidth]{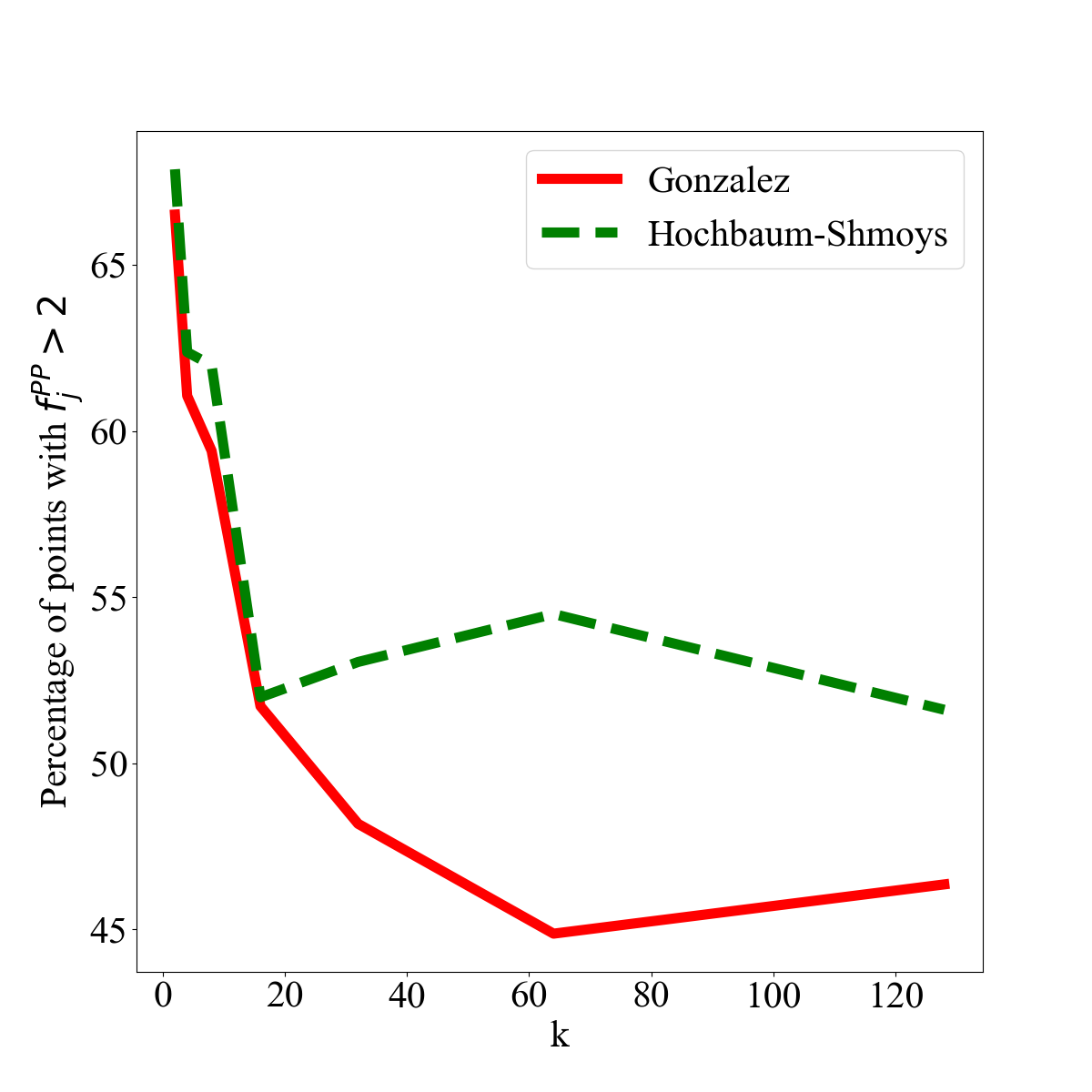}
  \caption{}
  \label{census1990-3a}
\end{subfigure}%
\begin{subfigure}{.33\textwidth}
  \centering
  \includegraphics[width=.99\linewidth]{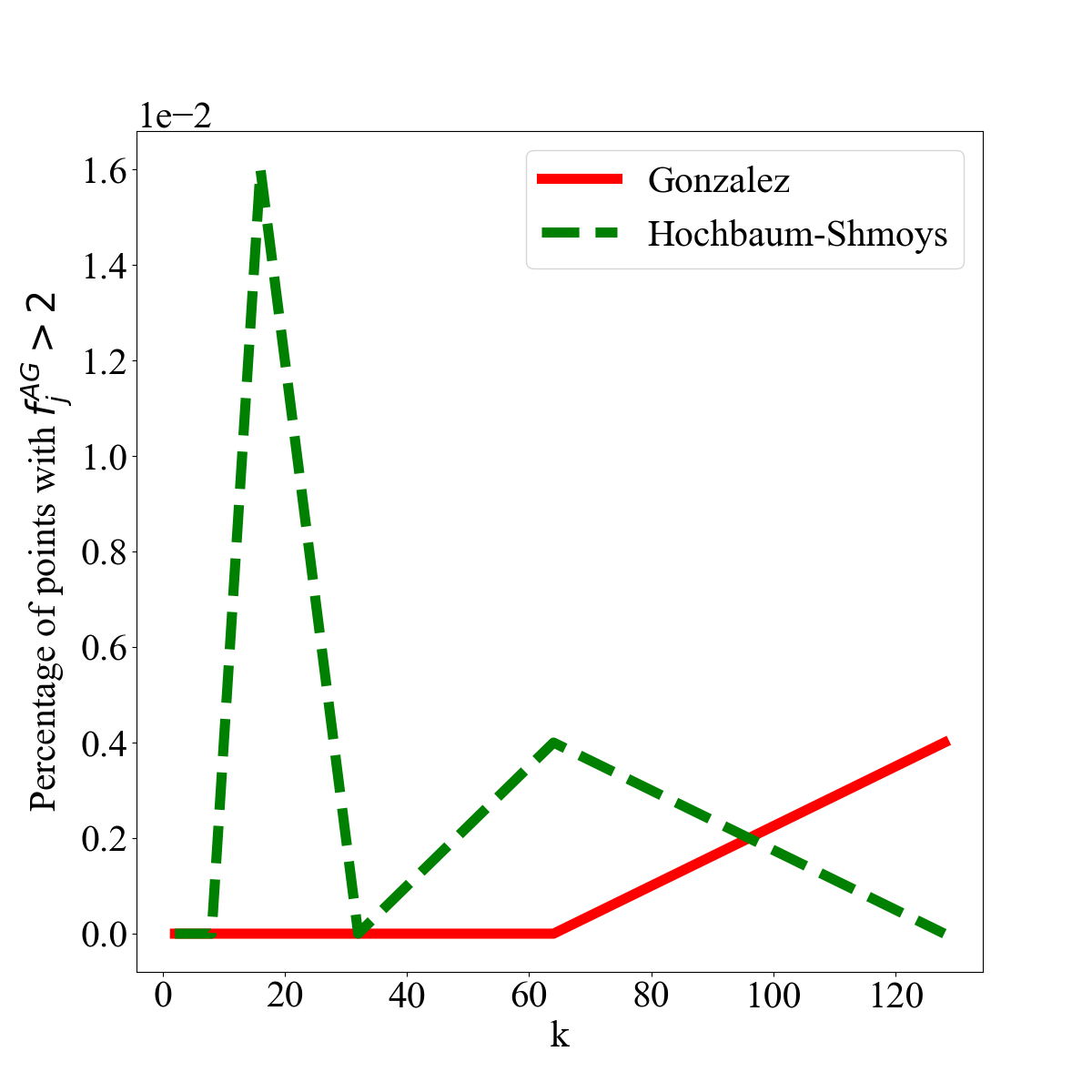}
  \caption{}
  \label{census1990-3b}
\end{subfigure}
\begin{subfigure}{.33\textwidth}
  \centering
  \includegraphics[width=.99\linewidth]{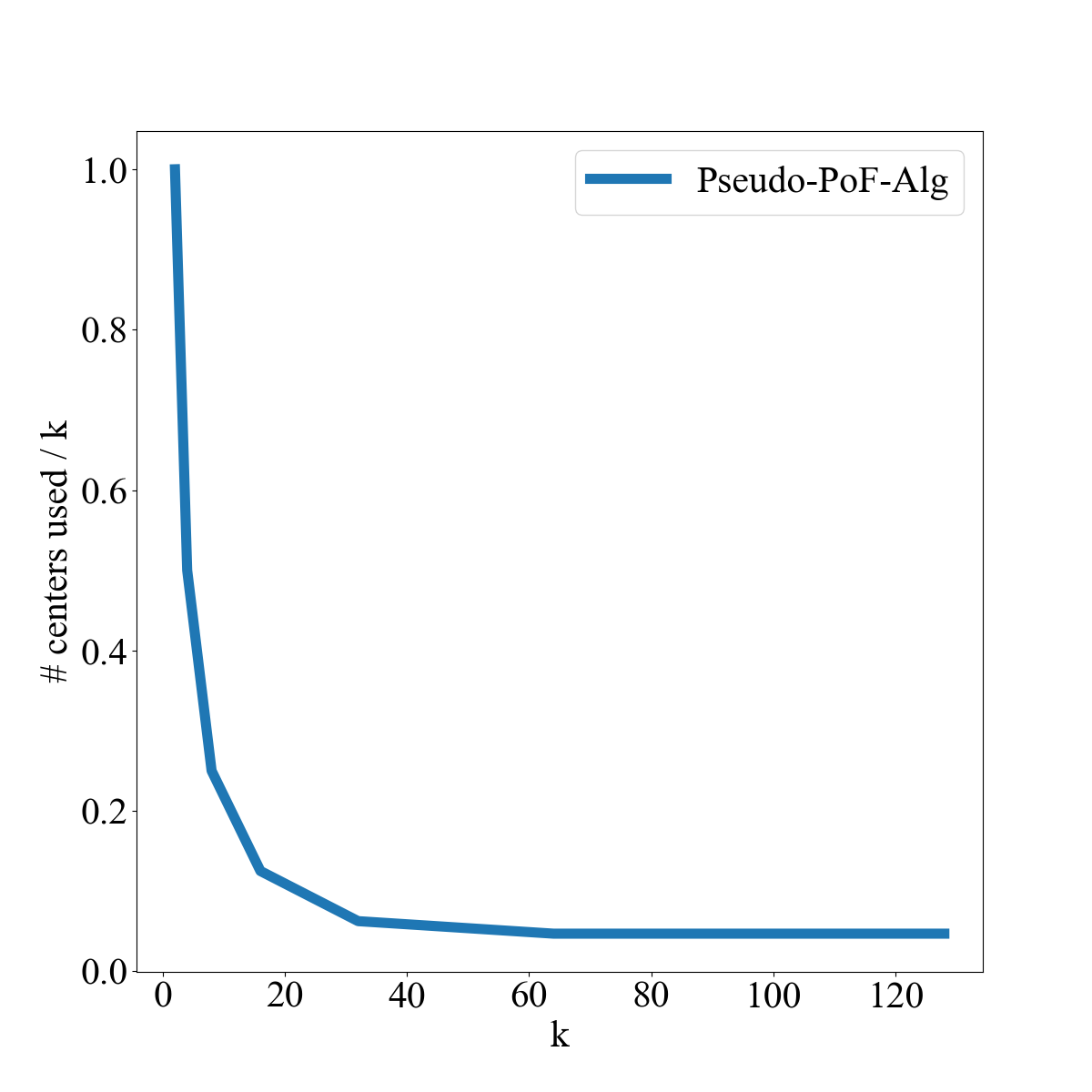}
  \caption{}
  \label{census1990-3c}
\end{subfigure}
\caption{Amount of constraint violation}
\label{census1990-3}
\end{figure}
\pagebreak

\textbf{Experimental results for Diabetes:}

\begin{figure}[h!]
\centering
\includegraphics[scale=0.18]{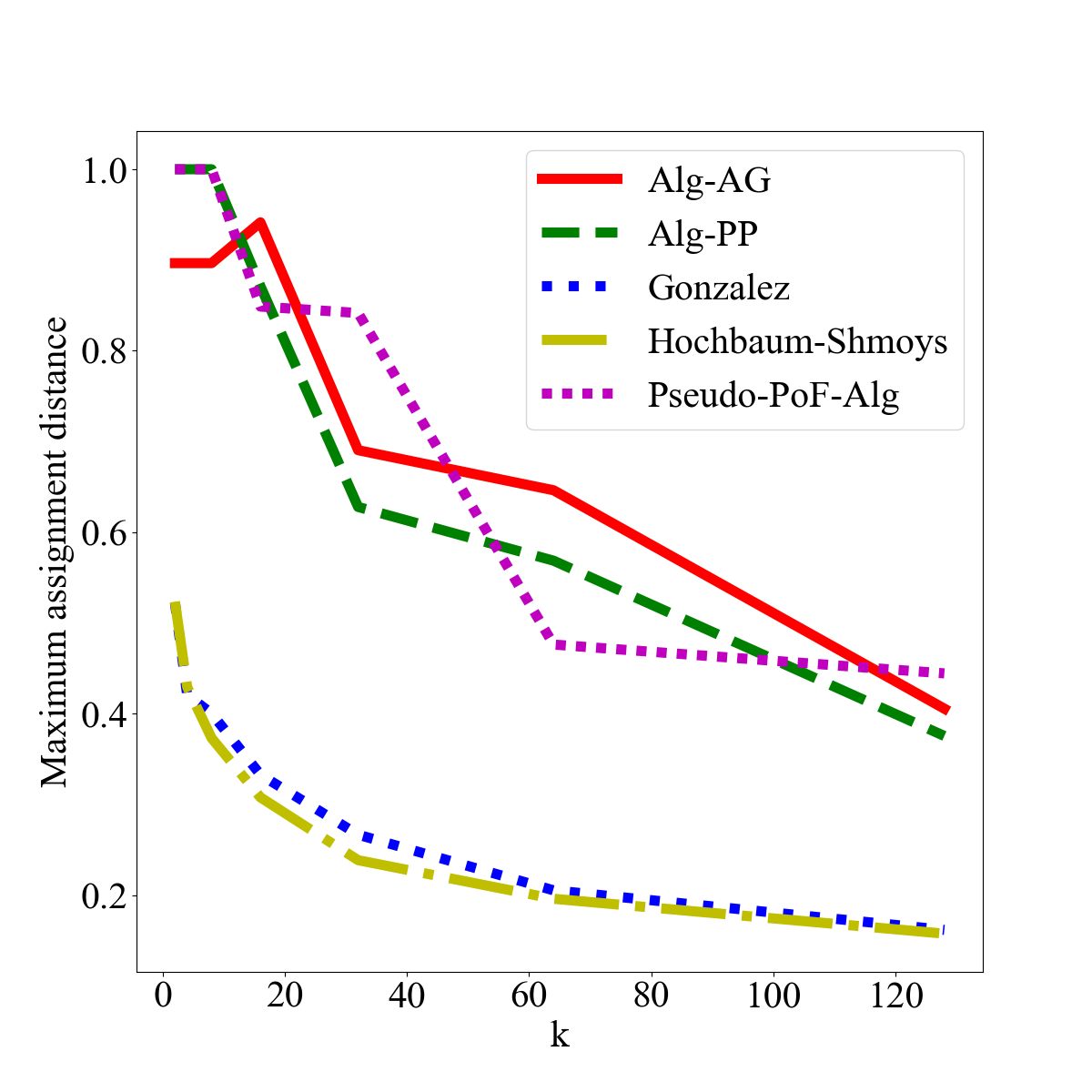}
\caption{Maximum assignment distance for all algorithms}
\label{diabetes-1}
\end{figure}

\begin{figure}[h!]
\begin{subfigure}{.33\textwidth}
  \centering
  \includegraphics[width=.99\linewidth]{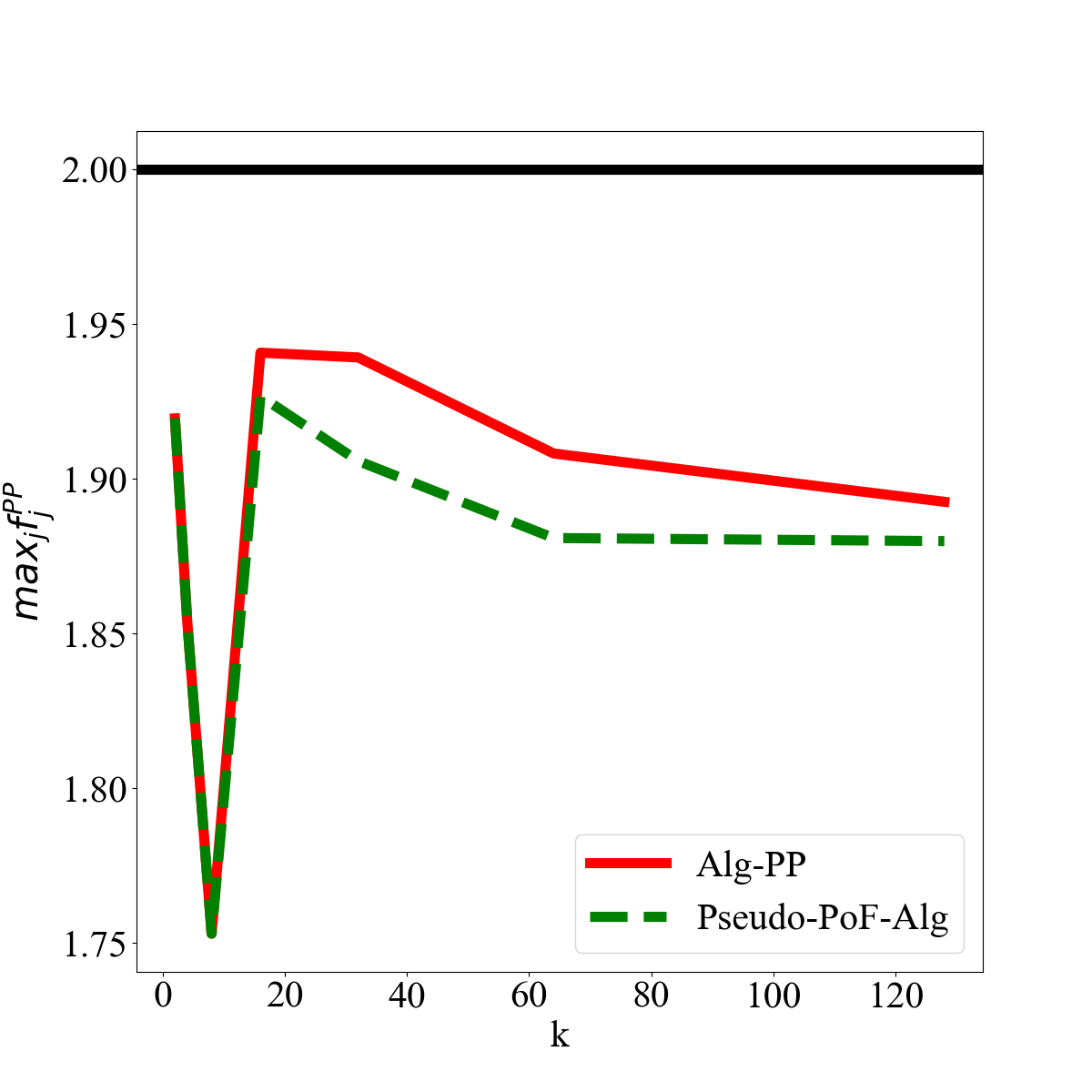}
  \caption{}
  \label{diabetes-2a}
\end{subfigure}%
\begin{subfigure}{.33\textwidth}
  \centering
  \includegraphics[width=.99\linewidth]{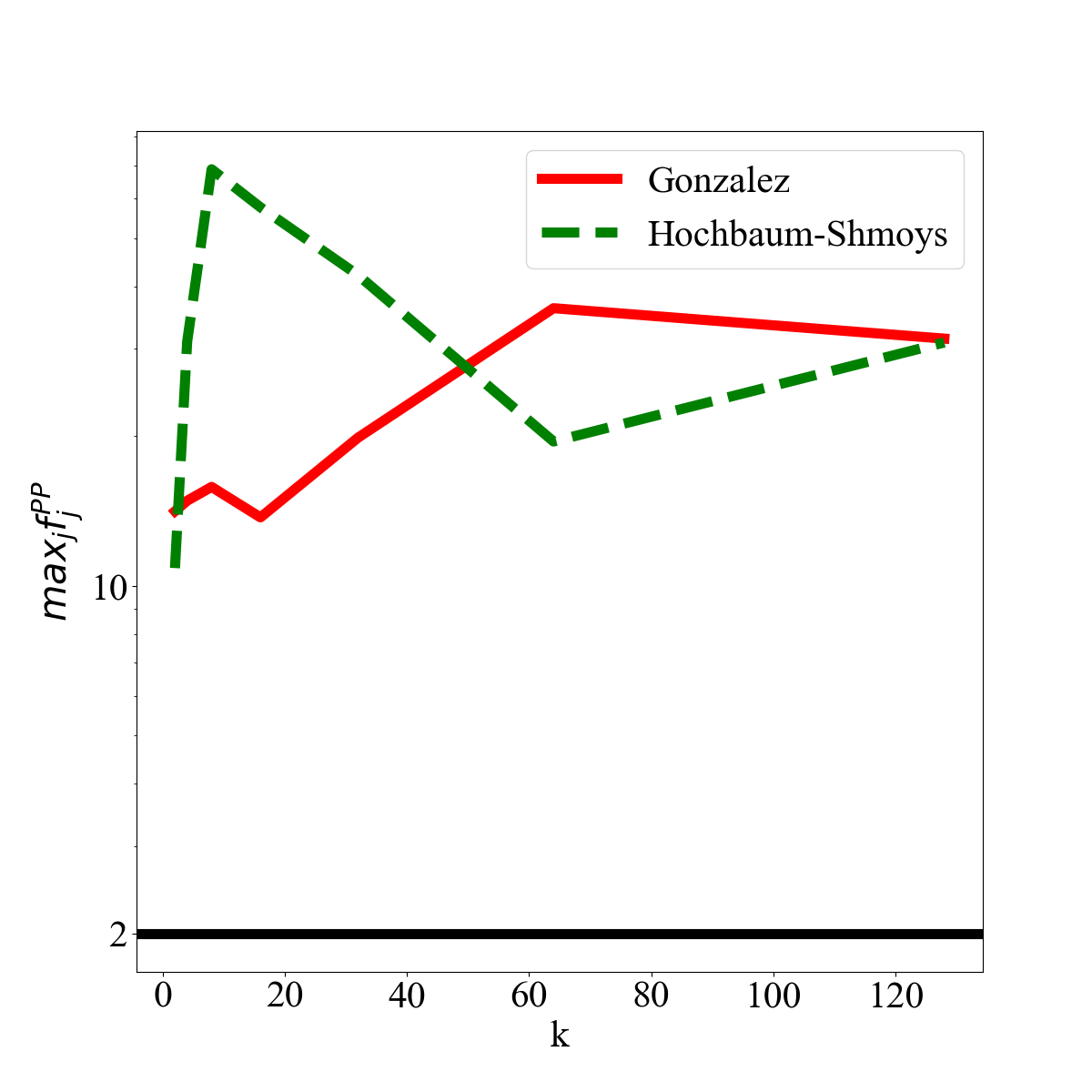}
  \caption{}
  \label{diabetes-2b}
\end{subfigure}
\begin{subfigure}{.33\textwidth}
  \centering
  \includegraphics[width=.99\linewidth]{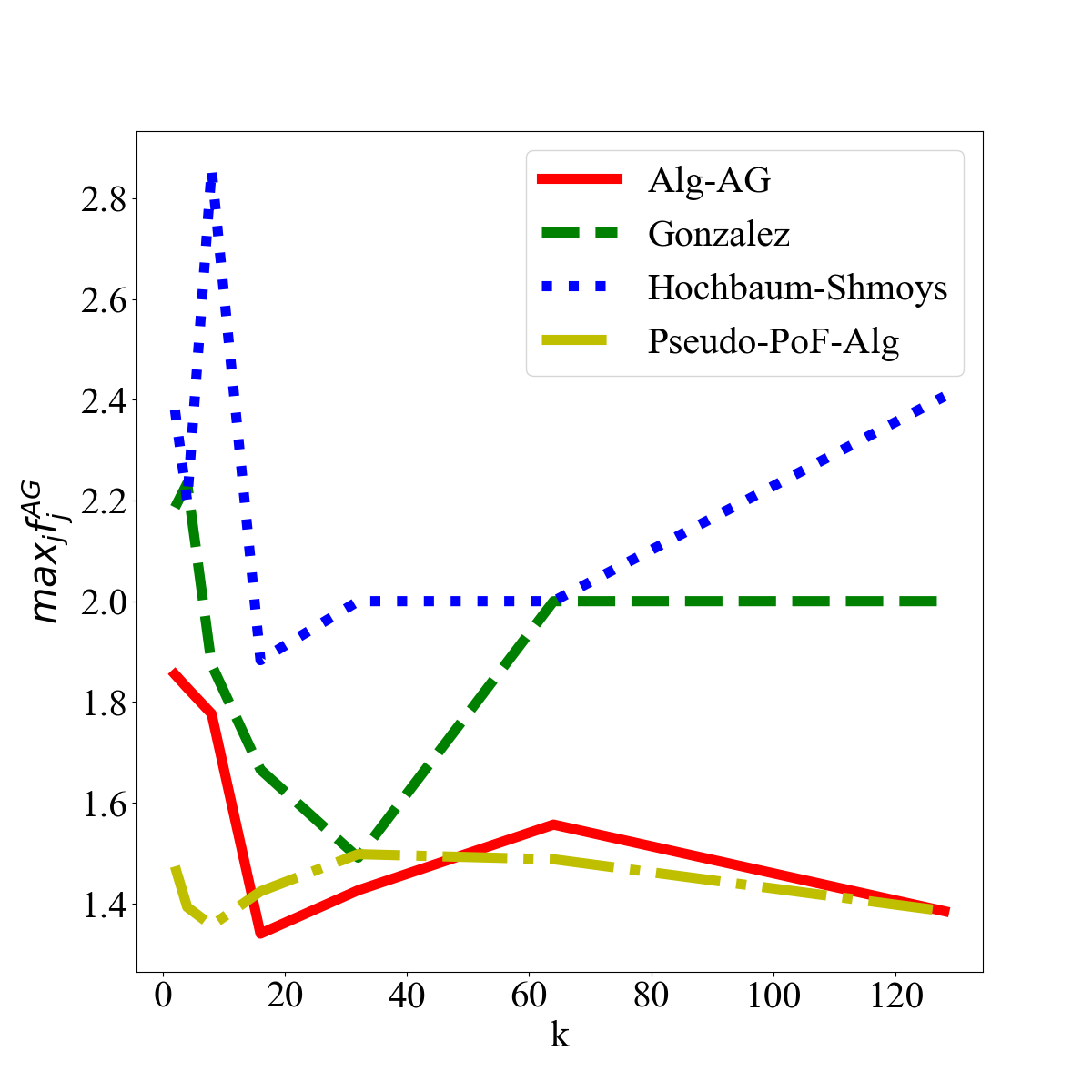}
  \caption{}
  \label{diabetes-2c}
\end{subfigure}
\caption{Satisfaction of fairness constraints}
\label{diabetes-2}
\end{figure}

\begin{figure}[h!]
\begin{subfigure}{.33\textwidth}
  \centering
  \includegraphics[width=.99\linewidth]{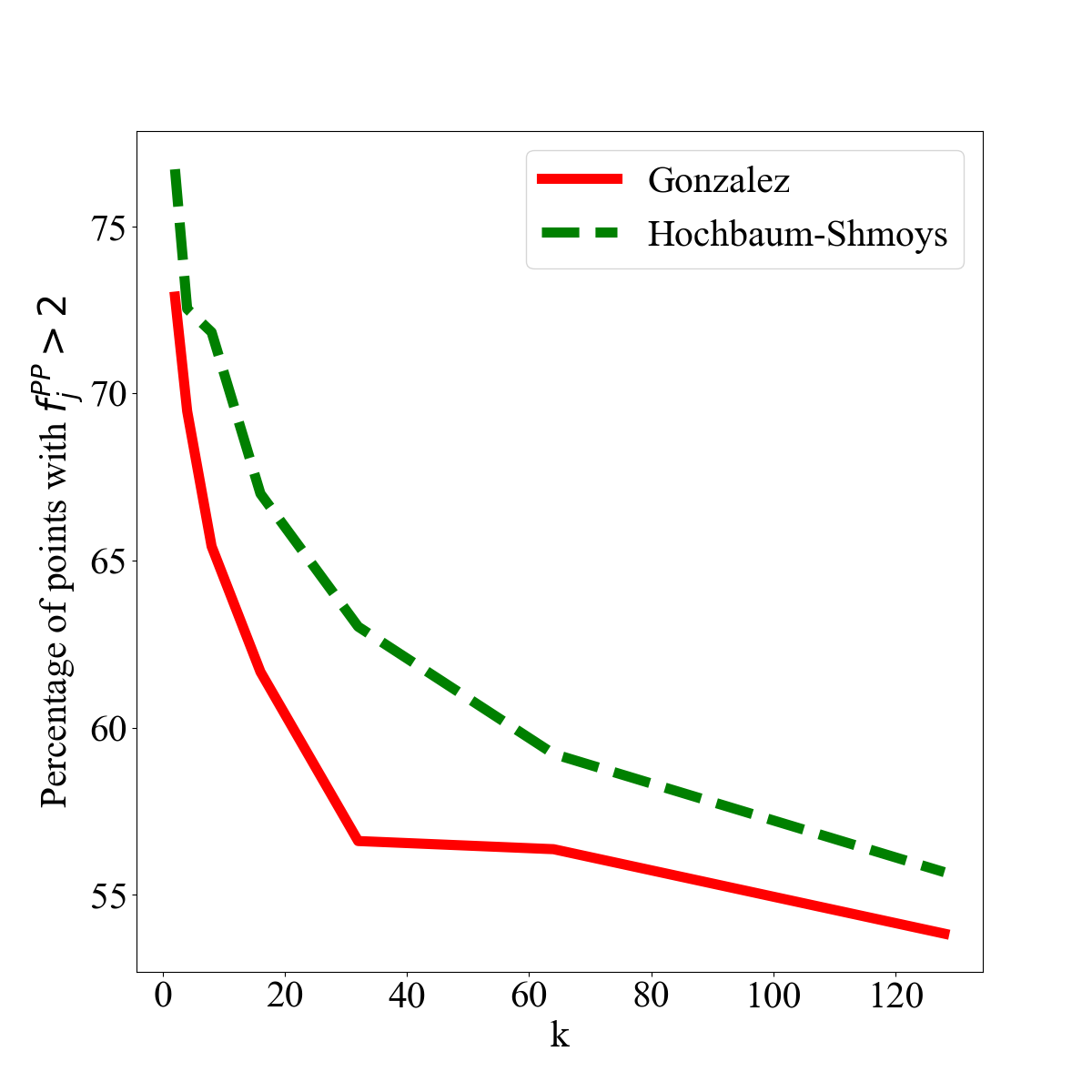}
  \caption{}
  \label{diabetes-3a}
\end{subfigure}%
\begin{subfigure}{.33\textwidth}
  \centering
  \includegraphics[width=.99\linewidth]{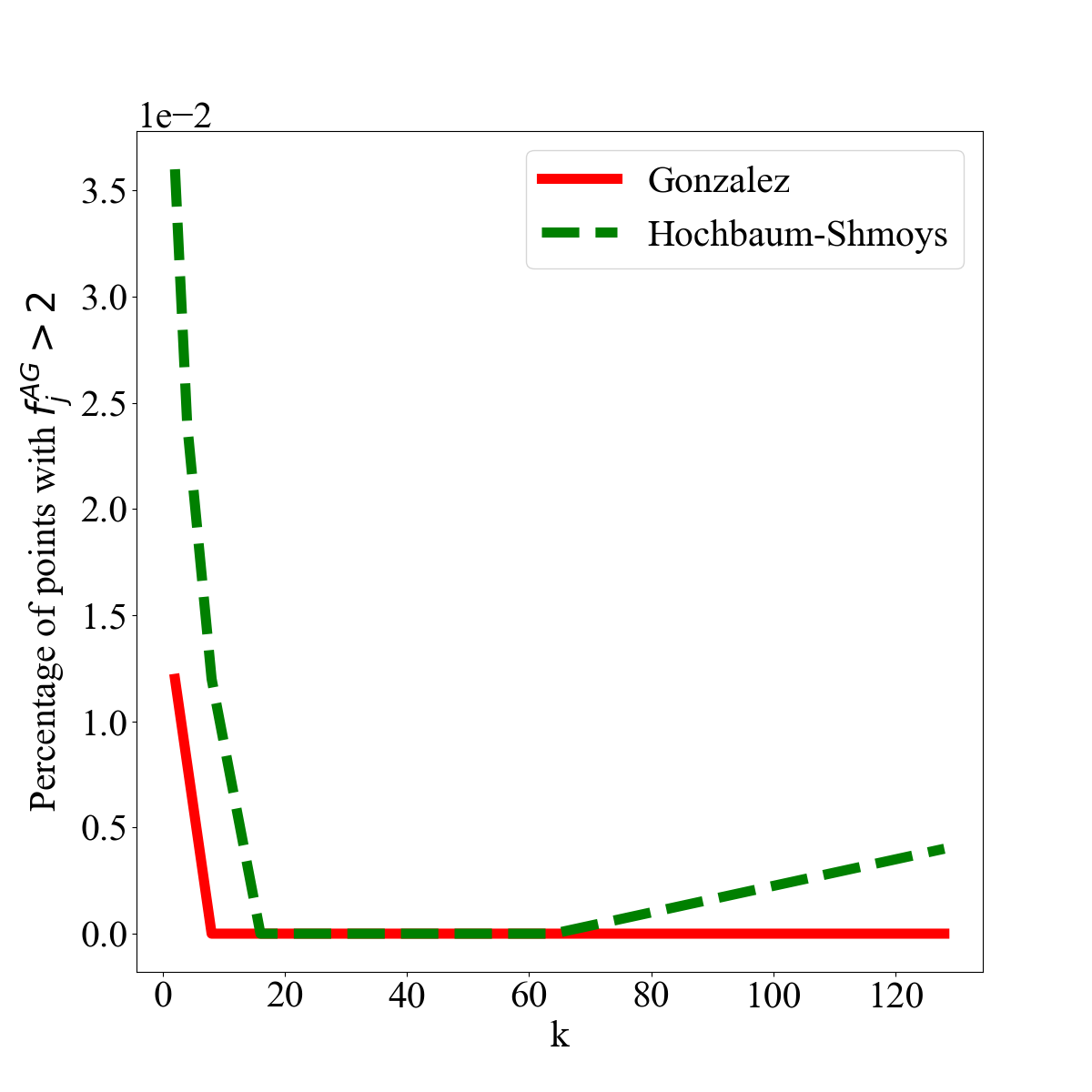}
  \caption{}
  \label{diabetes-3b}
\end{subfigure}
\begin{subfigure}{.33\textwidth}
  \centering
  \includegraphics[width=.99\linewidth]{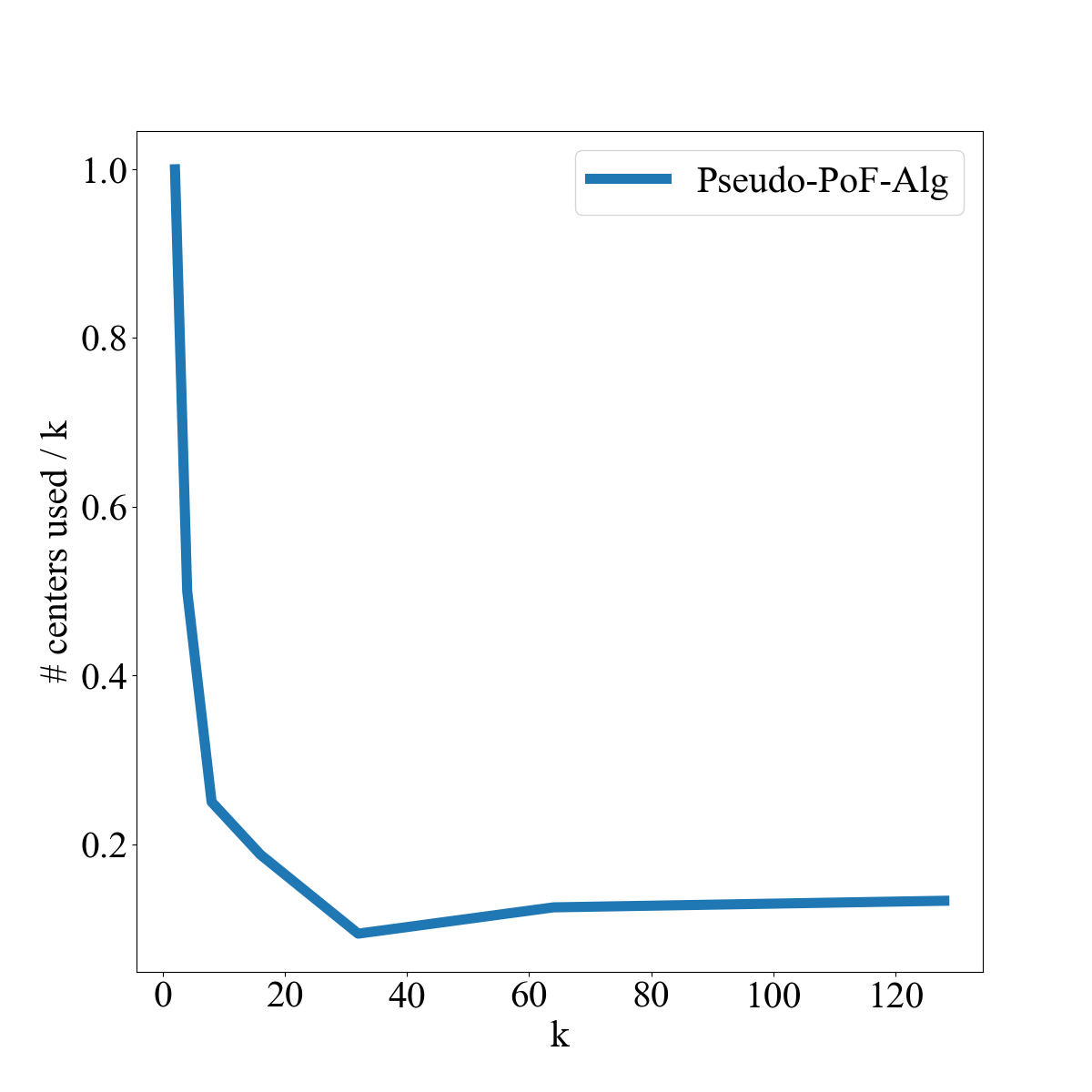}
  \caption{}
  \label{diabetes-3c}
\end{subfigure}
\caption{Amount of constraint violation}
\label{diabetes-3}
\end{figure}

\end{document}